\documentclass{article}

\usepackage{microtype}
\usepackage{graphicx}
\usepackage{subcaption}
\usepackage{booktabs} 

\usepackage{hyperref}

\usepackage{etoc}
\usepackage{titletoc} 
\usepackage[utf8]{inputenc}

\usepackage[preprint]{icml2026}

\usepackage{amsmath}
\usepackage{amssymb}
\usepackage{mathtools}
\usepackage{amsthm}
\usepackage{thm-restate}
\usepackage{thmtools}

\usepackage[capitalize,noabbrev]{cleveref}

\newtheorem{assumption}{\textbf{Assumption}}
\newtheorem{definition}{\textbf{Definition}}
\newtheorem{corollary}{\textbf{Corollary}}
\newtheorem{lemma}{\textbf{Lemma}}

\newtheorem{proposition}{\textbf{Proposition}}
\newtheorem{remark}{\textbf{Remark}}

\newcommand{\bP}{\text{\boldmath{$P$}}}
\newcommand{\bmu}{\text{\boldmath{$\mu$}}}

\newcommand{\bx}{\boldsymbol{x}}

\newcommand{\bxi}{\boldsymbol{\xi}}

\newcommand{\br}{\boldsymbol{r}}

\newcommand{\bu}{\boldsymbol{u}}
\newcommand{\bb}{\boldsymbol{b}}
\newcommand{\bc}{\boldsymbol{c}}

\newcommand{\by}{\boldsymbol{y}}

\newcommand{\bS}{\boldsymbol{S}}

\newcommand{\btheta}{\boldsymbol{\theta}}
\newcommand{\bphi}{\boldsymbol{\phi}}
\newcommand{\bp}{\boldsymbol{p}}

\newcommand{\bI}{\boldsymbol{I}}

\newcommand{\bA}{\boldsymbol{A}}

\newcommand{\bV}{\boldsymbol{V}}
\newcommand{\bbP}{\mathbb{P}}

\newcommand{\bbR}{\mathbb{R}}

\newcommand{\bv}{\boldsymbol{v}}
\newcommand{\bw}{\boldsymbol{w}}

\newcommand{\diag}{\mathsf{diag}}

\newcommand{\mE}{\mathbb{E}}

\newcommand{\Var}{\mathsf{Var}}

\newcommand{\TV}{\mathsf{TV}}
 \newcommand{\cE}{\mathcal{E}}
\newcommand{\cX}{\mathcal{X}}
\newcommand{\cC}{\mathcal{C}}

\newcommand{\cA}{\mathcal{A}}
\newcommand{\cH}{\mathcal{H}}
\newcommand{\cO}{\mathcal{O}}

\newif\ifappendix
\appendixtrue 

\usepackage[textsize=tiny]{todonotes}

\icmltitlerunning{Towards a Theoretical Understanding to the Generalization of RLHF}

\begin{document}

\twocolumn[
  \icmltitle{Towards a Theoretical Understanding to the Generalization of RLHF}



  \icmlsetsymbol{equal}{*}

  \begin{icmlauthorlist}
    \icmlauthor{Zhaochun Li}{bit,bjzgca}
    \icmlauthor{Mingyang Yi}{ruc}
    \icmlauthor{Yue Wang}{bjzgca}
    \icmlauthor{Shisheng Cui}{bit}
    \icmlauthor{Yong Liu}{ruc}
   
    
  \end{icmlauthorlist}

  \icmlaffiliation{bit}{Beijing Institute of Technology}
  \icmlaffiliation{bjzgca}{Zhongguancun Academy}
  \icmlaffiliation{ruc}{Renmin University of China}

  \icmlcorrespondingauthor{Mingyang Yi}{yimingyang@ruc.edu.cn}

  \icmlkeywords{Machine Learning, ICML}

  \vskip 0.3in
]



\printAffiliationsAndNotice{}  

\etocdepthtag.toc{mtchapter} 

\begin{abstract}
Reinforcement Learning from Human Feedback (RLHF) and its variants have emerged as the dominant approaches for aligning Large Language Models with human intent. While empirically effective, the theoretical generalization properties of these methods in high-dimensional settings remain to be explored. 
To this end, we build the generalization theory on RLHF of LLMs under the linear reward model, through the framework of algorithmic stability. 
In contrast to the existing works built upon the consistency of maximum likelihood estimations on reward model, our analysis is presented under an end-to-end learning framework, which is consistent with practice. Concretely, we prove that under a key \textbf{feature coverage} condition, the empirical optima of policy model have a generalization bound of order $\cO(n^{-\frac{1}{2}})$. Moreover, the results can be extrapolated to parameters obtained by gradient-based learning algorithms, i.e., Gradient Ascent (GA) and Stochastic Gradient Ascent (SGA). Thus, we argue that our results provide new theoretical evidence for the empirically observed generalization of LLMs after RLHF.    
\end{abstract}

\section{Introduction}

Reinforcement Learning from Human Feedback (RLHF) and its variants have emerged as the dominant approaches for aligning Large Language Models (LLMs) with human intent\citep{christiano2017deep,ouyang2022training,guo2025deepseek, yang2025qwen3, comanici2025gemini}. Guided by this paradigm, LLMs have demonstrated remarkable capabilities across diverse domains, such as image generation \citep{liu2024sora, shen2025vlm}, Embodied AI \citep{intelligence2025pi,ye2025vla} and so on . While empirically effective, the theoretical understanding of the generalization properties of RLHF remains to be explored, particularly in high-dimensional settings. 



Currently, the theoretical frameworks of RLHF for LLMs largely rely on the structural insight that the optimal policy for the RLHF objective admits a closed-form posterior Boltzmann distribution depending on the reward model \citep{rafailov2023direct}. Building on this, the existing literature \citep{zhu2023principled, xiong2023iterative} typically focus on the statistical quality of reward modeling: they analyze the convergence of reward parameters learned via Maximum Likelihood Estimation (MLE) from offline data, and subsequently treat the estimated parameters in the closed-form posterior Boltzmann distribution as the learned policy. 

However, this analytical paradigm faces two fundamental limitations. First, it necessitates stringent data coverage assumptions to guarantee the consistency of reward parameters\citep{song2024importance}, which is rarely satisfied in high-dimensional regimes \citep{ethayarajh2019contextual, aghajanyan2021intrinsic, ruppik2025less}. 
Second, and perhaps more critically, it focuses solely on the supervised learning nature of reward modeling. 
By treating policy optimization as an exact maximization step, such analyses fail to capture the benefits inherent to the reinforcement learning process itself. 


To address these limitations, 
we propose an end-to-end theoretical framework that captures the dynamic interaction between the policy and the environment during practical RLHF training. 
Our analyses are conducted on the KL-regularized RLHF objective \citep{zhao2024sharp, azar2024general} 
under the standard linear reward assumption \citep{jin2020provably}. 
Within this framework, we focus our analysis on the suboptimality gap of the learned policy, isolating the generalization error as a critical component.



Technically, our analyzes are grounded in a key \textbf{feature coverage assumption}, which intuitively posits that as the dataset grows, the column space of the empirical covariance matrix of model's feature progressively expands to cover the feature space. Consequently, the residual component of any feature vector falling outside the covered subspace diminishes. This geometric property allows us to rigorously evaluate the generalization error of model obtained across different optimization regimes, varying from the empirical optima to models obtained by gradient-based methods.


Our theoretical results can be summarized as follows:
\begin{itemize}
    \item \textbf{The Suboptimality Gap on Empirical Optima.} Under a strong no-degenerate assumption (Assumption \ref{ass:positive definite}), we prove the empirical optima can match the ground truth parameters when data is sufficient. Moreover, under the proposed feature coverage assumption (Assumption \ref{ass:span space}), by analyzing the algorithmic stability \citep{bousquet2002stability,yi2022characterization}, we prove a $\tilde{\cO}(n^{-\frac{1}{2}})$ bound of suboptimality gap on empirical optima, when we have $n$ training data. 
    \item  \textbf{The Suboptimality Gap under Gradient-Based Algorithms.}
    Moreover, we extend our algorithmic stability analysis to the practical setting of iterative optimization algorithms. Under the key feature coverage assumption, we derive explicit suboptimality gap bounds for Gradient Ascent (GA) and Stochastic Gradient Ascent (SGA) of order $\tilde{\cO}(T^{-\frac{1}{4}}+n^{-\frac{1}{2}})$ and $\tilde{\cO}(T^{-\frac{1}{8}}+n^{-\frac{1}{2}})$, respectively.
\end{itemize}

The remainder of this paper is organized as follows. Section \ref{sec:related work} and Section \ref{sec:preliminary} review related work and introduce the formal setting, respectively.
Section \ref{sec:suboptimality} formally defines the suboptimality gap and presents its decomposition. Section \ref{sec:suboptimality of empirical optimum} analyzes the suboptimality gap of the empirical optimum with respect to the true optimal policy. The analysis based on the algorithmic stability is extended to the gradient-based optimization algorithms (GA and SGA) in Section \ref{sec:suboptimality of gradient based algorithms}. Section \ref{sec:experiments} empirically verifies the theoretical results in this paper.

\section{Related Work}\label{sec:related work}
\paragraph{Theoretical Foundations of RLHF.}

The theoretical foundations of RLHF are deeply rooted in the context of dueling bandits \citep{dudik2015contextual, bengs2021preference, chang2024dataset} and preference-based reinforcement learning \citep{wirth2017survey, xu2020preference}. For dueling bandits, the learner aims to minimize regret based on pairwise comparisons between actions. Theoretical analyses in this domain have progressively expanded from the tabular setting \citep{zhan2023provable} to linear models \citep{xiong2023iterative, saha2023dueling, xu2024dpo} and general function approximation \citep{chen2022human}. With the empirical success of RLHF, these analyzes are extended to RLHF or directly preference-based optimization (DPO) \citep{rafailov2023direct}.  
Furthermore, \cite{zhu2023principled} quantified the estimation error between the reward parameters learned via MLE and the ground truth, demonstrating that applying pessimistic strategies to the learned reward model guarantees that the resulting policy achieves low sub-optimality. Subsequent research has adopted a similar analytical framework to extend these guarantees to DPO and other related algorithms \citep{ ethayarajh2024kto, bai2025online}.

In contrast to prior works \citep{xiong2023iterative,zhu2023principled,chaudhari2025rlhf} that primarily evaluate the consistency between parameters estimated from offline preference data and the ground truth, we adopt an end-to-end analytical framework and directly analyze the suboptimality gap of the policy obtained by RL algorithm. 

\paragraph{Algorithmic Stability and Generalization.}
Algorithmic stability, rigorously formalized by \citep{bousquet2002stability}, provides a powerful framework for bounding generalization error. While early works focused on deterministic algorithms, \citep{hardt2016train} significantly expanded this framework to randomized algorithms, establishing uniform stability guarantees for Stochastic Gradient Descent in both convex and non-convex settings \citep{bassily2020stability, yi2022characterization}. Subsequently, \citep{feldman2019high} and \citep{bousquet2020sharper} derived sharper high-probability bounds by improving the dependence on sample size and failure probability. 
Complementing these statistical advancements, recent works have grounded stability in the optimization dynamics, showing that the algorithm's sensitivity is explicitly controlled by the gradient norm \citep{kuzborskij2018data, lei2020fine, wang2022generalization, zhu2025stability}. Despite these advancements in supervised learning, the application of stability analysis to Reinforcement Learning remains limited. Existing works primarily focus on value-based methods in tabular settings \citep{fan2020theoretical, agarwal2020pc}. In this paper, we bridge this gap by applying the uniform stability framework to KL-regularized RLHF objective, deriving explicit generalization bounds for GA and SGA in the context of RLHF.

\section{Formal Setting}\label{sec:preliminary}
This section presents the background material that will be used throughout the paper.

\paragraph{Notation.} We use $\langle \bu,\bv \rangle$ to denote the inner product and $\|\cdot\|$ for the Euclidean norm of vectors and the spectral norm for matrices. For a matrix $\bA$, let $\mathrm{C}(\bA)$ and $\mathrm{N}(\bA)$ be its column and null spaces, while $\sigma_{\max}(\bA)$, $\sigma_{\min}(\bA)$ and $\sigma_{\min}^{+}(\bA)$ represent its largest, smallest, and smallest positive eigenvalues, respectively. $\mathcal{X}$ and $\mathcal{A}$ denote the context and action spaces. Finally, $\mathcal{O}(\cdot)$ refers to the standard big-O notation, $\tilde{\mathcal{O}}(\cdot)$ suppresses logarithmic factors, and $a \lesssim b$ implies $a \le C b$ for some constant $C > 0$.


\subsection{Preliminaries}

\textbf{Reinforcement Learning from Human Feedback}

We situate our work within the reinforcement learning from human feedback (RLHF) paradigm as formulated in \cite{ziegler2019fine}, given a dataset of prompts $\bS = \{x_{1},\cdots ,x_{n}\}$ and a fixed
reference policy $\pi_{\rm{ref}}$. Unlike the standard two-stage RLHF framework where a reward model $r(x, a)$ is learned from preferences, we operate in a setting where ground-truth rewards on prompts $x\in\cX$ and response $a\in\cA$ are available. 
For a given $x$, denote the objective as:
\begin{equation}\label{eq:single sample loss}
    \begin{aligned}
        f_{\pi}(x)=\mE_{a\sim\pi(\cdot\mid x)}\left[r(x,a)\right] - \lambda \text{D}_{\text{KL}}(\pi(\cdot \mid x)\|\pi_{\rm{ref}}(\cdot\mid x)).
    \end{aligned}
\end{equation}

The goal of RLHF is to learn a policy $\pi$ to maximize the objective $J(\pi)$, where:
\begin{equation}\label{eq:expectation loss}
    \small
     \begin{aligned}
     J(\pi)=\mathbb{E}_{x\in \mathcal{X}}[f_{\pi}(x)],
    \end{aligned}
\end{equation}
i.e., we want to maximize the expected reward without deviating too much from the reference policy. W.l.o.g., to simplify following analysis, we set $\lambda=1$ below. In the case that $\lambda \neq 1$, we divide both sides of \eqref{eq:single sample loss} by $\lambda$, and treat the term $\frac{r(x,a)}{\lambda}$ as a new reward function. This rescaling does not affect the solutions of the original objective function.

Before introducing the parametric model, we first characterize the optimal solution to \eqref{eq:expectation loss} over the space of all valid probability distributions.

\begin{proposition}[\citet{rafailov2023direct}]\label{pro:closed form}
    For any reward function $r(x, a)$, the optimal policy $\pi^*$ that maximizes \eqref{eq:expectation loss} admits the following closed-form expression:
    \begin{equation}\label{eq:closed form}
        \small
        \begin{aligned}
            \pi^{*}(a|x)=\frac{\pi_{\rm{ref}}(a|x)\exp(r(x,a))}{\sum_{a^{\prime}}\pi_{\rm{ref}}(a|x)\exp(r(x,a))}.
        \end{aligned}
    \end{equation}
\end{proposition}

Motivated by this optimal form, we adopt a parametric approach to approximate the reward and the policy. We consider the widely used linear reward assumption \citep{zhu2023principled,xiong2023iterative} and the corresponding posterior Boltzmann policy class.
\begin{definition}[linear reward function class]
    Let $\bphi(x, a)\in \bbR^{d}$ be a fixed feature mapping. We consider the class of linear reward functions defined as:
    \begin{equation*}
    \small
        \begin{aligned}
            \mathcal{F}=\left\{r_{\btheta}|r_{\btheta}(x,a)=\left\langle\btheta, \bphi(x,a)\right\rangle\right\}. 
        \end{aligned}
    \end{equation*}
\end{definition}
Correspondingly, motivated by the optimal solution form in \eqref{eq:closed form}, we consider the following policy class:
\begin{definition}[Posterior Boltzmann policy function class]\label{def:policy function}
    \begin{equation*}
    \small
        \begin{aligned}
            \Pi= \left\{\pi_{\btheta}|\pi_{\btheta}(a|x) =\frac{\pi_{\rm{ref}}(a|x)\exp\left(\left\langle\btheta, \bphi(x,a)\right\rangle\right)}{\sum_{a^{\prime}\in \mathcal{A}}\pi_{\rm{ref}}(a^{\prime}|x)\exp\left(\left\langle\btheta, \bphi(x,a)\right\rangle\right)}\right\}.
        \end{aligned}
    \end{equation*}   
\end{definition}

To facilitate theoretical analysis, we make the following standard regularity assumptions regarding the boundedness of features and the realizability of the true reward.
\begin{assumption}\label{ass:bounded}
    For all $(x, a) \in \mathcal{X} \times \mathcal{A}$, the feature norm is bounded such that $\|\boldsymbol{\phi}(x, a)\| \le C$. 
\end{assumption}

\begin{assumption}\label{ass:ground truth r}
    There exists a ground-truth parameter $\btheta^{*} \in \bbR^{d}$ such that the true reward function is linear, i.e. $r(x, a) = \langle \btheta^{*}, \bphi(x, a) \rangle$, with the parameter bounded by  $\left\|\btheta^{*}\right\| \le D$.
\end{assumption}



Under Assumption \ref{ass:ground truth r} and Proposition \ref{pro:closed form}, it is straightforward to verify that the policy model parameterized by ground-truth $\btheta^{*}$, denoted as $\pi_{\btheta^{*}}$, is the global maximizer of the expected objective $J(\pi_{\btheta})$, i.e.
\begin{equation}
    \small
    \begin{aligned}
        \btheta^{*} \in \arg\max_{\btheta \in \mathbb{R}^{d}}J(\pi_{\btheta}). 
    \end{aligned}
\end{equation}

However, since the context space $\mathcal{X}$ is vast, directly maximizing the expected objective in \eqref{eq:expectation loss} is computationally intractable. Instead, we employ the empirical approximation based on finite samples. In practice, 
we approximate the expectation using a finite prompt dataset $\bS=\left\{x_{i}\right\}_{i=1}^{n}$, leading to the following empirical objective:
\begin{equation}\label{eq:empirical loss}
    \small
    \begin{aligned}
        \max_{\btheta}J_{\bS}(\pi_{\btheta}) = \max_{\btheta}\frac{1}{n}\sum\limits_{i=1}^{n}f_{\btheta}(x_{i}),
    \end{aligned}
\end{equation}

We use $\btheta_{\bS}$ to denote the parameters obtained on $\bS$. For our analysis, we focus on the empirical stationary points within a bounded region $\Theta_{R}:= \{ \btheta \in \mathbb{R}^d \mid \|\btheta\| \le R \}$ in \eqref{eq:empirical optimal}, where we set $R=3D$. This focus is sufficient because the parameters obtained by gradient-based algorithms are naturally confined within this bound (see Lemma \ref{lem:bounded parameter}), eliminating the need to consider the entire space.
\begin{equation}\label{eq:empirical optimal}
    \small
    \begin{aligned}
        \btheta_{\bS}^{*} \in \{\btheta \in \Theta_{R}| \nabla_{\btheta}J_{\bS}(\btheta) = 0\} \triangleq \Theta_{\bS}.
    \end{aligned}
\end{equation}

\section{Suboptimality Decomposition}\label{sec:suboptimality}

To rigorously quantify the performance of the learned policy on the unknown context space $\mathcal{X}$, we define the sub-optimality gap and decompose it into three distinct components: concentration, optimization, and generalization.
\begin{definition}[Suboptimality Gap Decomposition]\label{def:decomposition}
    Let $\btheta_{\bS}$ be the output of an algorithm on dataset $\bS$. The suboptimality gap is decomposed as:
       \begin{equation}\label{eq:decomposition of suboptimality}
        \small
        \begin{aligned}
            &\mathrm{Subopt}(\btheta_{\bS})=J(\pi_{\btheta^{*}})-J(\pi_{\btheta_{\bS}})
            \\ \le &\underbrace{|J(\pi_{\btheta^{*}})-J_{\bS}(\pi_{\btheta_{\bS}^{*}})|}_{\mathrm{Concentration\ error}}+\underbrace{|J_{\bS}(\pi_{\btheta_{\bS}^{*}})-J_{\bS}(\pi_{\btheta_{\bS}})|}_{\mathrm{Optimization\ error}}
            \\ + & \underbrace{|J_{\bS}(\pi_{\btheta_{\bS}})-J(\pi_{\btheta_{\bS}})|}_{\mathrm{Generalization\ error}},
        \end{aligned}
    \end{equation}
    where $\btheta_{\bS}^{*} \in \Theta_{\bS}$ is one of the empirical stationary point.
\end{definition}

\paragraph{Optimization error.}
The optimization error quantifies the performance gap between the learned parameter $\btheta_{\bS}$ and the empirical stationary point $\btheta_{\bS}^{*}$ on the training objective $J_{\bS}$. In the context of non-convex optimization, iterative algorithms typically possess theoretical guarantees of convergence to empirical stationary points. Crucially, based on Lemma \ref{lem:key lemma}, we can demonstrate that such optimization error is explicitly controlled by the gradient norm $\small \|\nabla_{\btheta} J_{\bS}(\btheta_{\bS})\|$.
\begin{corollary}
Under Assumptions \ref{ass:span space}, for any $\btheta_{\bS}\in \Theta_{R}$ and a positive constant $\delta$, if $ \|\nabla_{\btheta}J_{\bS}(\pi_{\btheta_{\bS}})\| \le \delta$, then it holds
\begin{equation*}
    \small
    \begin{aligned}
        &|J_{\bS}(\pi_{\btheta_{\bS}^{*}})-J_{\bS}(\pi_{\btheta_{\bS}})| \le 4R\epsilon_{n}+\frac{2C\delta\cdot \sigma_{\max}(\bV_{\bS}(\btheta_{\bS}^{*}))}{\sigma_{\min}^{+}(\bV_{\bS}(\btheta_{\bS}^{*}))\sigma_{\min}^{+}(\bV_{\bS}(\btheta_{\bS}))}
        \\ &+2RC\sqrt{4R\epsilon_{n}+\frac{2C\delta\cdot \sigma_{\max}(\bV_{\bS}(\btheta_{\bS}^{*}))}{\sigma_{\min}^{+}(\bV_{\bS}(\btheta_{\bS}^{*}))\sigma_{\min}^{+}(\bV_{\bS}(\btheta_{\bS}))}}.
    \end{aligned}
\end{equation*}
\end{corollary}

\begin{remark}
    In this corollary, the optimization error is obtained under the feature coverage Assumption \ref{ass:span space}, which results in the $\epsilon_{n}$ appeared in the upper bound. 
\end{remark}
    Moreover, we can prove that for all empirical stationary points (Proposition \ref{pro:first order optimum is global optimum}), their  empirical risks are identical, i.e., $J_{\bS}(\btheta)$ is invariant over the set $\{\btheta\mid \nabla_{\btheta}J_{\bS}(\btheta) = 0\}$. Thus, all local optima are global optima. \footnote{Notably, this does not imply PL-inequality \citep{karimi2016linear}, this only serves as a sufficient condition to PL-equality.}  

\paragraph{Concentration error.}
The concentration error quantifies the statistical deviation of the empirical optimal objective value from the expected optimal objective value.
Since the parameter $\btheta_{\bS}^{*}$ in $J_{\bS}(\pi_{\btheta_{\bS}^{*}})$ depends on the dataset $\bS$, the standard concentration inequality does not work.

Fortunately, according to the discussion in above, we can prove that $\nabla_{\btheta}J_{\bS}(\btheta^{*}) = \nabla_{\btheta}J_{\bS}(\btheta_{\bS}^{*})$, so that $J_{\bS}(\pi_{\btheta_{\bS}^{*}})=J_{\bS}(\pi_{\btheta^{*}})$ and $\mE_{\bS}[J_{\bS}(\pi_{\btheta_{\bS}^{*}})] = J(\pi_{\btheta^{*}})$. Therefore, the standard concentration can be applied:
\begin{restatable}{lemma}{concentrationinequality}\rm{(Concentration inequality)}\label{lem:concentration of empirical optimal}
    Under Assumptions \ref{ass:bounded}, \ref{ass:ground truth r} and \ref{ass:span space}, since for any $n\ge 1$ and $\rho \in (0,1)$, the following bound holds with probability at least $1-\rho$ over the random draw of the sample dataset $\bS$, 
    \begin{equation}\label{eq:concentration of exact optima}
        \small
        \begin{aligned}
            |J_{\bS}(\pi_{\btheta_{\bS}^{*}})-J(\pi_{\btheta^{*}})|\le 3RC\sqrt{\frac{2\log(2/\rho)}{n}}.
        \end{aligned}
    \end{equation}
\end{restatable}



\paragraph{Generalization error.}
The generalization error measures the discrepancy between the true expected performance of the learned policy and its empirical performance. This gap arises because the policy $\pi_{\btheta_{\bS}}$ is optimized based on the training set $\bS$, rendering the empirical objective a potentially biased estimate of the true objective.

To rigorously bound this error, we adopt the framework of \textit{algorithmic stability}. The central premise is that if a learning algorithm is insensitive to small perturbations in the dataset, it prevents overfitting and ensures that the empirical objective is a reliable proxy for the expected objective.

For algorithm $\cH: \bS\rightarrow \mathbb{R}^{d}$, we first formally define the generalization error and the notion of uniform stability used in our analysis.

\begin{definition}[Generalization Error]
    For learned parameters $\cH(\bS) = \btheta_{\bS}$, the generalization error is the gap between its empirical/population risks: 
\begin{equation}
\small
    \left|  J(\pi_{\btheta_{\bS}}) - J_{\bS}(\pi_{\btheta_{\bS}}) \right|.
\end{equation}
\end{definition} 


\begin{definition}[Uniform Stability \cite{bousquet2002stability}]\label{def:uniform stability}
    An algorithm $\mathcal{H}$ is $\cE_{\rm{stab}}$-uniformly stable if for any two datasets $\bS, \bS^{\prime} \in \mathcal{X}^{n}$ that differ by at most one example, the uniform stability is:
    \begin{equation}
    \small
       \cE_{\rm{stab}}(\cH(\bS)) =  \sup_{x \in \mathcal{X},\bS,\bS^{\prime}} \left| f_{\mathcal{H}(\bS)}(x) - f_{\mathcal{H}(\bS')}(x) \right| .
    \end{equation}
\end{definition}
It has been proven that uniform stability implies  generalization, which is the theoretical foundation for our subsequent analysis in Sections 4 and 5. 

\begin{proposition}[\citet{bousquet2020sharper}]\label{pro:generalization bound}
    If the algorithm $\cH$ satisfies uniform stability and $|f_{\btheta}(\cdot)|\le 3RC$, we have that for any $\rho \in (0,1)$, with probability at least $1-\rho$,
    \begin{equation}
        \small
        \begin{aligned}
            \left|J(\pi_{\btheta_{\bS}}) \!-\! J_{\bS}(\pi_{\btheta_{\bS}})\right| 
             \lesssim \cE_{\rm{stab}}\log(n)\log\left(\frac{1}{\rho}\right)\!+\!6RC \sqrt{\frac{\log\left(1/\rho\right)}{n}}.
        \end{aligned}
    \end{equation}
\end{proposition}

\section{Suboptimality Gap of Empirical Stationary Points}\label{sec:suboptimality of empirical optimum}

In this section, we first analyze the performance of the empirical stationary points, i.e., $\btheta_{\bS}^{*}\in\Theta_{\bS}$. According to the discussion in Section \ref{sec:suboptimality},  
the critical point is to bound the generalization error. 

To this end, we analyze the error in two distinct regimes: the ideal regime with sufficient prompt coverage (Assumption \ref{ass:positive definite}), and a more realistic regime where prompts are insufficient to cover the entire feature space (Assumption \ref{ass:span space}).

\subsection{RL with Sufficient Prompts}\label{sec:rl with sufficient prompt}

Standard machine learning theory typically relies on the assumption that the training data provides sufficient coverage of feature patterns, thereby ensuring that the learned model can accurately extrapolate to the entire feature space. In linear settings, this requirement is formally characterized by the positive definiteness of the feature covariance matrix. 
\begin{assumption}\label{ass:positive definite}
	For any $\btheta$, $x$ and some $\sigma$, it holds $\mE_{x}\left[\Var_{\pi_{\btheta}(a\mid x)}[\bphi(x, a)]\right] = \mE_{x}[\bV_{x}(\btheta)]\succeq\sigma\bI$. 
\end{assumption}
\par
As can be seen, this Assumption implies the expected covariance matrix is positively-definite, meaning its column space covers the whole Euclidean space, thereby covering the feature space as well. Within this coverage assumption to the feature space, we have the following theorem, which states that with sufficient prompts, the empirical stationary point is exactly the ground-truth parameters with high probability. 
\begin{restatable}{theorem}{generalizationwithsufficientprompts}\label{thm:suffcient prompts to recover optimal para}
	Under Assumption \ref{ass:positive definite}, when
	\begin{equation}
		\small
		n > \frac{128C^{4}}{\sigma^{2}}\left(\log{\left(\frac{1}{\rho}\right)} + d\log{\left(\frac{432RC^{5}}{\sigma^{2}}\right)}\right), 
		\end{equation} 
	it holds that with probability at least $1 - \rho$, the ground truth parameter $\btheta^{*}$ is the unique empirical stationary point with the region $\Theta_{R}$. 
\end{restatable} 
Our result indicates that with sufficient data coverage, the empirical stationary point achieves exact identification of the ground truth $\btheta^{*}$ with high probability.

Notably, this result relies on the assumption that feature vectors span the entire Euclidean space. However, in practice, due to limited prompt diversity, feature vectors typically reside on a low-dimensional manifold. This motivates our analysis of the insufficient coverage regime in the following.

\subsection{RL without Sufficient Prompts}\label{sec:rl without sufficient prompts}


As discussed at the end of Section \ref{sec:rl with sufficient prompt}, the strict feature coverage assumption is often not met in practical scenarios. In regimes where the feature vectors fail to span the entire space, the exact identification of the ground truth parameter becomes mathematically impossible.
Thus, instead of pursuing the exact identification of the ground truth parameter, we shift our focus to analyzing the \textbf{suboptimality gap} of the policy induced by the learned parameters.
 
We begin by establishing a standard regularity condition on the feature representations within the action space. To ensure that the responses associated with a given prompt provide distinct feature signals, we assume the following:


\begin{assumption} \label{ass:linear indep}
    For a given $x$ and its corresponding response $\{a_{i}\}_{i=1}^{|\cA|}$, $|\cA| \leq d$\footnote{We can relax the assumption on $\cA$ to $\cA_{x}$ depends on $x$, which does not influence our conclusions in sequel.}, the set of feature vectors $\{\bphi(x,a_{i})|i=1,2,\cdots,|\cA|\}$ are linearly independent. 
\end{assumption}


Next, we introduce the core assumption of this paper. We characterize the feature vector $\bphi(x, a)$ by a direct sum decomposition, where part of it is from the column space of empirical covariance matrix of feature.
\begin{assumption}[\textbf{Feature Coverage}]\label{ass:span space}
    For any prompt $x$, action $a$, and training set $\bS$, the feature vector $\bphi(x, a)$ admits the following decomposition: 
    \[\small
    \bphi(x, a) = \bV_{\bS}(\btheta_{\bS}^{*})\bb_{\bS}(x, a) + \br_{\bS}(x, a),\] 
    where the residual $\br_{\bS}(x,a)$ satisfies  $\sup_{x,a}\|\br_{\bS}(x, a)\|\leq \epsilon_{n}$ with $\epsilon_{n} \leq \cO(n^{-1})$ when $n\to \infty$.  
\end{assumption}
Due to the fundamental theorem of linear algebra \citep{strang2012linear}, the $\br_{\bS}(x, a)$ can be located in $\mathrm{N}(\bV_{\bS}(\btheta_{\bS}^{*}))$ since $\bV_{\bS}(\btheta_{\bS}^{*})\bb_{\bS}(x, a)$$\in\mathrm{C}(\bV_{\bS}(\btheta_{\bS}^{*}))$ and $\bbR^{d} = \mathrm{N}(\bV_{\bS}(\btheta_{\bS}^{*}))\oplus \mathrm{C}(\bV_{\bS}(\btheta_{\bS}^{*}))$, where $\oplus$ is the direct sum. 

This assumption implies that although the empirical covariance matrix 
\[
\bV_{\bS}(\btheta_{\bS}^{*}) = \frac{1}{n}\sum_{i=1}^{n}\Var_{\pi_{\btheta_{\bS}^{*}}}(\bphi(a, x))
\]
fails to cover the entire feature space. However, as the diversity of collected prompts expands with increasing $n$, the column space of the empirical covariance matrix can approximately cover the feature space $\{\bphi(x, a)| (x, a)\in\cX\times\cA\}$, and the magnitude of the residual component is decreasing with sample size $n$. In fact, due to Lemma \ref{lem:general same span}, $\rm{C}(\bV_{\bS}(\btheta_{\bS}^{*}))$ is invariant over $\btheta$. Thus, it can be viewed as the feature accessed during the training process, which is required to approximately cover feature space $\{\bphi(x, a)| (x, a)\in\cX\times\cA\}$.

Building on the Assumption \ref{ass:span space}, we analyze the algorithmic stability $\cE_{\rm{stab}}(\btheta_{\bS}^{*})$ of the empirical stationary point $\btheta_{\bS}^{*}$ to get the generalization bound of $\btheta_{\bS}^{*}$.

\appendixfalse
\begin{restatable}{theorem}{generalizationoffirstorderoracle}\label{thm:generalization of exact optimization}
Under Assumptions \ref{ass:bounded}, \ref{ass:ground truth r} and \ref{ass:span space}, we have 
\begin{equation}
    \small
    \cE_{\rm{stab}}(\btheta_{\bS}^{*}) \leq \left(\frac{\Gamma_{1}}{n}+2RC\sqrt{\frac{\Gamma_{1}}{n}}\right),
\end{equation}
where \[\small\Gamma_{1}=\sup_{\bS,\bS^{\prime}}\left\{4nR\epsilon_{n}+\frac{8RC^{3}\cdot \sigma_{\max}(\bV_{\bS}(\btheta_{\bS}^{*}))}{\sigma_{\min}^{+}(\bV_{\bS}(\btheta_{\bS}^{*}))\sigma_{\min}^{+}(\bV_{\bS}(\btheta_{\bS^{\prime}}^{*}))}\right\}.\]
Then, for any $\rho \in (0,1)$, the following bound holds with probability at least $1-\rho$ over the randomness of $\bS$,  
    \begin{equation}
        \small
        \begin{aligned}
            &\left|J_{\bS}(\pi_{\btheta_{\bS}^{*}})-J(\pi_{\btheta_{\bS}^{*}})\right|
            \\
             & \le \cE_{\rm{stab}}(\btheta_{\bS}^{*})\log (n)\log\left(\frac{1}{\rho}\right)+6RC\sqrt{\frac{\log(1/\rho)}{n}} 
            \\
            & = \tilde{\cO}\left(\frac{1}{\sqrt{n}}\right).
        \end{aligned}
    \end{equation}
\end{restatable}

As detailed in Appendix \ref{app:sec_insufficient prompts}, we observe that under the key feature coverage Assumption \ref{ass:span space}, the generalization error bound is of order $\tilde{\cO}(n^{-\frac{1}{2}})$ and the dominated constant factor $\Gamma_{1}$ depends on the ``conditional number'' $\frac{\sigma_{\max}(\bV_{\bS}(\btheta_{\bS}^{*}))}{\sigma_{\min}^{+}(\bV_{\bS}(\btheta_{\bS}^{*}))}$ and the magnitude of the residual term $\br_{\bS}$.


The estimation of the algorithmic stability coefficient hinges on Lemma \ref{lem:key lemma}, which explicitly bridges the pointwise objective difference with the empirical gradient norms. The core intuition is that both empirical stationary points, $\btheta^{*}_{\bS}$ and $\btheta^{*}_{\bS^{\prime}}$, exhibit negligible gradient norms on the empirical objective $J_{\bS}$. Consequently, this smallness in gradient norms directly translates into a tight stability coefficient via Lemma \ref{lem:key lemma}



Building upon the generalization error bound for $\btheta_{\bS}^{*}$ established in Theorem \ref{thm:generalization of exact optimization}, we can now proceed to bound the suboptimality gap with respect to the ground truth $\btheta^{*}$. 


\begin{restatable}{theorem}{errorboundoffirstorderoracle}\label{thm:error bound under exact optimization}
Under Assumption \ref{ass:bounded}, \ref{ass:ground truth r}, \ref{ass:span space}, for any $\rho \in (0,1)$, the following bound holds with probability at least $1-\rho$ over the randomness of $\bS$, 
    \begin{equation}\label{eq:generalization of exact optima}
        \small 
        \begin{aligned}
            &|J(\pi_{\btheta_{\bS}^{*}})-J(\pi_{\btheta^{*}})| 
            \\ \le & \left(\frac{\Gamma_{1}}{n}+2RC\sqrt{\frac{\Gamma_{1}}{n}}\right)\log (n)\log\left(\frac{2}{\rho}\right)
            +6RC\sqrt{\frac{\log(2/\rho)}{n}}
            \\ &+3RC\sqrt{\frac{2\log(4/\rho)}{n}}
            \\ = &\tilde{\cO}\left(n^{-\frac{1}{2}}\right).
        \end{aligned}
    \end{equation} 
\end{restatable}

	This result demonstrates that under the posterior Boltzmann policy model structure, we achieve a dimension free suboptimality gap of order $\tilde{\cO}(n^{-\frac{1}{2}})$, even when the training data provides limited coverage of the feature space. The constant term $\Gamma_{1}$ dependence is adopted from the analysis to generalization error. 
	\par
    The intuitive explanation for this dimension-independent bound is that the feature patterns learned by the model are determined by the linear subspace spanned by the empirical covariance matrix $\bV_{\bS}$. Under the feature coverage assumption, as the sample size $n$ increases, this subspace progressively expands. Consequently, the feature patterns of unseen data are effectively covered by the patterns already captured in the training set, implying that the dimensions orthogonal to the covered subspace make negligible contribution to the generalization error, thereby decoupling the suboptimality gap from the dimension $d$.

\section{Suboptimality Gap under Gradient-Based Algorithms}\label{sec:suboptimality of gradient based algorithms}

In the previous section, we analyzed the suboptimality gap of the empirical stationary point $\btheta_{\bS}^{*}$. However, in practice, we can only approximate $\btheta_{\bS}^{*}$ instead of exactly find it. The approximation process relies on the iterative gradient-based algorithms, such as Gradient Ascent (GA) or Stochastic Gradient Ascent (SGA). In this section, we extend our analysis to characterize the suboptimality gap of the policy learned by these gradient-based algorithms.
	
	
	As clarified in Section \ref{sec:suboptimality}, the suboptimality gap depends on generalization error, optimization error, and concentration error, where the last one is handled in Lemma \ref{lem:concentration of empirical optimal}. Thus, we focus on the first two terms in sequel.
	\subsection{The Suboptimality Gap under GA}\label{sec:suboptimality of GA}
	We first analyze GA, which has the following update rule for each iteration $t$:
	\begin{equation} \label{eq:GD update rule}
		\small
		\btheta_{t + 1} = \btheta_{t} + \eta\nabla_{\btheta}J_{\bS}(\pi_{\btheta_{t}}),
	\end{equation} 
	where $\eta > 0$ is the learning rate. 
	\begin{definition}[GA algorithm output]\label{def:GA output}
		Let $T$ be the number of iterations, $\small \Theta_{T}=\{\btheta_{t}\}_{t=1}^{T}$ be the set of iterations obtained by \eqref{eq:GD update rule}. We define $\btheta_{\bS,T}^{\rm{GA}}$ as:
		\begin{align*}
			\small
			\btheta_{\bS,T}^{\rm{GA}}= \arg\min_{\btheta \in \Theta_{T}}\{\|\nabla_{\btheta}J_{\bS}(\pi_{\btheta})\|\}.
		\end{align*}
	\end{definition}
	We start with the optimization error. The standard result in non-convex optimization theory (e.g., Theorem 3.7 in \cite{bubeck2015convex}) shows that under the L-smoothness condition proved in Lemma \ref{lem:smoothness coefficient}, $\btheta_{\bS,T}^{\rm{GA}}$ converges to the stationary point as in the following lemma.
	\begin{restatable}{lemma}{convergencerateofGA}\label{lem:convergence of GA}
	    By taking $\eta = \frac{1}{L_{f}}$ ($L_{f}$ is defined in Lemma \ref{lem:smoothness coefficient}), the parameter $\btheta_{\bS,T}^{\rm{GA}}$ satisfies
	    \begin{equation}
		      \small
		      \|\nabla_{\btheta}J_{\bS}(\pi_{\btheta_{\bS,T}^{\rm{GA}}})\|^{2} \leq \frac{12L_{f}RC}{T}.
	    \end{equation}
	\end{restatable}
	This lemma shows that GA approximates the empirical stationary point. As discussed in Section \ref{sec:suboptimality}, we can prove that small empirical gradient norm leads to a small optimization error as detailed in the following lemma.
    \begin{restatable}{lemma}{optimizationofGA}\label{lem:optimization of GA}
		Under Assumptions \ref{ass:bounded}, \ref{ass:ground truth r} and \ref{ass:span space}, for any $\btheta_{\bS}^{*}\in \Theta_{\bS}$, and $\btheta_{\bS,T}^{\rm{GA}}$ defined in Definition \ref{def:GA output}, we have:
		\begin{equation}\label{eq:optimization of GA}
			\small
			\begin{aligned}
				&\left|J_{\bS}(\pi_{\btheta_{\bS,T}^{\rm{GA}}})-J_{\bS}(\pi_{\btheta_{\bS}^{*}})\right| 
                =\cO\left(T^{-\frac{1}{4}}+n^{-\frac{1}{2}}\right).
			\end{aligned}
		\end{equation}
	\end{restatable}
    The proof of Lemma \ref{lem:optimization of GA} is provided in Section \ref{sec:suboptimality}. This lemma shows that the optimization error of the parameter $\btheta_{\bS,T}^{\rm{GA}}$ obtained via GA is of order $\mathcal{O}(T^{-\frac{1}{4}}+n^{-\frac{1}{2}})$ and the constant factor depends on ``conditional number'' of $\bV_{\bS}(\btheta_{\bS}^{*})$ and the magnitude of the residual term $\br_{\bS}$, respectively. 
    
    Having quantified the optimization error for $\btheta_{\bS,T}^{\rm{GA}}$, we now turn to the generalization error. 
    We adapt the algorithmic stability framework established in Section \ref{sec:rl without sufficient prompts} for empirical stationary points to the analysis of $\btheta_{\bS,T}^{\rm{GA}}$ as follows:
    
	\begin{restatable}{theorem}{generalizationofGA}\label{thm:generalization of GA}
		Under Assumptions \ref{ass:bounded}, \ref{ass:ground truth r} and \ref{ass:span space}, for the GA output $\btheta_{\bS,T}^{\rm{GA}}$ defined in Definition \ref{def:GA output}, we have 
        \[ \small \cE_{\rm{stab}}(\btheta_{\bS,T}^{GA}) \lesssim \sqrt{\epsilon_n} + \left(\sqrt{\Gamma^{\rm{GA}}_{\bS,\bS}}+\sqrt{\Gamma^{\rm{GA}}_{\bS,\bS^{\prime}}}\right) \left( T^{-\frac{1}{4}} + n^{-\frac{1}{2}} \right), \]
        where \[\small \Gamma^{\rm{GA}}_{\bS,\bS^{\prime}}=\sup_{\bS,\bS^{\prime}}\frac{\sigma_{\max}(\bV_{\bS}(\btheta_{\bS}^{*}))}{\sigma_{\min}^{+}(\bV_{\bS}(\btheta_{\bS}^{*}))\sigma_{\min}^{+}(\bV_{\bS}(\btheta_{\bS^{\prime},T}^{\rm{GA}}))}\] and 
        \[\small \Gamma^{\rm{GA}}_{\bS,\bS}=\sup_{\bS}\frac{\sigma_{\max}(\bV_{\bS}(\btheta_{\bS}^{*}))}{\sigma_{\min}^{+}(\bV_{\bS}(\btheta_{\bS}^{*}))\sigma_{\min}^{+}(\bV_{\bS}(\btheta_{\bS,T}^{\rm{GA}}))}.\] 
        
       Then, for any $\rho \in (0,1)$, the following bound holds with probability at least $1-\rho$ over the randomness of $\bS$:
		\begin{equation}\label{eq:generalization of GA}
			\small
			\begin{aligned}
				&\left|J(\pi_{\btheta_{\bS,T}^{\rm{GA}}})-J_{\bS}(\pi_{\btheta_{\bS,T}^{\rm{GA}}})\right| 
				\\ \le & \cE_{\rm{stab}}(\btheta_{\bS,T}^{GA})\log (n)\log\left(\frac{1}{\rho}\right)+6RC\sqrt{\frac{\log(1/\rho)}{n}} 
				\\ = & \tilde{\cO}\left(T^{-\frac{1}{4}}+n^{-\frac{1}{2}}\right).
			\end{aligned}
		\end{equation}
	\end{restatable}

The proof of Theorem \ref{thm:generalization of GA} is in Appendix \ref{app:sec_GA}. As can be seen, the generalization error of $\btheta_{\bS,T}^{\rm{GA}}$ is of order $\tilde{\cO}(T^{-\frac{1}{4}} + n^{-\frac{1}{2}})$ with constant dependence similar to Theorem \ref{thm:generalization of exact optimization}.   


Interestingly, standard results in \citep{hardt2016train} suggest that in the non-convex setting, operating more iterations leads to worse algorithmic stability, thereby worse generalization. In contrast, our analysis shows the algorithmic stability of $\btheta_{\bS,T}^{\rm{GA}}$ is decreasing with $T$. This phenomenon occurs because, for RLHF problem we explored, the critical Lemma \ref{lem:key lemma} links the gradient norm $\small \|\nabla_{\btheta}J_{\bS}(\btheta_{\bS,T}^{\rm{GA}})\|$ with algorithmic stability, which implies that the small gradient norm of $\btheta_{\bS,T}^{\rm{GA}}$ leads to small algorithmic stability. 
	
Finally, as discussed in Section \ref{sec:suboptimality}, by combining the generalization error bound \eqref{eq:generalization of GA}, the optimization error bound \eqref{eq:optimization of GA}, and the concentration error bound \eqref{eq:concentration of exact optima}, we prove the suboptimality gap of $\btheta_{\bS,T}^{\rm{GA}}$ obtained by GA. 
	\begin{restatable}{theorem}{errorboundofGA}\label{thm:error bound of GA}
		Under Assumptions \ref{ass:bounded}, \ref{ass:ground truth r}, and \ref{ass:span space}, for any  $\rho \in (0,1)$, the following bound holds with probability at least $1-\rho$ over the randomness of $\bS$:
		\begin{equation}
			\small 
			\begin{aligned}
				&\left|J(\pi_{\btheta_{\bS,T}^{\rm{GA}}})-J(\pi_{\btheta^{*}})\right|
				=\tilde{\cO}\left(T^{-\frac{1}{4}}+n^{-\frac{1}{2}}\right).
			\end{aligned}
		\end{equation}
	\end{restatable}
    By taking $T = n^{2}$, i.e., conducting the GA for $n^{2}$ steps, the suboptimality gap of $\btheta_{\bS,T}^{\rm{GA}}$ has the same order of empirical stationary points proved in Section \ref{sec:rl without sufficient prompts}. 

\subsection{The Suboptimality Gap under SGA}\label{sec:suboptimality of SGA}

While GA provides theoretical clarity, computing full-batch gradients is computationally prohibitive in practice. A more realistic algorithm is SGA \citep{harvey2019tight}, which has less computational complexity. The update rule of SGA is:
	\begin{equation} \label{eq:SGD update rule}
		\small
		\btheta_{t+1} = \btheta_{t} + \eta_{t}\nabla_{\btheta}f_{\btheta_{t}}(x_{i_{t}}),
	\end{equation}
where the index $i_t$ is uniformly sampled from $\{1, \dots, n\}$ for each iteration $t$, and $\eta_{t}$ is the learning rate. 



\begin{definition}[SGA Algorithm Output]\label{def:SGA output}
    Let $T$ be the number of iterations, $\Theta_{T}=\{\btheta_{t}\}_{t=1}^{T}$ be the set of iterations obtained by \eqref{eq:SGD update rule}. Similar to GA, we define $\btheta_{\bS,T}^{\rm{SGA}}$ as:
    \begin{equation*}
        \small
        \btheta_{\bS,T}^{\rm{SGA}}=\arg\min_{\btheta \in \Theta_{T}} \left\{\left\|\nabla_{\btheta}J_{\bS}(\pi_{\btheta})\right\|\right\}.
    \end{equation*}
\end{definition}  
\begin{remark}
     We note that identifying $\btheta_{\bS,T}^{\rm{SGA}}$ as Definition \ref{def:SGA output} requires evaluating the full gradient at each step, which incurs additional computational cost. However, this is standard in the theoretical analysis of non-convex stochastic optimization to ensure optimal convergence rates \citep{ghadimi2013stochastic}. While practical implementations often use the last iterate, theoretical guarantees for the last iterate typically require stronger assumptions (e.g., Polyak-Lojasiewicz condition), which are beyond the scope of this work.
\end{remark}
Similar to the results in Lemma \ref{lem:convergence of GA}, we also have the following lemma to characterize the convergence rate of SGA. 
\begin{restatable}{lemma}{convergencerateofSGA}\label{lem:gradient norm of SGA}
    By taking $\eta_{t}=\frac{1}{2L_{f}\sqrt{t}}$\rm{(}$L_{f}$ is defined in Lemma \ref{lem:smoothness coefficient}\rm{)}, for any $\rho^{\prime} \in (0,1)$, with probability at least $1-\rho^{\prime}$, the parameter $\small \btheta_{\bS,T}^{\rm{SGA}}$ satisfies
     \begin{equation}
        \small
        \begin{aligned}
            &\left\|\nabla_{\btheta}J_{\bS}(\pi_{\btheta_{\bS,T}^{\rm{SGA}}})\right\|^{2}
            \\ \le &\frac{6L_{f}}{\sqrt{T}}\left(\left(6RC+\frac{2R^{2}C^{4}}{L_{f}}\right)+\frac{2R^{2}C^{4}}{L_{f}}\log(T)+ \frac{1}{\lambda}\log\left(\frac{1}{\rho^{\prime}}\right)\right).
        \end{aligned}
    \end{equation}
    where $\lambda>0$ satisfies $\frac{e^{\lambda}-\lambda-1}{\lambda}\le\frac{L_{f}}{8R^{2}C^{4}}$.
\end{restatable}
SGA updates rely on gradient estimates with stochastic noise. Thus, unlike the deterministic convergence of GA in Section \ref{sec:suboptimality of GA}, Lemma \ref{lem:gradient norm of SGA} establishes a high-probability bound for $\btheta_{\bS,T}^{\rm{SGA}}$. Similar to Lemma \ref{lem:optimization of GA}, Lemma \ref{lem:gradient norm of SGA} implies the optimization error. Detailed proofs are in Appendix \ref{app:sec_SGA}.

\begin{restatable}{lemma}{optimizationofSGA}\label{lem:optimization of SGA}
 Under Assumptions \ref{ass:bounded}, \ref{ass:ground truth r} and \ref{ass:span space}, for any $\rho^{\prime} \in (0,1)$, with probability at least $1-\frac{\rho^{\prime}}{2}$, $\btheta_{\bS,T}^{\rm{SGA}}$ satisfies:  
    \begin{equation}
        \small
        \begin{aligned}
            &\left|J_{\bS}(\pi_{\btheta_{\bS,T}^{\rm{SGA}}})-J_{\bS}(\pi_{\btheta_{\bS}^{*}})\right| = \tilde{\cO}\left(T^{-\frac{1}{8}}+n^{-\frac{1}{2}}\right).
        \end{aligned}
    \end{equation}
\end{restatable}

Next, we turn to the analysis of generalization error by applying algorithmic stability.
\begin{restatable}{theorem}{generalizationofSGA}\label{thm:generalization of SGA}
 Under Assumptions \ref{ass:bounded}, \ref{ass:ground truth r} and \ref{ass:span space}, for any $\rho^{\prime} \in (0,1)$ and parameter $\btheta_{\bS,T}^{\rm{SGA}}$ , we have
 \begin{equation*}
     \small
     \cE_{\rm{stab}}(\btheta_{\bS,T}^{\rm{SGA}}) \lesssim \sqrt{\epsilon_n} + \left(\sqrt{\Gamma^{\rm{SGA}}_{\bS,\bS}}+\sqrt{\Gamma^{\rm{SGA}}_{\bS,\bS^{\prime}}}\right) \left( T^{-\frac{1}{8}} + n^{-\frac{1}{2}} \right),
 \end{equation*}
 where \[\small\Gamma^{\rm{SGA}}_{\bS,\bS}=\sup_{\bS}\frac{\sigma_{\max}(\bV_{\bS}(\btheta_{\bS}^{*}))}{\sigma_{\min}^{+}(\bV_{\bS}(\btheta_{\bS}^{*}))\sigma_{\min}^{+}(\bV_{\bS}(\btheta_{\bS,T}^{\rm{SGA}}))}\] 
 and \[\small\Gamma^{\rm{SGA}}_{\bS,\bS^{\prime}}=\sup_{\bS,\bS^{\prime}}\frac{\sigma_{\max}(\bV_{\bS}(\btheta_{\bS}^{*}))}{\sigma_{\min}^{+}(\bV_{\bS}(\btheta_{\bS}^{*}))\sigma_{\min}^{+}(\bV_{\bS}(\btheta_{\bS^{\prime},T}^{\rm{SGA}}))}.\]
 Then, for any $\rho \in (0,1)$, the following bound holds with probability at least $1-\rho-\rho^{\prime}$ over the randomness of $\bS$ and $i_{t}$ in SGA:
    \begin{equation}
        \small
        \begin{aligned}
            &\left|J(\pi_{\btheta_{\bS,T}^{\rm{SGA}}})-J_{\bS}(\pi_{\btheta_{\bS,T}^{\rm{SGA}}})\right| 
            \\ \le & \cE_{\rm{stab}}(\btheta_{\bS,T}^{\rm{SGA}})\log (n)\log\left(\frac{1}{\rho}\right)+6RC\sqrt{\frac{\log(1/\rho)}{n}} 
            \\ = & \tilde{\cO}\left(T^{-\frac{1}{8}}+n^{-\frac{1}{2}}\right).
        \end{aligned}
    \end{equation}
\end{restatable}
The proof of Theorem \ref{thm:generalization of SGA} is provided in Appendix \ref{app:sec_SGA}. The proof strategy aligns with the analysis of GA in Section \ref{sec:suboptimality of GA}, with the difference in the bound's order stemming from the slower convergence rate of the stochastic algorithm.


\begin{figure*}[t!] 
    \centering
    \begin{subfigure}[b]{0.24\textwidth}
        \centering
        \includegraphics[width=\linewidth]{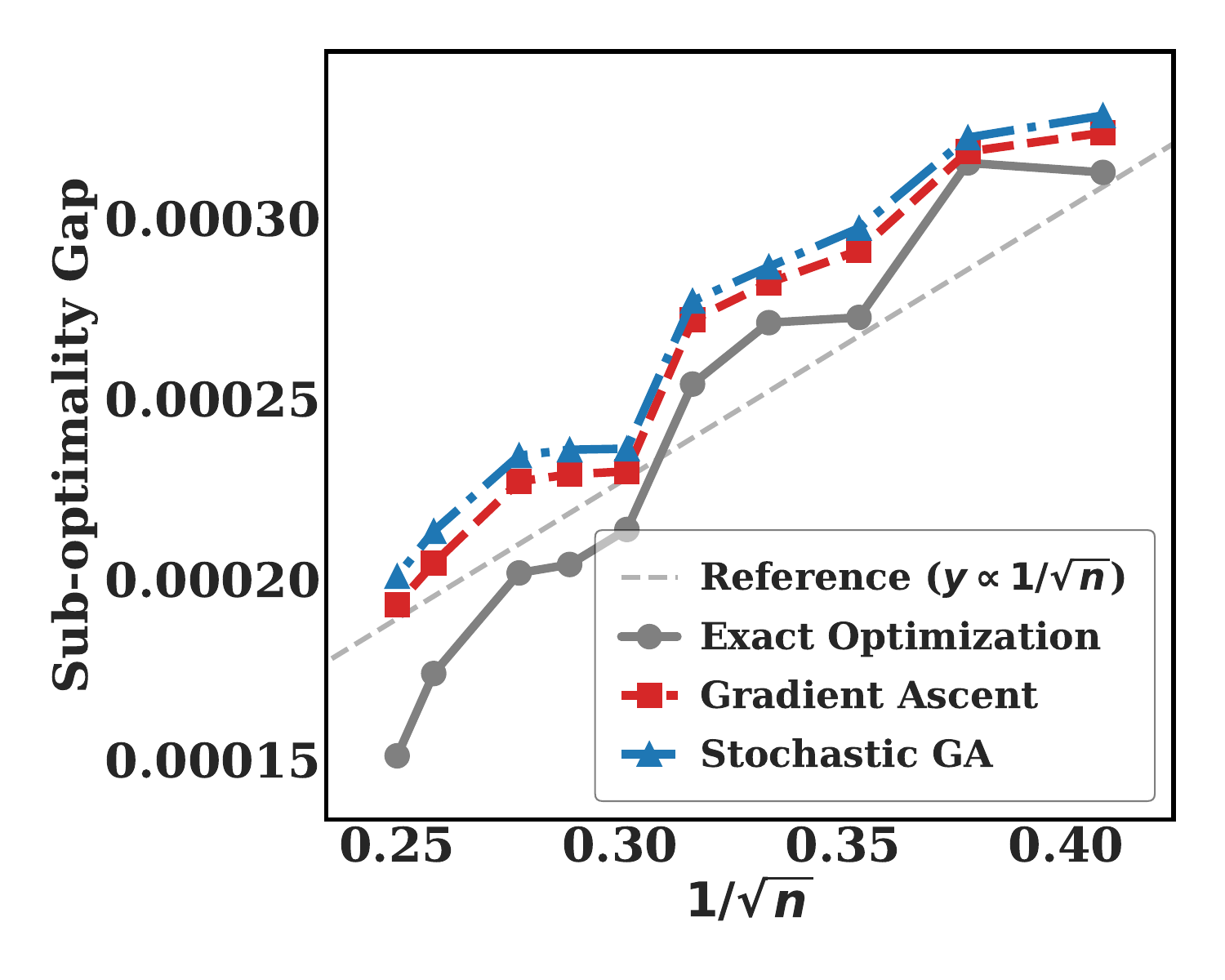}
        \vspace{-6mm}
        \caption{$d=24$}
        \label{fig:img1}
    \end{subfigure}
    \hfill 
    \begin{subfigure}[b]{0.24\textwidth}
        \centering
        \includegraphics[width=\linewidth]{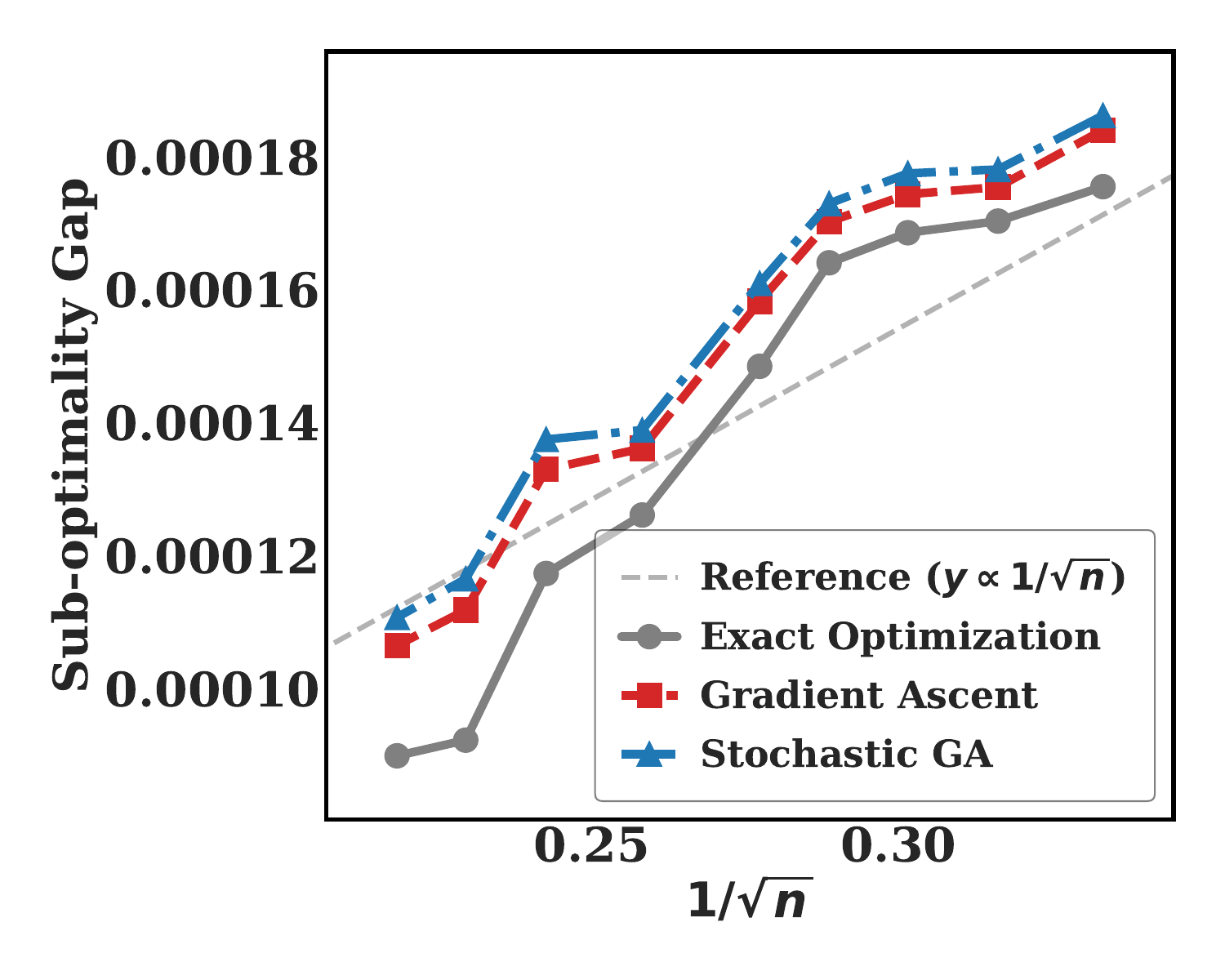}
        \vspace{-6mm}
        \caption{$d=32$}
        \label{fig:img2}
    \end{subfigure}
    \hfill 
    \begin{subfigure}[b]{0.24\textwidth}
        \centering
        \includegraphics[width=\linewidth]{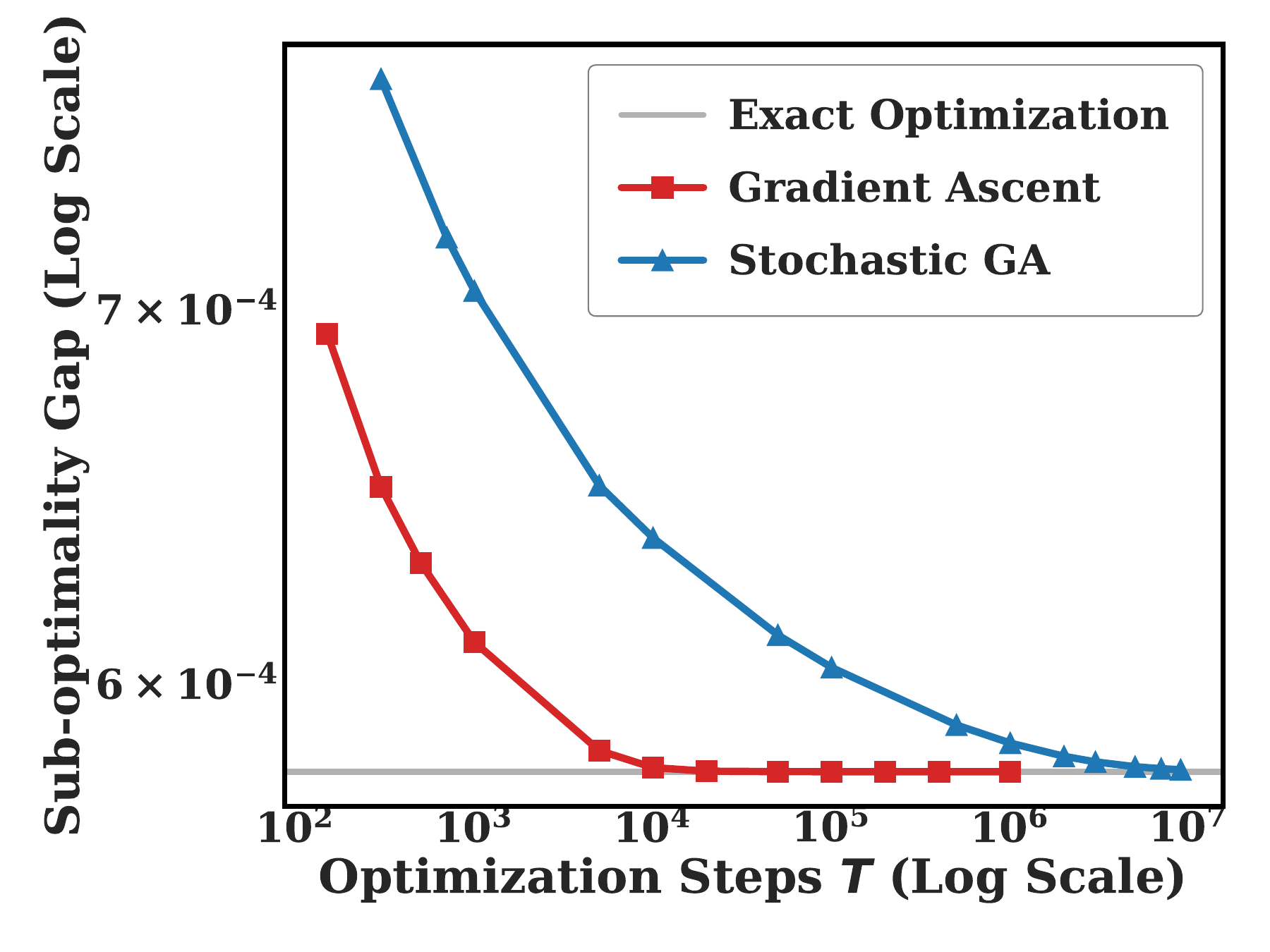}
        \vspace{-6mm}
        \caption{$n=7$}
        \label{fig:img3}
    \end{subfigure}
    \hfill 
    \begin{subfigure}[b]{0.24\textwidth}
        \centering
        \includegraphics[width=\linewidth]{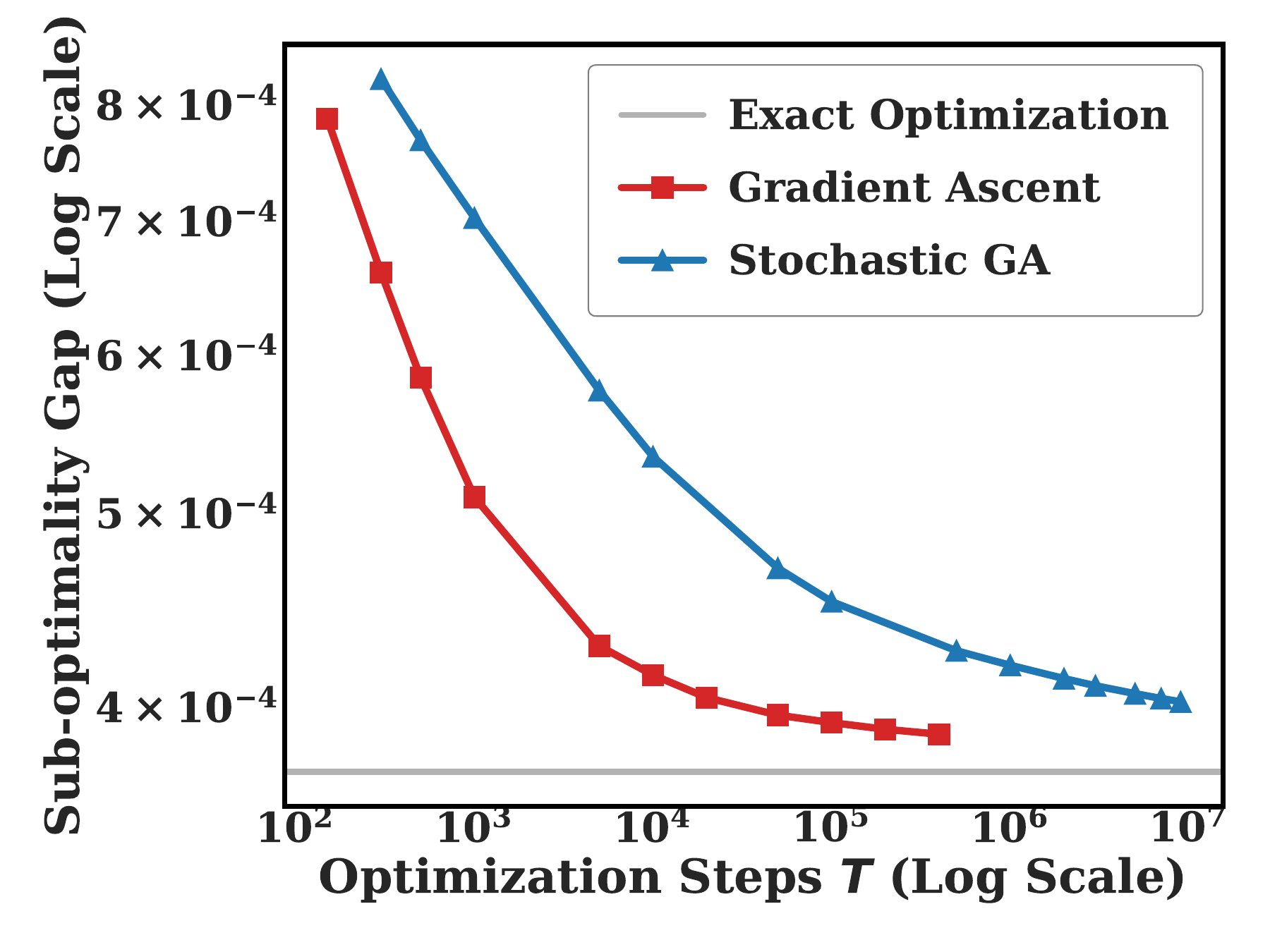}
        \vspace{-6mm}
        \caption{$n=10$}
        \label{fig:img4}
    \end{subfigure}
    \vspace{-1mm}
    \caption{The suboptimality gap under the feature coverage assumption. The parameters are empirical stationary points or obtained by gradient-based methods. \textbf{Left panels (\ref{fig:img1}-\ref{fig:img2}):} The suboptimality gap plotted against $n^{-\frac{1}{2}}$.  \textbf{Right panels (\ref{fig:img3}-\ref{fig:img4}):} The suboptimality gap w.r.t. optimization steps $T$ in log-log scale.}
    \label{fig:suboptimality gap}
    \vspace{-5mm}
\end{figure*}

Finally, we prove the suboptimality gap of $\btheta_{\bS,T}^{\rm{SGA}}$.
\begin{restatable}{theorem}{errorboundofSGA}\label{thm:error bound of SGA}
Under Assumptions \ref{ass:bounded}, \ref{ass:ground truth r} and \ref{ass:span space}, for any $\rho,\rho^{\prime} \in (0,1)$, the following bound holds with probability at least $1-\rho-\rho^{\prime}$ over the randomness of $\bS$:
    \begin{equation}
        \small
        \begin{aligned}
            &\left|J(\pi_{\btheta_{\bS,T}^{\rm{SGA}}})-J(\pi_{\btheta^{*}})\right|
            = \tilde{\cO}\left(T^{-\frac{1}{8}}+n^{-\frac{1}{2}}\right). 
        \end{aligned}
    \end{equation}
\end{restatable}
The proof proceeds according to the sub-optimality decomposition defined in Section \ref{sec:suboptimality}. By taking $T=n^{4}$, the suboptimality gap of $\btheta_{\bS,T}^{\rm{SGA}}$ has the same order of empirical stationary points in Section \ref{sec:rl without sufficient prompts}.
\begin{remark}
    Notably, to obtain the algorithmic stability of SGA, we use the high-probabilistic algorithmic stability \citep{yuan2023l_2} mentioned in Appendix \ref{app:sec_SGA}, since SGA is proved to be stable with high probability. 
\end{remark}


\section{Experiment}\label{sec:experiments}
In this section, we present simulation experiments to empirically verify our theoretical findings. 

\subsection{The Results Positive Definite Covariance Matrix}\label{subsec:parameter identifiable}
We first verify the results under Assumption \ref{ass:positive definite}, i.e., the expected covariance matrix of feature is positive-definite. Our Theorem \ref{thm:suffcient prompts to recover optimal para} indicates that under this assumption, the empirical stationary point is the ground-truth parameters.
\par
Lemma \ref{lem:gradient} admits that $\nabla_{\btheta}J_{\bS}(\pi_{\btheta})=\bV_{\bS}(\btheta)(\btheta^{*}-\btheta)$. Thus, when $\bV_{\bS}(\btheta)$ is positive-definite matrix for all $\btheta$, the empirical stationary point becomes the ground truth parameters. Thus, verifying the positive-definite property of $\bV_{\bS}(\btheta_{\bS}^{*})$ is enough. The setup of our example refers to Appendix \ref{app:positive definite covariance}.


\begin{figure}[ht]
    \centering
    \begin{subfigure}[b]{0.48\linewidth}
        \centering
        \includegraphics[width=\linewidth]{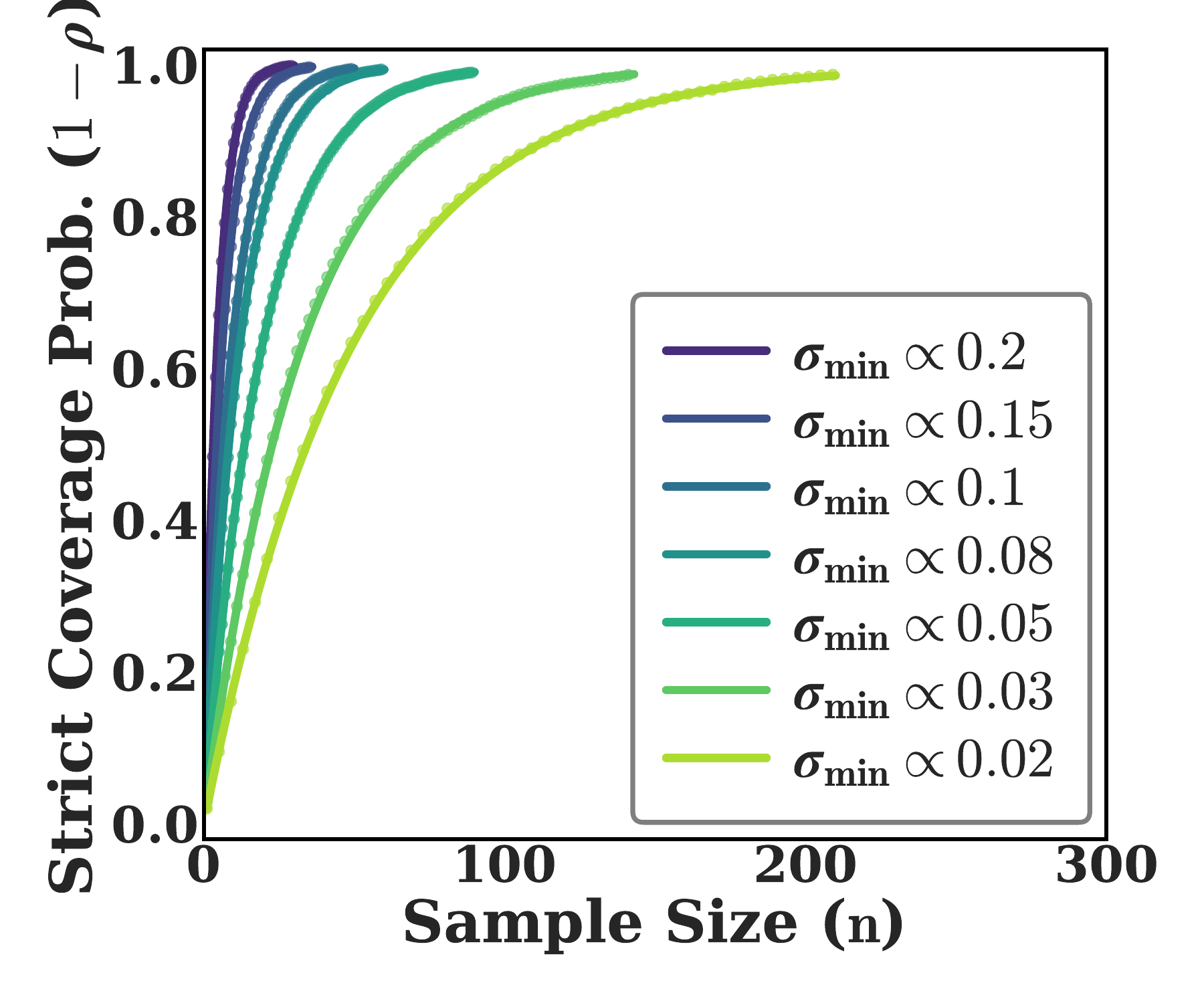}
        \label{fig:data_complexity}
    \end{subfigure}
    \hfill 
    \begin{subfigure}[b]{0.48\linewidth}
        \centering
        \includegraphics[width=\linewidth]{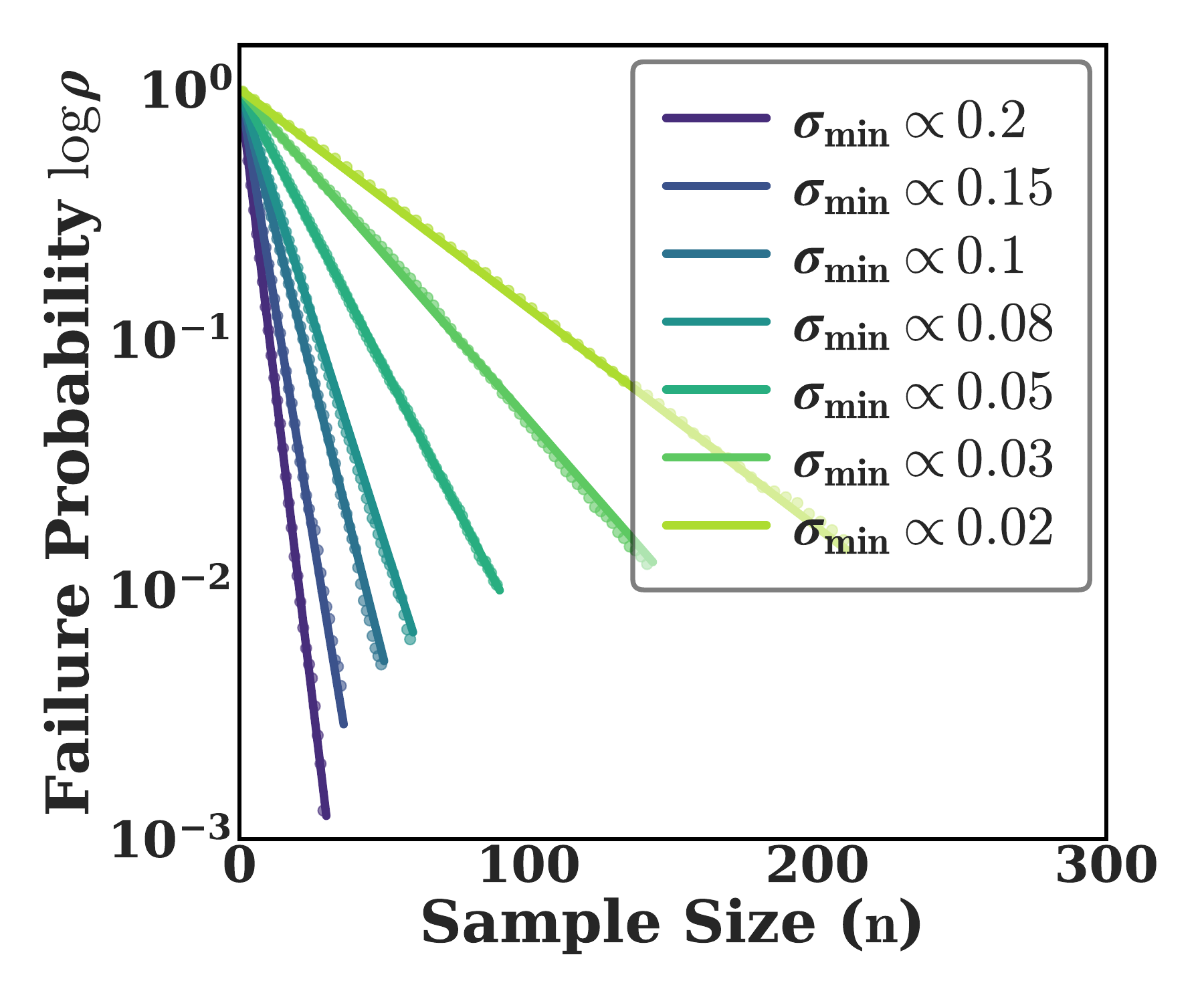}
        \label{fig:coverage_log}
    \end{subfigure}
    
    \vspace{-5mm} 
    \caption{The required samples to make empirical stationary points becomes ground truth parameters.} 
    \label{fig:full_coverage_results}
    \vspace{-3mm} 
\end{figure}

\paragraph{Results:}
In Figure \ref{fig:full_coverage_results}, we report the percentage of positive-definite matrices under varied $\sigma_{\min} = \min_{\btheta}\sigma_{\min}(\mathbb{E}_{x}[\bV_{x}(\btheta)])$. According to our conclusion in Theorem \ref{thm:suffcient prompts to recover optimal para}, the percentage of non-positive definite matrices $\rho \propto e^{-n} / \sigma_{\min}$. The relationship is verified in Figure \ref{fig:full_coverage_results}.   

\subsection{The Results under Feature Coverage Assumption}\label{subsec:insufficient coverage}
Next, we focus our analysis on the results under our key feature coverage Assumption \ref{ass:span space}. Theorems \ref{thm:error bound under exact optimization}, \ref{thm:error bound of GA} and \ref{thm:error bound of SGA} indicate that under this assumption, the suboptimality gap is guaranteed to be small for both empirical stationary points and model obtained by gradient-based methods. The detailed experimental setup is in Appendix \ref{app:feature coverage}

\paragraph{Suboptimality Gap Scaling with sample size $n$.} We verify the proved dimension-independent generalization bound of order $\mathcal{O}(n^{-\frac{1}{2}})$ in Theorem \ref{thm:generalization of exact optimization}, \ref{thm:generalization of GA}, \ref{thm:generalization of SGA}. For problems in dimensions $d = 24, d = 32$, the suboptimality gap in Figure \ref{fig:suboptimality gap} verify the order of $\mathcal{O}(n^{-\frac{1}{2}})$. 

\paragraph{Suboptimality Gap vs. steps $T$.} On the other hand, our conclusions in Theorems \ref{thm:generalization of GA} and \ref{thm:generalization of SGA} show that the generalization error is decreasing with the optimization steps. The results in Figure \ref{fig:suboptimality gap} verify our conclusions as well. 

\section{Concluding Discussion}
In this paper, we have established an end-to-end theoretical framework for RLHF based on algorithmic stability. Departing from existing paradigms grounded in the consistency of reward parameter estimation, we prove that under a key \textbf{feature coverage} condition, the empirical optima of the policy model have a suboptimality gap of order $\mathcal{O}(n^{-\frac{1}{2}})$. Furthermore, we extend our analysis to gradient-based algorithms, specifically Gradient Ascent and Stochastic Gradient Ascent , deriving explicit suboptimality bounds for the policies obtained by these algorithms. These results theoretically confirm that RLHF processes can learn generalizable policies solely from limited feedback, providing new theoretical evidence for the empirically observed success of RLHF.

\section*{Impact Statement}
This work establishes a theoretical framework for the generalization of RLHF. Positively, our derivation of dimension-independent bounds validates that alignment can be achieved with high sample efficiency under sufficient feature coverage. This theoretical insight supports the development of more data-efficient training pipelines, potentially reducing the computational footprint of large-scale model development. However, practitioners must be cognizant of the limitations inherent in our analysis. The reliance on the linear reward assumption implies that our bounds may not fully capture the risks in scenarios involving complex, non-linear human preferences. Applying these theoretical guaranties to safety-critical domains without accounting for model misspecification could lead to an underestimation of alignment errors. Furthermore, our results on algorithmic stability suggest that the learned policy will faithfully capture the distributional properties of the covered features. Consequently, if the training data contains societal biases, the stability of the algorithm implies that these biases will be propagated to the deployed model, necessitating rigorous data auditing alongside algorithmic improvements.

\bibliography{example_paper}
\bibliographystyle{icml2026}

\newpage
\appendix
\onecolumn



\vspace*{0.2cm} 
\begin{center}
    {\LARGE \textbf{Appendix}}
\end{center}

\etocdepthtag.toc{mtappendix}

\etocsettagdepth{mtchapter}{none}     
\etocsettagdepth{mtappendix}{subsection} 

\etocsettocstyle
    {} 

\tableofcontents

\vspace{2em}
\hrule
\vspace{2em}





\section{One Critical Lemma}
Firstly, we illustrate the critical lemma in this paper. 
\begin{lemma}\label{lem:key lemma}
Under Assumptions \ref{ass:span space}, for any $\btheta_{1},\btheta_{2}\in \Theta_{R}$ and positive constants $\delta_{1},\delta_{2}$, if we have 
    \begin{align*}
    \small
        \|\nabla_{\btheta}J_{\bS}(\pi_{\btheta_{i}})\| \le \delta_{i},\ i\in \{1,2\},
    \end{align*}
then for any $x \in \cA$, it holds that
    \begin{equation}
    \small
        \begin{aligned}
            |f_{\btheta_{1}}(x) - f_{\btheta_{2}}(x)| \le & 4R\epsilon_{n}+\frac{2C\delta_{1}\cdot \sigma_{\max}(\bV_{\bS}(\btheta_{\bS}^{*}))}{\sigma_{\min}^{+}(\bV_{\bS}(\btheta_{\bS}^{*}))\sigma_{\min}^{+}(\bV_{\bS}(\btheta_{1}))}+\frac{2C\delta_{2}\cdot \sigma_{\max}(\bV_{\bS}(\btheta_{\bS}^{*}))}{\sigma_{\min}^{+}(\bV_{\bS}(\btheta_{\bS}^{*}))\sigma_{\min}^{+}(\bV_{\bS}(\btheta_{2}))}
            \\ &+2CR\sqrt{4R\epsilon_{n}+\frac{2C\delta_{1}\cdot \sigma_{\max}(\bV_{\bS}(\btheta_{\bS}^{*}))}{\sigma_{\min}^{+}(\bV_{\bS}(\btheta_{\bS}^{*}))\sigma_{\min}^{+}(\bV_{\bS}(\btheta_{1}))}+\frac{2C\delta_{2}\cdot \sigma_{\max}(\bV_{\bS}(\btheta_{\bS}^{*}))}{\sigma_{\min}^{+}(\bV_{\bS}(\btheta_{\bS}^{*}))\sigma_{\min}^{+}(\bV_{\bS}(\btheta_{2}))}}.
        \end{aligned}
    \end{equation}
\end{lemma}
Before proving this lemma, we need the following lemmas. 

\begin{lemma}\label{lem:energy gap}
    Let $Z_{\btheta}(x) = \log\mE_{\pi_{\rm{ref}}(a\mid x)}[\exp(\langle\btheta, \bphi(a, x)\rangle)]$, we have 
    \begin{equation}
        \small
        Z_{\btheta_{1}}(x) - Z_{\btheta_{2}}(x) = \left\langle\int_{0}^{1}\mE_{\pi_{\btheta_{t}}(a\mid x)}[\bphi(a, x)], \btheta_{1} - \btheta_{2}\right\rangle,
    \end{equation}
    for any $\btheta_{1}$ and $\btheta_{2}$, with $\btheta_{t} = t\btheta_{1} + (1 - t)\btheta_{2}$. 
\end{lemma}
\begin{proof}
    Define the function $h(t) = \mE_{\pi_{\rm{ref}}(a\mid x)}\left[\exp(\langle\btheta_{1} + t(\btheta_{2} - \btheta_{1}), \bphi(a, x)\rangle)\right]$, which satisfies $\log{h(0)} = Z_{\btheta_{1}}(x)$ and $\log{h(1)} = Z_{\btheta_{2}}(x)$. Then we see 
    \begin{equation}
        \small
        \begin{aligned}
            \log{h(1)} - \log{h(0)} & = \int_{0}^{1}\frac{h^{\prime}(t)}{h(t)}dt \\
            & = \int_{0}^{1}\frac{\mE_{\pi_{\rm{ref}}(a\mid x)}\left[\exp(\langle\btheta_{1} + t(\btheta_{2} - \btheta_{1}), \bphi(a, x)\rangle)\right\langle\btheta_{1} - \btheta_{2}, \bphi(a, x)\rangle]}{\mE_{\pi_{\rm{ref}}(a\mid x)}\left[\exp(\langle\btheta_{1} + t(\btheta_{2} - \btheta_{1}), \bphi(a, x)\rangle)\right]} \\
            & = \left\langle\int_{0}^{1}\mE_{\pi_{\btheta_{t}}(a\mid x)}[\bphi(a, x)], \btheta_{1} - \btheta_{2}\right\rangle,
        \end{aligned}
    \end{equation}
    which proves our conclusion.  
\end{proof}

\begin{lemma}\label{lem:general stability gap step I}
    Under Assumption \ref{ass:bounded}, for a given prompt $x$ and any two parameters $\btheta_{1},\btheta_{2} \in \Theta_{R}$. We have 
    \begin{equation}
        \small
        \begin{aligned}
             \left|f_{\btheta_{1}}(x) - f_{\btheta_{2}}(x)\right|
            \leq & 2CR\sqrt{\left\langle \mE_{\pi_{\btheta_{1}}}\bphi(x, a), \btheta_{1} - \btheta_{2}\right\rangle + \left\langle\int_{0}^{1}\mE_{\pi_{\btheta_{t}}}[\bphi(a, x)], \btheta_{1} - \btheta_{2}\right\rangle} \\
             + & \left|\left\langle\mE_{\pi_{\btheta_{1}}}\bphi(x, a), \btheta_{1} - \btheta_{2}\right\rangle\right| + \left|\left\langle\int_{0}^{1}\mE_{\pi_{\btheta_{t}}}[\bphi(a, x)], \btheta_{1} - \btheta_{2}\right\rangle\right|. 
        \end{aligned}
    \end{equation}
\end{lemma}
\begin{proof}
According to the definition of the function $f_{\btheta}(x)$ in \eqref{eq:single sample loss} and the statement below, we have
\begin{equation}\label{eq:stability gap}
    \small
    \begin{aligned}
         \left|f_{\btheta_{1}}(x) - f_{\btheta_{2}}(x)\right| & = \left|\sum_{a\in\cA}r(a, x)\left(\pi_{\btheta_{1}}(a\mid x) - \pi_{\btheta_{2}}(a\mid x)\right)\right. \\
        & \left.- \lambda\sum_{a\in\cA}\left(\pi_{\btheta_{1}}(a\mid x)\log{\frac{\pi_{\btheta_{1}}(a\mid x)}{\pi_{\rm{ref}}(a\mid x)}} - \pi_{\btheta_{2}}(a\mid x)\log{\frac{\pi_{\btheta_{2}}(a\mid x)}{\pi_{\rm{ref}}(a\mid x)}}\right)\right|. 
    \end{aligned}
\end{equation}
where the hyper-parameter $\lambda=1$. Let us tackle the two terms respectively. For the first, by Pinsker's inequality, we have 
\begin{equation}\label{eq:first gap}
    \small
    \begin{aligned}
        \left|\sum_{a\in\cA}r(a, x)\left(\pi_{\btheta_{1}}(a\mid x) - \pi_{\btheta_{2}}(a\mid x)\right)\right| & \leq \max_{a\in\cA}r(a, x)\left|\sum_{a\in\cA}\left(\pi_{\btheta_{1}}(a\mid x) - \pi_{\btheta_{2}}(a\mid x)\right)\right| 
        \\ & =\max_{a \in \cA}\left\langle \bphi(x,a), \btheta^{*} \right\rangle\left|\sum_{a\in\cA}\left(\pi_{\btheta_{1}}(a\mid x) - \pi_{\btheta_{2}}(a\mid x)\right)\right|
        \\ & \le \|\btheta^{*}\|\max_{a \in \cA}\|\bphi(x,a)\|\left|\sum_{a\in\cA}\left(\pi_{\btheta_{1}}(a\mid x) - \pi_{\btheta_{2}}(a\mid x)\right)\right|
        \\ & \leq CR\cdot\TV\left(\pi_{\btheta_{1}}(a\mid x), \pi_{\btheta_{2}}(a\mid x)\right)
        \\ & \leq CR\sqrt{\frac{1}{2}D_{KL}(\pi_{\btheta_{1}}(a\mid x)\parallel \pi_{\btheta_{2}}(a\mid x))}.
    \end{aligned}
\end{equation}
For the KL divergence term, we have 
\begin{equation}\label{eq:kl gap}
    \small
    \begin{aligned}
        D_{KL}\left(\pi_{\btheta_{1}}(a\mid x)\parallel \pi_{\btheta_{2}}(a\mid x)\right) & = \sum_{a\in\cA}\left(\left\langle\btheta_{1} - \btheta_{2}, \bphi(x, a)\right\rangle\right)\pi_{\btheta_{1}}(a\mid x) \\
        & + \log\mE_{\pi_{\rm{ref}}(a\mid x)}[\exp\left(\langle\btheta_{1}, \bphi(a, x)\rangle\right)] - \log\mE_{\pi_{\rm{ref}}(a\mid x)}[\exp\left(\langle\btheta_{2}, \bphi(a, x)\rangle\right)] \\
        & = \left\langle \mE_{\pi_{\btheta_{1}}(a\mid x)}[\bphi(x, a)], \btheta_{1} - \btheta_{2}\right\rangle + Z_{\btheta_{1}}(x) - Z_{\btheta_{2}}(x) \\
        & = \left\langle \mE_{\pi_{\btheta_{1}}(a\mid x)}[\bphi(x, a)], \btheta_{1} - \btheta_{2}\right\rangle + \left\langle\int_{0}^{1}\mE_{\pi_{\btheta_{t}}(a\mid x)}[\bphi(a, x)], \btheta_{1} - \btheta_{2}\right\rangle.
    \end{aligned}
\end{equation} 
where the last equality is from Lemma \ref{lem:energy gap}, and $\btheta_{t}$ is an interpolation between $\btheta_{1}$ and $\btheta_{2}$ as in Lemma \ref{lem:energy gap}. On the other hand, let us upper bound the second term in \eqref{eq:stability gap}. By basic algebra, we have: 
\begin{equation}\label{eq:second gap}
    \small
    \begin{aligned}
        & \left|\sum_{a\in\cA}\left(\pi_{\btheta_{1}}(a\mid x)\log{\frac{\pi_{\btheta_{1}}(a\mid x)}{\pi_{\rm{ref}}(a\mid x)}} - \pi_{\btheta_{2}}(a\mid x)\log{\frac{\pi_{\btheta_{2}}(a\mid x)}{\pi_{\rm{ref}}(a\mid x)}}\right)\right| \\
        & \leq \left|\mE_{\pi_{\btheta_{1}}(a\mid x)}[\langle\btheta_{1}, \bphi(x, a)\rangle] - \mE_{\pi_{\btheta_{2}}(a\mid x)}[\langle\btheta_{2}, \bphi(x, a)\rangle]\right| + \left|Z_{\btheta_{2}}(x) - Z_{\btheta_{1}}(x)\right| \\
        & \leq \left|\langle\btheta_{1} - \btheta_{2}, \mE_{\pi_{\btheta_{1}}}[\phi(x, a)]\rangle\right| + \left|\left\langle\btheta_{2}, \mE_{\pi_{\btheta_{1}}}[\bphi(x, a)] - \mE_{\pi_{\btheta_{2}}}[\bphi(x, a)]\right\rangle\right| + \left|\left\langle\int_{0}^{1}\mE_{\pi_{\btheta_{t}}(a\mid x)}[\bphi(a, x)], \btheta_{1} - \btheta_{2}\right\rangle\right| \\
        & \leq \left|\left\langle\mE_{\pi_{\btheta_{1}}(a\mid x)}\bphi(x, a), \btheta_{1}-\btheta_{2}\right\rangle\right| + \left|\left\langle\int_{0}^{1}\mE_{\pi_{\btheta_{t}}(a\mid x)}[\bphi(a, x)], \btheta_{1}-\btheta_{2}\right\rangle\right| + \left|\left\langle\btheta_{2}, \mE_{\pi_{\btheta_{1}}}[\bphi(x, a)] - \mE_{\pi_{\btheta_{2}}}[\bphi(x, a)]\right\rangle\right|.
    \end{aligned}
\end{equation}
For the last term above, by definition, we have 
\begin{equation}\label{eq:third gap}
    \small
    \begin{aligned}
        \left|\left\langle\btheta_{2}, \mE_{\pi_{\btheta_{1}}}[\bphi(x, a)] - \mE_{\pi_{\btheta_{2}}}[\bphi(x, a)]\right\rangle\right| & \leq \|\btheta_{2}\|\max_{a\in\cA}\|\bphi(x,a)\|\TV\left(\pi_{\btheta_{1}}(a\mid x), \pi_{\btheta_{2}}(a\mid x)\right) \\
        & \leq CR\sqrt{\frac{1}{2}D_{KL}(\pi_{\btheta_{1}}(a\mid x)\parallel \pi_{\btheta_{2}}(a\mid x))} \\
        & \leq CR\sqrt{\left\langle \mE_{\pi_{\btheta_{1}}(a\mid x)}[\bphi(x, a)], \btheta_{1}-\btheta_{2}\right\rangle + \left\langle\int_{0}^{1}\mE_{\pi_{\btheta_{t}}(a\mid x)}[\bphi(a, x)], \btheta_{1}-\btheta_{2}\right\rangle}.
    \end{aligned}
\end{equation}
Then by combining \eqref{eq:stability gap}-\eqref{eq:third gap}, we obtain our conclusion. 
\end{proof}

\begin{lemma} \label{lem:projection bound}
    For given $\btheta$ and positive constant $\delta$, if 
    \begin{equation*}
        \small
        \begin{aligned}
             \|\nabla_{\btheta}J_{\bS}(\pi_{\btheta})\| \le \delta,
        \end{aligned}
    \end{equation*}
    then we have 
    \begin{equation}
        \small
        \begin{aligned}
            \|\bV_{\bS}(\btheta_{\bS}^{*})(\btheta_{\bS}^{*}-\btheta)\|=\|\bV_{\bS}(\btheta_{\bS}^{*})(\btheta^{*}-\btheta)\| \le \frac{\delta\cdot \sigma_{\max}(\bV_{\bS}(\btheta_{\bS}^{*}))}{\sigma_{\min}^{+}(\bV_{\bS}(\btheta))}.
        \end{aligned}
    \end{equation}
\end{lemma}

\begin{proof}
    Due to the definition of the parameter $\btheta_{\bS}^{*}$ in \eqref{eq:empirical optimal}, we have
    \[\small 0=\nabla_{\btheta}J_{\bS}(\pi_{\btheta_{\bS}^{*}})=\bV_{\bS}(\btheta_{\bS}^{*})(\btheta^{*}-\btheta_{\bS}^{*}),\] it is straightforward to verify the first equality holds. 
    According to the fundamental theorem of linear algebra, it holds that
    \[\small \bbR^{d} = \mathrm{C}(\bV_{\bS}(\btheta)) + \mathrm{N}(\bV_{\bS}^{\top}(\btheta)),\]
    then we decompose the vector $\btheta^{*}-\btheta$ into a sum of orthogonal components, \[\small \btheta^{*}-\btheta=\btheta^{C}+\btheta^{N},\] where $\btheta^{C} \in \mathrm{C}(\bV_{\bS}(\btheta_{\bS}^{*}))$ and $\btheta^{N} \in \mathrm{N}(\bV_{\bS}^{\top}(\btheta_{\bS}^{*}))$. According to Lemma \ref{lem:general same span}, we have
    \begin{equation*}
        \small
        \begin{aligned}
            \nabla_{\btheta}J_{\bS}(\pi_{\btheta})=\bV_{\bS}(\btheta)(\btheta^{*}-\btheta)=\bV_{\bS}(\btheta)\cdot \btheta^{C},
        \end{aligned}
    \end{equation*}
    Firstly, We decompose the vector $\btheta^{C}$ as a linear combination $\btheta^{C}=\sum_{i=1}^{m}c_{i}\bxi_{i}$, where $\{\bxi_{i}\}_{i=1}^{m}$ is the set of eigenvectors of the matrix $\bV_{\bS}(\btheta)$ corresponding to its positive eigenvalues $\{\sigma_{i}\}_{i=1}^{m}$ ($m\leq d$). 
    Then, it holds that:
    \begin{equation*}
        \small
        \begin{aligned}
            \left\|\bV_{\bS}(\btheta)\cdot \btheta^{C}\right\|=\left\|\sum_{i=1}^{m}c_{i}\sigma_{i}\bxi_{i}\right\|=\sqrt{\sum_{i=1}^{m}c_{i}^{2}\sigma_{i}^{2}} \ge\sqrt{\min_{\sigma \in \{\sigma_{j}\}_{j=1}^{m}}\sum_{i=1}^{m}c_{i}^{2}\sigma^{2}}=\min_{1\le i \le m}\sigma_{i}(\bV_{\bS}(\btheta))\|\btheta^{C}\|.
        \end{aligned}
    \end{equation*}
     
     Let $\sigma_{\min}^{+}(\bV_{\bS}(\btheta))$ be the smallest positive eigenvalue of the matrix $\bV_{\bS}(\btheta)$, then we can bound the $\|\btheta^{C}\|$ as follows:
    \begin{equation*}
        \small
        \begin{aligned}
            \|\btheta^{C}\|\le \frac{\|\nabla_{\btheta}J_{\bS}(\pi_{\btheta})\|}{\sigma_{\min}^{+}(\bV_{\bS}(\btheta))}\le\frac{\delta}{\sigma_{\min}^{+}(\bV_{\bS}(\btheta))},
        \end{aligned}
    \end{equation*}
     according to the lemma \ref{lem:general same span}, it holds that $\mathrm{C}(\bV_{\bS}(\btheta_{\bS}^{*}))=\mathrm{C}(\bV_{\bS}(\btheta))$, therefore we can calculate that
     \begin{equation}
         \small
         \begin{aligned}
            \|\bV_{\bS}(\btheta_{\bS}^{*})(\btheta^{*}-\btheta)\|=\|\bV_{\bS}(\btheta_{\bS}^{*})\btheta^{C}\| \le \sigma_{\max}(\bV_{\bS}(\btheta_{\bS}^{*}))\cdot \|\btheta^{C}\| \le \frac{\delta\cdot \sigma_{\max}(\bV_{\bS}(\btheta_{\bS}^{*}))}{\sigma_{\min}^{+}(\bV_{\bS}(\btheta))}.
         \end{aligned}
     \end{equation}
\end{proof}

\begin{lemma}\label{lem:upper bound on b}
    Under Assumption \ref{ass:bounded} and \ref{ass:span space}, for any feature vector $\bphi(x,a)$, there exists a proposed $\bb_{\bS}(x,a)$ in Assumption \ref{ass:span space} satisfies $\|\bb_{\bS}(x, a)\| \leq \max_{\sigma_{i}(\bV_{\bS}(\btheta_{\bS}^{*}))>0}\frac{C}{\sigma_{i}(\bV_{\bS}(\btheta_{\bS}^{*}))}$, where $\{\sigma_{i}(\bV_{\bS}(\btheta_{\bS}^{*}))\}$ are eigenvalues of $\bV_{\bS}(\btheta_{\bS}^{*})$.
\end{lemma}
\begin{proof}
     The Assumption \ref{ass:span space} supposes that $\bb_{\bS}(x,a) \in \mathrm{C}(\bV_{\bS}(\btheta_{\bS}^{*}))$. And $\bV_{\bS}$ is a real symmetric matrix, which indicates that $\bV_{\bS}(\btheta_{\bS}^{*})\bP = \bP\Lambda$, where the column vectors of $\bP$ are eigenvectors and $\Lambda = \diag\left\{\sigma_{1},\cdots, \sigma_{d}\right\}$ are eigenvalues. Since 
    \[ \small
    \bV_{\bS}(\btheta_{\bS}^{*})\bx = \bV_{\bS}(\btheta_{\bS}^{*})\bP\bP^{-1}\bx = \bP(\Lambda\bP^{-1}\bx),\]
    we know that $\bb_{\bS}(x, a)$ can be expressed as $\bb_{\bS}(x,a) = \sum_{i=1}^{m}c_{i}(x,a)\bxi_{i}$, with $\{\bxi\}_{i=1}^{m}$ are eigenvectors of $\bV_{\bS}(\btheta_{\bS}^{*})$ corresponded to non-negative eigenvalues $\sigma_{i}$. Thus, from one-side, we have 
    \begin{equation}
        \small
        \|\bV_{\bS}(\btheta_{\bS}^{*})\bb_{\bS}(x,a)\| = \|\bphi(x,a)-\br_{\bS}(x,a)\| \leq C,
    \end{equation}
    since the residual term $\br_{\bS}(x, a)$ is orthogonal to $\mathrm{C}(\bV_{\bS}(\btheta_{\bS}^{*}))$. On the other hand, 
    \begin{equation}
        \small
        \begin{aligned}
        \|\bV_{\bS}(\btheta_{\bS}^{*})\bb_{\bS}(x,a)\| &= \left\|\sum_{i=1}^{m}c_{i}(x,a)\sigma_{i}\bxi_{i}\right\| \\
        & = \sqrt{\sum_{i=1}^{m}c_{i}^{2}(x,a)\sigma_{i}^{2}(\bV_{\bS}(\btheta_{\bS}^{*}))} \\
        & \geq \min_{1\leq i \leq m}\sigma_{i}(\bV_{\bS}(\btheta_{\bS}^{*}))\|\bc(x,a)\| \\
        & = \min_{1\leq i \leq m}\sigma_{i}(\bV_{\bS}(\btheta_{\bS}^{*}))\|\bb_{\bS}\|.
        \end{aligned}
    \end{equation}
    Combining it with the above inequality proves our conclusion. 
\end{proof}

Now we are ready to prove Lemma \ref{lem:key lemma}.

\begin{proof}
    According to Lemma \ref{lem:general stability gap step I}, we have
    \begin{equation}\label{eq:key lemma part 1}
        \small
        \begin{aligned}
            |f_{\btheta_{1}}(x) - f_{\btheta_{2}}(x)| \le & 2C\sqrt{\left\langle \mE_{\pi_{\btheta_{1}}}\bphi(x, a), \btheta_{1} - \btheta_{2}\right\rangle + \left\langle\int_{0}^{1}\mE_{\pi_{\btheta_{t}}}[\bphi(a, x)], \btheta_{1} - \btheta_{2}\right\rangle} \\
            & + \left|\left\langle\mE_{\pi_{\btheta_{1}}}\bphi(x, a), \btheta_{1} - \btheta_{2}\right\rangle\right| + \left|\left\langle\int_{0}^{1}\mE_{\pi_{\btheta_{t}}}[\bphi(a, x)], \btheta_{1} - \btheta_{2}\right\rangle\right|.
        \end{aligned}
    \end{equation}

    Under the assumption \ref{ass:span space}, by using Lemma \ref{lem:projection bound}, \ref{lem:upper bound on b}, it holds that
    \begin{equation}\label{eq:key lemma part 2}
        \small
        \begin{aligned}
            \left\|\langle\bphi(x,a),\btheta_{1}-\btheta_{2}\rangle\right\|=&\left\|\left\langle\bb_{\bS}(x,a),\bV_{\bS}(\btheta_{\bS}^{*})(\btheta_{1}-\btheta_{2})\right\rangle+\left\langle \br_{\bS}(x,a),\btheta_{1}-\btheta_{2}\right\rangle\right\|
            \\ =& \left\|\left\langle\bb_{\bS}(x,a),\bV_{\bS}(\btheta_{\bS}^{*})(\btheta_{1}-\btheta^{*})\right\rangle+\left\langle\bb_{\bS}(x,a),\bV_{\bS}(\btheta_{\bS}^{*})(\btheta^{*}-\btheta_{2})\right\rangle+\left\langle \br_{\bS}(x,a),\btheta_{1}-\btheta_{2}\right\rangle\right\|
            \\ \le&2R\epsilon_{n}+\left\|\bb_{\bS}(x,a) \right\| \cdot \left(\left\|\bV_{\bS}(\btheta_{\bS}^{*})(\btheta_{1}-\btheta^{*})\right\|+\left\|\bV_{\bS}(\btheta_{\bS}^{*})(\btheta_{2}-\btheta^{*})\right\|\right)
            \\ \le &2R\epsilon_{n}+\frac{C\delta_{1}\cdot \sigma_{\max}(\bV_{\bS}(\btheta_{\bS}^{*}))}{\sigma_{\min}^{+}(\bV_{\bS}(\btheta_{\bS}^{*}))\sigma_{\min}^{+}(\bV_{\bS}(\btheta_{1}))}+\frac{C\delta_{2}\cdot \sigma_{\max}(\bV_{\bS}(\btheta_{\bS}^{*}))}{\sigma_{\min}^{+}(\bV_{\bS}(\btheta_{\bS}^{*}))\sigma_{\min}^{+}(\bV_{\bS}(\btheta_{2}))}.
        \end{aligned}
    \end{equation}
    Combining \eqref{eq:key lemma part 2} with \eqref{eq:key lemma part 1}, we complete the proof.
\end{proof}

\section{Some Regularity Properties on $J_{\bS}(\pi_{\btheta})$, $\bV_{\bS}(\btheta)$}

\begin{proposition}\label{pro:expected optimum is empirical optimum}
    When $\pi_{\btheta}$ is from Definition \ref{def:policy function}, we have 
    \begin{equation}\label{eq:ground truth is expected optimal}
        \small
        \btheta^{*}\in \arg\max_{\btheta}J(\pi_{\btheta}),
    \end{equation}
    and 
    \begin{equation}\label{eq:ground truth is empirical optimal}
        \small
        \btheta^{*}\in \arg\max_{\btheta}J_{\bS}(\pi_{\btheta}).
    \end{equation}
    
\end{proposition}
\begin{proof}
    \begin{equation}
        \small
        \begin{aligned}
            J(\pi_{\btheta})= & \mathbb{E}_{x}\left[\mathbb{E}_{a \sim \pi_{\btheta}(\cdot|x)}[r(x,a)]-D_{KL}\left(\pi_{\btheta}(\cdot|x)\|\pi_{\rm{ref}}(\cdot|x)\right)\right]
            \\ = & \mathbb{E}_{x}\left[\mathbb{E}_{a \sim \pi_{\btheta}(\cdot|x)}\left[r(x,a)-\log\frac{\pi_{\btheta}(a|x)}{\pi_{\rm{ref}}(a|x)}\right]\right]
            \\ = & -\mathbb{E}_{x}\left[\mathbb{E}_{a \sim \pi_{\btheta}(\cdot|x)}\left[\log\frac{\pi_{\btheta}(a|x)}{\pi_{\rm{ref}}(a|x)\exp(r(x,a))}\right]\right]
            \\ = & -\mathbb{E}_{x}\left[\mathbb{E}_{a \sim \pi_{\btheta}(\cdot|x)}\left[\log\frac{\pi_{\btheta}(a|x) \cdot Z(x)}{\pi_{\rm{ref}}(a|x)\exp(r(x,a))}\right]-\log Z(x)\right]
            \\ = & \mathbb{E}_{x}[\log Z(x)]-\mathbb{E}_{x}[D_{KL}(\pi_{\btheta}(\cdot|x)\|\pi^{*}(\cdot|x))]
            \\ \le & \mathbb{E}_{x}[\log Z(x)],
        \end{aligned}
    \end{equation}
    where $Z(x)=\mathbb{E}_{a \sim \pi_{\rm{ref}}(\cdot|x)}[\exp(r(x,a))]$ and $\pi^{*}(a|x)=\frac{\pi_{\rm{ref}}(a|x)\cdot\exp(r(x,a))}{Z(x)}$, for any $(x,a) \in \mathcal{X} \times \mathcal{A}$. The last inequality holds for $D_{KL}(\pi_{\btheta}(\cdot|x)\|\pi^{*}(\cdot|x)) \ge 0$, and the equality holds if and only if $\pi_{\btheta}(a|x)\equiv\pi^{*}(a|x)$. The definition of $Z(x)$ indicates that $\mathbb{E}_{x}[\log Z(x)]$ is a constant independent of $\btheta$.

    According to the definition of $\btheta^{*}$, it holds that $\pi_{\btheta^{*}}(a|x)\equiv\pi^{*}(a|x)$, thus we prove the \eqref{eq:ground truth is expected optimal}. Similarly, we can prove the conclusion in \eqref{eq:ground truth is empirical optimal}.
\end{proof}
This proposition shows that the ground truth parameter $\btheta^{*}$ locates in the optimal solutions of $J(\pi_{\btheta})$ and $J_{\bS}(\pi_{\btheta})$.

\begin{lemma} \label{lem:gradient}
    Under Assumption \ref{ass:ground truth r}, it holds that 
    \begin{equation}
    \small
        \nabla_{\btheta} J_{\bS}(\pi_{\btheta}) = \bV_{\bS}(\btheta)(\btheta^{*} - \btheta),
    \end{equation}
    where $\bV_{\bS}(\pi_{\btheta}) = \frac{1}{n}\sum_{i=1}^{n}\Var_{\pi_{\btheta}(a\mid x_{i})}[\bphi(x_{i}, a)]$.
\end{lemma}

\begin{proof}
    Firstly, let us check the gradient of loss function $J_{\bS}(\pi_{\btheta})$, which has the following formulation 
    \begin{equation}\label{eq:gradient}
        \small
        \begin{aligned}
            \nabla_{\btheta} J_{\bS}(\pi_{\btheta}) = \frac{1}{n}\sum_{i=1}^{n}\mE_{\pi_{\btheta}(a\mid x_{i})}&\left[\nabla_{\btheta}\log{\pi_{\btheta}(a\mid x_{i})} \left(r(x_{i}, a) - \left\langle\btheta, \bphi(x_{i}, a)\right\rangle\right)\right].
        \end{aligned}
    \end{equation}
    Due to the formulation of $\pi_{\btheta}$, we know that
    \begin{equation}\label{eq:log-trick}
        \small
        \nabla_{\btheta}\log{\pi_{\btheta}(a\mid x)} = \bphi(x, a) - \mE_{\pi_{\btheta}(a\mid x)}[\bphi(x, a)].
    \end{equation}
    Then, due to Assumption \ref{ass:ground truth r} and combining the above equations, we obtain our conclusion.  
\end{proof}


\begin{proposition}\label{pro:first order optimum is global optimum}
    Under the defined linear reward parameterization and Posterior Boltzmann policy class $\Pi$, any parameter $\btheta$ satisfying the first-order stationarity condition $\nabla_{\btheta} J_{\bS}(\pi_{\btheta}) = 0$ is a global maximizer of the empirical objective $J_{\bS}(\pi_{\btheta})$. In other words, for any $\btheta'\in\mathbb{R}^{d}$, we have 
    \begin{equation*}
    \small
        \begin{aligned}
            \btheta' \in \arg \max_{\btheta}J_{\bS}(\pi_{\btheta}) \quad \Longleftrightarrow \quad \nabla_{\btheta}J_{\bS}(\pi_{\btheta'})=0.
        \end{aligned}
    \end{equation*}
\end{proposition}
\begin{proof}
    The necessity is straightforward. We now focus only on the proof of sufficiency. Observe that $\btheta^{*} \in \arg \max_{\btheta}J_{\bS}(\pi_{\btheta})$. According to Proposition \ref{lem:gradient}, we know
    \begin{equation*}
        \small
        \begin{aligned}
            \nabla_{\btheta}J_{\bS}(\pi_{\btheta'})=\bV_{\bS}(\btheta')(\btheta^{*}-\btheta')=0.
        \end{aligned}
    \end{equation*}

    Then we have
    \begin{equation}
        \small
        \begin{aligned}
            0 &=(\btheta^{*}-\btheta')^{\top}\bV_{\bS}(\btheta')(\btheta^{*}-\btheta') \notag
            \\ &=\frac{1}{n}\sum_{i=1}^{n}(\btheta^{*}-\btheta')^{\top}\mathbb{E}_{\pi_{\btheta'}}\left[\left(\bphi(x_{i},a)-\mathbb{E}_{ \pi_{\btheta'}}[\bphi(x_{i},a')]\right)\left(\bphi(x_{i},a)-\mathbb{E}_{\pi_{\btheta'}}[\bphi(x_{i},a')]\right)^{\top}\right](\btheta^{*}-\btheta')  \notag
            \\ &=\frac{1}{n}\sum_{i=1}^{n}\mathbb{E}_{\pi_{\btheta'}}\left[\left((\btheta^{*}-\btheta')^{\top}\left(\bphi(x_{i},a)-\mathbb{E}_{\pi_{\btheta'}}[\bphi(x_{i},a')]\right)\right)^{2}\right].
        \end{aligned}
    \end{equation}
    Therefore, for any $x_{i}\in \bS$ and $a$, we have
    \begin{equation*}
        \small
        \begin{aligned}
            (\btheta^{*}-\btheta')^{\top}\bphi(x_{i},a)=(\btheta^{*}-\btheta')^{\top}\mathbb{E}_{ \pi_{\btheta'}}[\bphi(x_{i},a')],
        \end{aligned}
    \end{equation*}
    which means that 
    \begin{equation}
        \small
        \begin{aligned}
            \frac{\pi_{\btheta'}(a|x_{i})}{\pi_{\rm{ref}}(a|x_{i})} &=\frac{\text{exp}(\lambda^{-1}\langle \btheta',\bphi(x_{i},a)\rangle)}{\mathbb{E}_{\pi_{\rm{ref}}}\left[\text{exp}(\lambda^{-1}\langle \btheta^{'},\bphi(x_{i},a)\rangle)\right]} \notag
            \\ &=\frac{\text{exp}(\lambda^{-1}\langle \btheta^{'}-\btheta^{*},\bphi(x_{i},a)\rangle)\cdot\text{exp}(\lambda^{-1}\langle\btheta^{*},\bphi(x_{i},a)\rangle)}{\text{exp}\left(\lambda^{-1}\langle \btheta'-\btheta^{*},\bphi(x_{i},a)\rangle\right)\cdot \mathbb{E}_{\pi_{\rm{ref}}}\left[\text{exp}(\lambda^{-1}\langle \btheta^{*},\bphi(x_{i},a)\rangle)\right]}  \notag
            \\ &=\frac{\pi_{\btheta^{*}}(a|x_{i})}{\pi_{\rm{ref}}(a|x_{i})}.
        \end{aligned}
    \end{equation}
    Therefore, it holds that 
    \begin{equation}\label{eq:stationary policy equiv ground truth}
        \small
        \pi_{\btheta^{\prime}}(a|x_{i})=\pi_{\btheta^{*}}(a|x_{i}),\forall (x_{i},a) \in \bS \times \mathcal{A}.
    \end{equation}
    Combining \eqref{eq:stationary policy equiv ground truth} with \eqref{eq:empirical loss}, we have
    \begin{equation*}
        \small
        \begin{aligned}
            J_{\bS}(\pi_{\btheta'})&=\frac{1}{n}\sum_{i=1}^{n}\mathbb{E}_{a \sim \pi_{\btheta^{\prime}}(\cdot|x_{i})}\left[r(x_{i},a)-\log \frac{\pi_{\btheta^{\prime}}(a|x_{i})}{\pi_{\rm{ref}}(a|x_{i})}\right]=\frac{1}{n}\sum_{i=1}^{n}\sum_{a} \pi_{\btheta^{\prime}}(a|x_{i})\left[r(x_{i},a)-\log \frac{\pi_{\btheta^{\prime}}(a|x_{i})}{\pi_{\rm{ref}}(a|x_{i})}\right]
            \\ & =\frac{1}{n}\sum_{i=1}^{n}\sum_{a} \pi_{\btheta^{*}}(a|x_{i})\left[r(x_{i},a)-\log \frac{\pi_{\btheta^{*}}(a|x_{i})}{\pi_{\rm{ref}}(a|x_{i})}\right] =\frac{1}{n}\sum_{i=1}^{n}\mathbb{E}_{a \sim \pi_{\btheta^{*}}(\cdot|x_{i})}\left[r(x_{i},a)-\log \frac{\pi_{\btheta^{*}}(a|x_{i})}{\pi_{\rm{ref}}(a|x_{i})}\right]=J_{\bS}(\pi_{\btheta^{*}}).
        \end{aligned}
    \end{equation*}
    By combining with Proposition \ref{pro:expected optimum is empirical optimum}, we prove the conclusion.
\end{proof}
\begin{remark}
    This proposition connects our definition of $\btheta_{\bS}^{*}$ in \eqref{eq:empirical optimal} to global optima. It implies that finding a stationary point is equivalent to finding the empirical global optimum, i.e., $J_{\bS}(\pi_{\btheta_{\bS}^{*}}) \ge J_{\bS}(\pi_{\btheta})$ for any $\btheta \in \mathbb{R}^{d}$.
\end{remark}

    

Next, we introduce two lemmas, which states the smoothness of $J_{\bS}(\pi_{\btheta})$ w.r.t. $\btheta$.

\begin{lemma}\label{lem:smoothness coefficient}
    Under Assumptions \ref{ass:bounded} and \ref{ass:ground truth r}, the empirical risk $J_{\bS}(\pi_{\btheta})$ is $L_{f}-$smoothness with coefficient $8RC^{3} + C^{2}$ within the region $\Theta_{R}$. That says, 
    \begin{equation}
        \small
        \|\nabla^{2}_{\btheta} J_{\bS}(\pi_{\btheta})\| \leq 8RC^{3} + C^{2}.
    \end{equation}
\end{lemma}  
\begin{proof}
    To simplify the notations, we define $\Sigma_{x}(\btheta) = \mE_{\pi_{\btheta}(a\mid x)}[\bphi(x,a)\bphi(x,a)^{\top}]$, and $\bmu_{x}(\btheta) = \mE_{\pi_{\btheta}(a\mid x)}[\bphi(x,a)]$. Then we know 
    \begin{equation}\label{eq:empirical hessian}
        \small
        \begin{aligned}
            \left\|\nabla^{2} J_{\bS}(\pi_{\btheta})\right\| & = \frac{1}{n}\left\|\sum_{i=1}^{n}\nabla_{\btheta}(\Sigma_{x_{i}}(\btheta) - \bmu_{x_{i}}(\btheta)\bmu_{x_{i}}(\btheta)^{\top})(\btheta^{*} - \btheta) - \bV_{x_{i}}(\btheta)\right\| \\
            & \le \frac{1}{n}\sum\limits_{i=1}^{n}\left(\left\|\nabla_{\btheta}(\Sigma_{x_{i}}(\btheta) - \bmu_{x_{i}}(\btheta)\bmu_{x_{i}}(\btheta)^{\top})(\btheta^{*} - \btheta)\right\| + \left\|V_{x_{i}}(\btheta)\right\|\right).	
        \end{aligned}
    \end{equation}
    For the first term, we denote $g_{i}(\bu,\btheta)=\bu^{\top}(\Sigma_{x_{i}}(\btheta) - \bmu_{x_{i}}(\btheta)\bmu_{x_{i}}(\btheta)^{\top})\bu$, then for any $x_{i}$ we have
    \begin{equation}\label{eq:g_i}
        \small
        \begin{aligned}
            &\left\|\nabla_{\btheta}(\Sigma_{x_{i}}(\btheta) - \bmu_{x_{i}}(\btheta)\bmu_{x_{i}}(\btheta)^{\top})(\btheta^{*} - \btheta)\right\|
            \\ =&\max_{\bu\in \mathbb{R}^{d},\|\bu\|\le1}\bu^{\top}\nabla_{\btheta}(\Sigma_{x_{i}}(\btheta) - \bmu_{x_{i}}(\btheta)\bmu_{x_{i}}(\btheta)^{\top})(\btheta^{*} - \btheta)\bu
            \\ =&\max_{\bu\in \mathbb{R}^{d},\|\bu\|\le1}\left\langle\nabla_{\btheta }g_{i}(\bu,\btheta),\btheta^{*} - \btheta\right\rangle.
        \end{aligned}
    \end{equation}
    By applying lemma\eqref{eq:log-trick}, we have
    \begin{equation}\label{eq:g_i gradient}
        \small 
        \begin{aligned}
            &\nabla_{\btheta }g_{i}(\bu,\btheta) \\
            =&\nabla_{\btheta}\left(\mathbb{E}_{\pi_{\btheta}}[\bu^{\top}\bphi(x_{i},a)]^{2}-\mathbb{E}_{\pi_{\btheta}}^{2}[\bu^{\top}\bphi(x_{i},a)]\right)
        \\ = &\mathbb{E}_{\pi_{\btheta}}\left[\nabla_{\btheta}\log \pi_{\btheta}(a|x_{i})\cdot\left(\bu^{\top}\bphi(x_{i},a)\right)^{2}\right]-2\mathbb{E}_{\pi_{\btheta}}\left[\bu^{\top}\bphi(x_{i},a)\right]\mathbb{E}_{\pi_{\btheta}}\left[\nabla_{\btheta}\log \pi_{\btheta}(a|x_{i})\cdot\bu^{\top}\bphi(x_{i},a)\right]
        \\ =& \mathbb{E}_{\pi_{\btheta}}\left[\left(\bphi(x_{i},a)-\bmu_{x_{i}}(\btheta)\right)\cdot\left(\bu^{\top}\bphi(x_{i},a)\right)^{2}\right]-2\mathbb{E}_{\pi_{\btheta}}\left[\bu^{\top}\bphi(x_{i},a)\right]\mathbb{E}_{\pi_{\btheta}}\left[\left(\bphi(x_{i},a)-\bmu_{x_{i}}(\btheta)\right)\cdot\bu^{\top}\bphi(x_{i},a)\right]
        \\ =&\mathbb{E}_{\pi_{\btheta}}\left[\bphi(x_{i},a)\left(\bu^{\top}\bphi(x_{i},a)\right)^{2}\right]-2\mathbb{E}_{\pi_{\btheta}}\left[\bu^{\top}\bphi(x_{i},a)\right]\mathbb{E}_{\pi_{\btheta}}\left[\bphi(x_{i},a)\bu^{\top}\bphi(x_{i},a)\right]
        \\
        & -\bmu_{x_{i}}(\btheta)\left(\Var_{\pi_{\btheta}}\left(\bu^{\top}\bphi(x_{i},a)\right)-\mathbb{E}_{\pi_{\btheta}}^{2}\left[\bu^{\top}\bphi(x_{i},a)\right]\right).
        \end{aligned}
    \end{equation}
    Substituting \eqref{eq:g_i gradient} into \eqref{eq:g_i}, we obtain
    \begin{equation}\label{eq:empirical hessian 1}
        \small
        \begin{aligned}
            &\left\|\nabla_{\btheta}(\Sigma_{x_{i}}(\btheta) - \bmu_{x_{i}}(\btheta)\bmu_{x_{i}}(\btheta)^{\top})(\btheta^{*} - \btheta)\right\| 
            \\ \le &\max_{\bu\in \mathbb{R}^{d},\|\bu\|\le1}\left\|\nabla_{\btheta }g_{i}(\bu,\btheta)\right\|\cdot\left\|\btheta^{*} - \btheta\right\|
            \\ \le & 2R\left\|\mathbb{E}_{\pi_{\btheta}}\left[\bphi(x_{i},a)\left(\bu^{\top}\bphi(x_{i},a)\right)^{2}\right]\right\|+4R\left\|\mathbb{E}_{\pi_{\btheta}}\left[\bu^{\top}\bphi(x_{i},a)\right]\mathbb{E}_{\pi_{\btheta}}\left[\bphi(x_{i},a)\bu^{\top}\bphi(x_{i},a)\right]\right\|
            \\ &+2R\left\|\bmu_{x_{i}}(\btheta)\left(\Var_{\pi_{\btheta}}\left(\bu^{\top}\bphi(x_{i},a)\right)-R\mathbb{E}_{\pi_{\btheta}}^{2}\left[\bu^{\top}\bphi(x_{i},a)\right]\right)\right\|
            \\ \le& 2RC^{3}+4RC^{3}+2RC\max\left\{\Var_{\pi_{\btheta}}\left(\bu^{\top}\bphi(x_{i},a)\right),\mathbb{E}_{\pi_{\btheta}}^{2}\left[\bu^{\top}\bphi(x_{i},a)\right]\right\}
            \\ \le &8RC^{3}.
        \end{aligned}
    \end{equation}
    For the second term in \eqref{eq:empirical hessian}, we have
    \begin{equation}\label{eq:covariance norm}
        \small
        \begin{aligned}
            \left\|\bV_{x_{i}}(\btheta)\right\|=\left\|\Sigma_{x_{i}}(\btheta) - \bmu_{x_{i}}(\btheta)\bmu_{x_{i}}(\btheta)^{\top}\right\| \le \max\left\{\left\|\Sigma_{x_{i}}(\btheta)\right\|,\left\|\bmu_{x_{i}}(\btheta)\bmu_{x_{i}}(\btheta)^{\top}\right\|\right\}\le C^{2}.
        \end{aligned}
    \end{equation}
    Then, by combining \eqref{eq:empirical hessian}, \eqref{eq:empirical hessian 1} and \eqref{eq:covariance norm}, we can prove that the Hessian matrix has all eigenvalues with absolute values smaller than $8RC^{3} + C^{2}$.   
\end{proof}

\begin{lemma}\label{lem:Holder continuous}
     Let $\bV_{x}(\btheta) = \Var_{\pi_{\btheta}(a\mid x)}[\bphi(a, x)]$, it satisfies the Lipschitz continuous condition 
    \begin{equation}
        \small
        \|\bV_{x}(\btheta_{1}) - \bV_{x}(\btheta_{2})\| \leq 3C^{\frac{5}{2}}\|\btheta_{1} - \btheta_{2}\|^{\frac{1}{2}},
    \end{equation}
    for any $\btheta_{1}, \btheta_{2}\in\mathbb{R}^{d}$ and $x$. 
\end{lemma}
\begin{proof}
    Due to the definition, we have 
    \begin{equation}
        \small
        \begin{aligned}
            & \|\bV_{x}(\btheta_{1}) - \bV_{x}(\btheta_{2})\| = \sup_{\bu:\|\bu\| = 1}\bu^{\top}(\bV_{x}(\btheta_{1}) - \bV_{x}(\btheta_{2}))\bu \\
            & \leq \left|\mE_{\pi_{\btheta_{1}}(a\mid x)}\left[(\bu^{\top}\bphi(a, x))^{2}\right] - \mE_{\pi_{\btheta_{2}}(a\mid x)}\left[(\bu^{\top}\bphi(a, x))^{2}\right]\right| + \left|\left(\bu^{\top}\mE_{\pi_{\btheta_{1}}(a\mid x)}\left[\bphi(a, x)\right]\right)^{2} - \left(\bu^{\top}\mE_{\pi_{\btheta_{2}}(a\mid x)}\left[\bphi(a, x)\right]\right)^{2}\right|\\
            & \leq C^{2}\TV(\pi_{\btheta_{1}}(a\mid x), \pi_{\btheta_{2}}(a\mid x)) + 2C^{2}\TV(\pi_{\btheta_{1}}(a\mid x), \pi_{\btheta_{2}}(a\mid x))\\
            & \leq  3C^{\frac{5}{2}}\|\btheta_{1} - \btheta_{2}\|^{\frac{1}{2}},
        \end{aligned}
    \end{equation}
    which proves our conclusion. Here the last inequality is based on the proof of Lemma \ref{lem:general stability gap step I}.
\end{proof}

The following lemma shows that the loss function $f_{\btheta}(x)$ is bounded.
\begin{lemma}\label{lem:bounded loss}

Under Assumption \ref{ass:bounded}, for any $\btheta \in \Theta_{R}$ and $x$, we have $|f_{\btheta}(x)| \leq 3RC$, which implies $|J_{\bS}(\pi_{\btheta})| \leq 3RC$ as well.

\end{lemma}
\begin{proof}
    According to the definition of $f_{\btheta}(x)$ in \eqref{eq:single sample loss}, we have:
    \begin{equation}
        \small
        \begin{aligned}
            \left|f_{\btheta}(x)\right|
             &\le \max_{x}\left|\mathbb{E}_{\pi_{\btheta}}\left[r(x,a)-\log\frac{\pi_{\btheta}(a|x)}{\pi_{\rm{ref}}(a|x)}\right]\right|
            \\ &=\max_{x}\left|\sum_{a\in \mathcal{A}}\pi_{\btheta}(a|x)\left(\left\langle\btheta^{*},\bphi(x,a)\right\rangle-\log\frac{\exp(\langle\btheta,\bphi(x,a)\rangle)}{\sum_{a^{\prime}}\pi_{\rm{ref}}(a^{\prime}|x)\exp(\langle\btheta,\bphi(x,a)\rangle)}\right)\right|
            \\ &\le\max_{x,a}\left|\left\langle\btheta^{*}-\btheta,\bphi(x,a)\right\rangle\right|+\max_{x,a}\left|\log(\exp(\langle\btheta,\bphi(x,a)\rangle))\right|
            \\ &\le 3RC.
        \end{aligned}
    \end{equation}
    Furthermore, according to the definition of $J_{\bS}(\pi_{\btheta})$ in \eqref{eq:empirical loss}, we have:
    \begin{equation}
        \small
        \begin{aligned}
            \left|J_{\bS}(\pi_{\btheta})\right|&=\left|\frac{1}{n}\sum_{i=1}^{n}f_{\btheta}(x_{i})\right|\le\max_{x}\left|f_{\btheta}(x)\right| \le 3RC.
        \end{aligned}
    \end{equation}
\end{proof}

\begin{lemma}\label{lem:gradient norm between two different sets}
For any $\btheta \in \Theta_{R}$, it holds that
    \begin{equation*}
        \small
        \left\|\nabla_{\btheta}J_{\bS}(\pi_{\btheta})-\nabla_{\btheta}J_{\bS^{\prime}}(\pi_{\btheta})\right\| \le \frac{4RC^{2}}{n}.
    \end{equation*}
\end{lemma}
\begin{proof}
Since datasets $\bS$ and $\bS^{\prime}$ differ by one example, W.o.l.g. $x_{i} \in \bS$ and $x_{i}^{\prime} \in \bS^{\prime}$ are different, then
    \begin{equation*}
        \small 
        \begin{aligned}
            \left\|\nabla_{\btheta}J_{\bS}(\pi_{\btheta})-\nabla_{\btheta}J_{\bS^{\prime}}(\pi_{\btheta})\right\|= & \left\|\frac{1}{n}(\nabla_{\btheta}f_{\btheta}(x_{i})-\nabla_{\btheta}f_{\btheta}(x_{i}^{\prime}))\right\| \\
            = & \frac{1}{n}\left\|\left(\bV_{x_{i}}(\btheta)-\bV_{x_{i}^{\prime}}(\btheta)\right)(\btheta^{*}-\btheta)\right\| 
            \\ \le & \frac{1}{n}\left\|\bV_{x_{i}}(\btheta)-\bV_{x_{i}^{\prime}}(\btheta)\right\|\|\btheta^{*}-\btheta\|\\
            \le & \frac{2R}{n}(\|\bV_{x_{i}}(\btheta)\|+\|\bV_{x_{i}^{\prime}}(\btheta)\|) \\
            \le & \frac{4RC^{2}}{n}.
        \end{aligned}
    \end{equation*}
    where the last inequality is from \eqref{eq:covariance norm}.
\end{proof}
\begin{lemma}\label{lem:second-order moment}
    For any $\btheta \in \Theta_{R}$, we have 
    \[ \small \mE_{i_{t}}\left[\left|\nabla_{\btheta}f_{\btheta}(x_{i_{t}}) - \mE_{i_{t}}[\nabla_{\btheta}f_{\btheta}(x_{i_{t}})]\right|^{2}\right] \leq 16R^{2}C^{4}.\]
\end{lemma}
\begin{proof}
Firstly, due to the selection of the index $i_{t}$, $\nabla_{\btheta}f_{\btheta}(x_{i_{t}})$ is an unbiased estimation of $\nabla_{\btheta}J_{\bS}(\pi_{\btheta})$, i.e., $\mathbb{E}[\nabla_{\btheta}f_{\btheta}(x_{i_{t}})]=\nabla_{\btheta}J_{\bS}(\pi_{\btheta})$. Thus we have 
    \begin{equation}
        \small
        \begin{aligned}
    \mathbb{E}\left[\left\|\nabla_{\btheta}f_{\btheta}(x_{i_{t}})-\mathbb{E}[\nabla_{\btheta}f_{\btheta}(x_{i_{t}})]\right\|^{2}\right]  = & \mathbb{E}\left[\left\|\nabla_{\btheta}f_{\btheta}(x_{i_{t}})-\nabla_{\btheta}J_{\bS}(\pi_{\btheta})\right\|^{2}\right]
            \\ = & \mathbb{E}\left[\frac{1}{n^{2}}\left\|\sum_{i\neq i_{t}}\left(\nabla_{\btheta}f_{\btheta}(x_{i_{t}})-\nabla_{\btheta}f_{\btheta}(x_{i})\right)\right\|^{2}\right] 
            \\ \le & \sup_{x_{i},x_{j}\in \bS}\left\|\nabla_{\btheta}f_{\btheta}(x_{i})-\nabla_{\btheta}f_{\btheta}(x_{j})\right\|^{2} 
            \\ = & \sup_{x_{i},x_{i} \in \bS}\left\| (\bV_{x_{i}}(\btheta)-\bV_{x_{j}}(\btheta))(\btheta^{*}-\btheta) \right\|^{2}
            \\ \le & (2\sup_{x}\|\bV_{x}(\btheta)\|\cdot \|\btheta^{*}-\btheta\|)^{2}
            \\ \le &16R^{2}C^{4}.
        \end{aligned}
    \end{equation}
\end{proof}

\begin{remark}
    In fact, according to the proof of Lemma \ref{lem:second-order moment}, we have a stronger results: 
    
     Given $\bS$, for any $\btheta \in \Theta_{R}$, we have 
    \begin{equation}
        \small
        \begin{aligned}
            \sup_{x\in\bS}\|\nabla_{\btheta}f_{\btheta}(x)-\nabla_{\btheta}J_{\bS}(\pi_{\btheta})\|^{2} \le 16R^{2}C^{4} =G^{2}.
        \end{aligned}
    \end{equation}
\end{remark}

We now introduce two lemmas that describe the properties of the covariance matrices $\bV_{x}(\btheta)$ and $\bV_{\bS}(\btheta)$.
\begin{lemma}\label{lem:the same span space}
    Under Assumption \ref{ass:linear indep}, for any $\btheta_{1}, \btheta_{2}$ and $x$, it holds that
    \begin{equation}\label{eq:the same span space}
        \small
        \mathrm{C}(\bV_{x}(\btheta_{1})) = \mathrm{C}(\bV_{x}(\btheta_{2})), 
    \end{equation}
    when $d \ge |\cA|$. Specifically, for any $\btheta$, it holds that 
    \begin{equation*}
        \small
        \mathrm{C}(\bV_{x}(\btheta)) = \mathrm{C}(\bV_{x}(\btheta_{\bS}^{*})).
    \end{equation*}
\end{lemma}
\begin{proof}
    According to \eqref{lem:gradient}, we have 
    \begin{equation*}
        \small
        \begin{aligned}
            \bV_{x}(\btheta) & =\mathbb{E}_{\pi_{\btheta}}\left[\left(\bphi(x,a)-\mathbb{E}[\bphi(x,a^{\prime})]\right)\left(\bphi(x,a)-\mathbb{E}[\bphi(x,a^{\prime})]\right)^{\top}\right]
            \\ &=\mathbb{E}\left[\bphi(x,a)\bphi(x,a)^{\top}\right]-\mathbb{E}[\bphi(x,a)]\mathbb{E}[\bphi(x,a)]^{\top}
            \\ &=\sum_{i=1}^{|\mathcal{A}|}\bphi(x,a_{i})\pi_{\btheta}(a_{i}|x)\bphi(x,a_{i})^{\top}-\left(\sum_{i=1}^{|\mathcal{A}|}\bphi(x,a_{i})\pi_{\btheta}(a_{i}|x)\right)\left(\sum_{i=1}^{|\mathcal{A}|}\bphi(x,a_{i})\pi_{\btheta}(a_{i}|x)\right)^{\top}
            \\ & =\Phi(x)\Lambda_{\btheta}\Phi(x)^{\top}-\left(\Phi(x)\bp_{\btheta}\right)\left(\Phi(x)\bp_{\btheta}\right)^{\top}
            \\ & =\Phi(x)(\Lambda_{\btheta}-\bp_{\btheta}\bp_{\btheta}^{\top})\Phi(x)^{\top},
        \end{aligned}
    \end{equation*}
    where $\Phi(x) \triangleq \left(\bphi(x,a_{1}),\bphi(x,a_{2}),\cdots,\bphi({x,a_{|\cA|}})\right)\in \mathbb{R}^{d\times|\mathcal{A}|}$, $\Lambda_{\btheta}=\diag\{\pi_{\btheta}(a_{1}|x),\cdots,\pi_{\btheta}(a_{|\mathcal{A}|}|x)\}$ and $\bp_{\btheta}=(\pi_{\btheta}(a_{1}|x),\cdots,\pi_{\btheta}(a_{|\mathcal{A}|}|x))^{\top}$. Let $C_{\bp_{\btheta}}=\Lambda_{\btheta}-\bp_{\btheta}\bp_{\btheta}^{\top}$. 
    
    By Lemma \ref{lem:Cp2}, it holds that:
    \begin{equation*}
        \small
        \begin{aligned}
            \mathrm{C}(\bV_{x}(\btheta_{1}))=\mathrm{C}\left(\Phi(x)C_{\bp_{\btheta_{1}}}\Phi(x)^{\top}\right)=\{Aw|w^{\top}v=0,v=(1,\cdots,1)^{\top}\}=\mathrm{C}\left(\Phi(x)C_{\bp_{\btheta_{2}}}\Phi(x)^{\top}\right)=
        \mathrm{C}(\bV_{x}(\btheta_{2})),
        \end{aligned}
    \end{equation*}
    which proves our conclusion.
\end{proof}

\begin{lemma}\label{lem:general same span}
    Under Assumption \ref{ass:linear indep}, for any $\btheta_{1}, \btheta_{2}$ and dataset $\bS$, it holds
    \begin{align}\label{eq:general same span}
        \small
        \mathrm{C}(\bV_{\bS}(\btheta_{1}))=\mathrm{C}(\bV_{\bS}(\btheta_{2})),
    \end{align}
    when $d \ge |\cA|$. Specifically, for any $\btheta$, it holds that 
    \begin{equation*}
        \small
        \mathrm{C}(\bV_{\bS}(\btheta)) = \mathrm{C}(\bV_{\bS}(\btheta_{\bS}^{*})).
    \end{equation*}
\end{lemma}
\begin{proof}
    We prove this lemma by induction on $n$, i.e., the number of elements in the set $\bS$. When $n=1$, \eqref{eq:general same span} becomes Lemma \ref{lem:the same span space}. Thus Lemma \ref{lem:the same span space} implies the conclusion. 
    
    For $n=2$, we prove that 
    \[\small \mathrm{C}\left(\bV_{x_{1}}(\btheta_{1})+\bV_{x_{2}}(\btheta_{1})\right)=\mathrm{C}\left(\bV_{x_{1}}(\btheta_{2})+\bV_{x_{2}}(\btheta_{2})\right),\] given the conditions that\[\small \mathrm{C}\left(\bV_{x_{1}}(\btheta_{1})\right)=\mathrm{C}\left(\bV_{x_{1}}(\btheta_{2})\right) \text{ and }\mathrm{C}\left(\bV_{x_{2}}(\btheta_{1})\right)=\mathrm{C}\left(\bV_{x_{2}}(\btheta_{2})\right).\]
    Then, the conclusion is implied by Lemma \ref{lem:A two matrix column space}.
    \par
    Now, suppose the lemma holds for $n = k \geq 2$, since $\small\mathrm{C}(\sum_{i=1}^{k}\bV_{x_{i}}(\btheta_{1}))=\mathrm{C}(\sum_{i=1}^{k}\bV_{x_{i}}(\btheta_{2}))$ and $\small \mathrm{C}(\bV_{x_{k+1}}(\btheta_{1}))=\mathrm{C}(\bV_{x_{k+1}}(\btheta_{2}))$, we prove the conclusion for $n = k + 1$ by Lemma \ref{lem:the same span space}. Then we prove our conclusion. 
\end{proof}

\begin{lemma}\label{lem:bounded parameter}
    Under Assumption \ref{ass:ground truth r}, for a given initial parameter $\btheta_{1}$ satisfying $\left\|\btheta_{1}\right\| \le D$, any parameter $\btheta$ obtained by \eqref{eq:GD update rule} and \eqref{eq:SGD update rule} satisfies 
    \begin{equation}
        \small
        \begin{aligned}
            \left\|\btheta\right\| \le 3D.
        \end{aligned}
    \end{equation}
\end{lemma}
\begin{proof}
    We first analyze the case where the parameters are updated using \eqref{eq:GD update rule}. According to Lemma \ref{lem:gradient}, the learning rate $\eta=\frac{1}{L_{f}}$, for any $t\ge 1$, we have
    \begin{equation}\label{eq:nonexpansive of GA output}
        \small
        \begin{aligned}
            \btheta_{t+1}-\btheta^{*}=\btheta_{t}+\eta \nabla_{\btheta}J_{\bS}(\pi_{\btheta_{t}})-\btheta^{*}=\btheta_{t}-\btheta^{*}+\eta \bV_{\bS}(\btheta_{t})(\btheta^{*}-\btheta_{t})=\left(I-\eta \bV_{\bS}(\btheta_{t})\right)(\btheta_{t}-\btheta^{*}).
        \end{aligned}
    \end{equation}
    According to the definition of $\bV_{\bS}(\btheta_{t})$ and \eqref{eq:covariance norm}, for any $\bu$ satisfies $\|\bu\| = 1$, it holds that
    \begin{equation}
        \small 
        \begin{aligned}
            \bu^{\top}(I-\eta \bV_{\bS}(\btheta_{t}))\bu=\bu^{\top}\bu-\eta\cdot\bu^{\top}\bV_{\bS}(\btheta_{t})\bu \ge \bu^{\top}\bu(1-\eta\cdot C^{2})=1-\frac{1}{8RC+1} \ge 0.
        \end{aligned}
    \end{equation}
    Additionally, by the positive semi-definiteness of covariance matrix $\bV_{\bS}(\btheta)$, it holds that
    \begin{equation}
        \small
        \bu^{\top}(I-\eta \bV_{\bS}(\btheta_{t}))\bu=\bu^{\top}\bu-\eta\cdot\bu^{\top}\bV_{\bS}(\btheta_{t})\bu \le \bu^{\top}\bu =1,
    \end{equation}
    which means that the spectral radius of the matrix $I-\eta \bV_{\bS}(\btheta_{t})$ is no more than $1$. Combining with \eqref{eq:nonexpansive of GA output}, we have
    \begin{equation}
        \small
        \|\btheta_{t+1}\|\le\|\btheta_{t+1}-\btheta^{*}\|+\|\btheta^{*}\|=\|\left(I-\eta \bV_{\bS}(\btheta_{t})\right)(\btheta_{t}-\btheta^{*})\|+\|\btheta^{*}\| \le \|\btheta_{t}-\btheta^{*}\|+\|\btheta^{*}\| \le \|\btheta_{1}-\btheta^{*}\|+\|\btheta^{*}\| \le 3D.
    \end{equation}
    Analogously, the parameters obtained from \eqref{eq:SGD update rule} are also bounded by $\left\|\btheta\right\| \le 3D$.
\end{proof}
This lemma demonstrates that with proper initialization, the parameters obtained by gradient-based algorithms are naturally constrained to lie within a compact set.

\section{Some Technical Lemmas}
\concentrationinequality*
\begin{proof}
    According to Propositions \ref{pro:first order optimum is global optimum} and \ref{pro:expected optimum is empirical optimum}, it holds that
    \begin{equation*}
        \small
        J_{\bS}(\pi_{\btheta_{\bS}^{*}})=J_{\bS}(\pi_{\btheta^{*}}).
    \end{equation*}
    
    Let the dataset be $\small \bS = \{x_{1}, \cdots, x_{n}\}$. We define the random variable $X_{i}$ for each sample as: $X_{i} = f_{\btheta^{*}}(x_{i})$. Since the samples $x_{i}$ are drawn i.i.d. from the context space $\cX$, $X_{i}$ are independent random variables, therefore
    \begin{equation*}
        \small 
        J_{\bS}(\pi_{\btheta_{\bS}^{*}})=J_{\bS}(\pi_{\btheta^{*}})=\frac{1}{n}\sum_{i=1}^{n}X_{i}.
    \end{equation*}
    Then we can calculate the expectation of $X_{i}$ as follow:
    \begin{equation*}
        \small
        \mathbb{E}[X_{i}]=\mathbb{E}_{\bS}[J_{\bS}(\pi_{\btheta_{\bS}^{*}})]=\mathbb{E}_{\bS}[J_{\bS}(\pi_{\btheta^{*}})]=J(\pi_{\btheta^{*}}).
    \end{equation*}
    Furthermore, according to Lemma \ref{lem:bounded loss}, $X_{i}$ is bounded by:
    \begin{equation*}
        \small
        \left|X_{i}\right|=\left|f_{x_{i}}(\pi_{\btheta^{*}})\right| \le 3RC.
    \end{equation*}
    We apply the Chernoff-Hoeffding inequality for bounded random variables $X_{i}$. For the empirical mean of $n$ independent random variables $X_{i} \in [-3RC, 3RC]$, the probability of deviating from the expectation by more than $\epsilon$ is bounded by:
    \begin{equation}\label{eq:hoeffding inequality}
        \small
        \begin{aligned}
            \mathbb{P}\left( \left|J_{\bS}(\pi_{\btheta_{\bS}^{*}}) - J(\pi_{\btheta^{*}})\right| \ge \epsilon \right)=\mathbb{P}\left(\left|\sum_{i=1}^{n}\left(X_{i}-\mathbb{E}[X_{i}]\right)\right|\ge n\epsilon \right) \le 2 \exp\left( -\frac{2n\epsilon^{2}}{(6RC)^2} \right).
        \end{aligned}
    \end{equation}
    We set the failure probability in \eqref{eq:hoeffding inequality} to $\rho$, with a simple identity transformation, we complete the proof of this Lemma.
\end{proof}
\convergencerateofGA*
\begin{proof}
	Owing to the smoothness of $J_{\bS}(\pi_{\btheta})$, we know that 
	\begin{equation}
		\small
		\begin{aligned}
			J_{\bS}(\pi_{\btheta_{t + 1}}) - J_{\bS}(\pi_{\btheta_{t}}) 
             \geq \left\langle\nabla_{\btheta}J_{\bS}(\btheta_{t}), \btheta_{t + 1} - \btheta_{t}\right\rangle + \frac{L_{f}}{2}\|\btheta_{t + 1} - \btheta_{t}\|^{2} 
			= \frac{1}{2L_{f}}\|\nabla_{\btheta}J_{\bS}(\pi_{\btheta_{t}})\|^{2}.  
		\end{aligned}
	\end{equation}
	Taking a telescoping sum of the above inequality, we obtain 
	\begin{equation}
		\small
		\begin{aligned}
			\|\nabla_{\btheta}J_{\bS}(\pi_{\btheta_{\bS,T}^{\rm{GA}}})\|^{2} 
            \leq  \frac{2L_{f}}{T}(J_{\bS}(\pi_{\btheta_{T+1}}) - J_{\bS}(\pi_{\btheta_{1}})) 
            \leq  \frac{12L_{f}RC}{T},	
		\end{aligned}
	\end{equation}
	where the last inequality is obtained by Assumption \ref{ass:bounded} and Lemma \ref{lem:bounded loss}.  
\end{proof}


\paragraph{Lemma 4.}\emph{By taking $\eta_{t}=\frac{1}{2L_{f}\sqrt{t}}$\rm{(}$L_{f}$ is defined in Lemma \ref{lem:smoothness coefficient}\rm{)}, for any $\rho^{\prime} \in (0,1)$, with probability at least $1-\rho^{\prime}$, the parameter $\small \btheta_{\bS,T}^{\rm{SGA}}$ satisfies
     \begin{equation}
        \small
        \begin{aligned}
            \left\|\nabla_{\btheta}J_{\bS}(\pi_{\btheta_{\bS,T}^{\rm{SGA}}})\right\|^{2}
            \le \frac{6L_{f}}{\sqrt{T}}\left(\left(6RC+\frac{2R^{2}C^{4}}{L_{f}}\right)+\frac{2R^{2}C^{4}}{L_{f}}\log(T)+ \frac{1}{\lambda}\log\left(\frac{1}{\rho^{\prime}}\right)\right).
        \end{aligned}
    \end{equation}
    where $\lambda>0$ satisfies $\frac{e^{\lambda}-\lambda-1}{\lambda}\le\frac{L_{f}}{8R^{2}C^{4}}$.}
\begin{proof}
    Let the stochastic gradient noise be \[\small \nu_{t}=\nabla_{\btheta}f_{\btheta_{t}}(x_{i_{t}})-\nabla_{\btheta}J_{\bS}(\pi_{\btheta_{t}}).\] 
    By the $L_{f}-$smoothness of the empirical objective $J_{\bS}$ according to Lemma \ref{lem:smoothness coefficient}, for the SGA update rule in \eqref{eq:SGD update rule}, we have the following equality:
    \begin{equation}
        \small
        \begin{aligned}
            J_{\bS}(\pi_{\btheta_{t+1}})-J_{\bS}(\pi_{\btheta_{t}}) \ge& \left\langle \nabla_{\btheta}J_{\bS}(\pi_{\btheta_{t}}),\btheta_{t+1}-\btheta_{t}\right\rangle-\frac{L_{f}}{2}\left\|\btheta_{t+1}-\btheta_{t}\right\|^{2}
            \\ =&\left\langle \nabla_{\btheta}J_{\bS}(\pi_{\btheta_{t}}),\eta_{t}(\nu_{t}+\nabla_{\btheta}J_{\bS}(\pi_{\btheta_{t}}))\right\rangle-\frac{L_{f}}{2}\eta_{t}^{2}\left\|\nu_{t}+\nabla_{\btheta}J_{\bS}(\pi_{\btheta_{t}})\right\|^{2}
            \\ =& \left(\eta_{t}-\frac{L_{f}}{2}\eta_{t}^{2}\right)\left\|\nabla_{\btheta}J_{\bS}(\pi_{\btheta_{t}})\right\|^{2}-\frac{L_{f}}{2}\eta_{t}^{2}\|\nu_{t}\|^{2}+(\eta_{t}-L_{f}\eta_{t}^{2})\langle \nabla_{\btheta}J_{\bS}(\pi_{\btheta_{t}}),\nu_{t}\rangle.
        \end{aligned}
    \end{equation}
    By rearranging the terms, we have
    \begin{equation}\label{eq:single t}
        \small
        \begin{aligned}
            \left(\eta_{t}-\frac{L_{f}}{2}\eta_{t}^{2}\right)\left\|\nabla_{\btheta}J_{\bS}(\pi_{\btheta_{t}})\right\|^{2} 
            \le J_{\bS}(\pi_{\btheta_{t+1}})-J_{\bS}(\pi_{\btheta_{t}})+\frac{L_{f}}{2}\eta_{t}^{2}\|\nu_{t}\|^{2}-(\eta_{t}-L_{f}\eta_{t}^{2})\langle \nabla_{\btheta}J_{\bS}(\pi_{\btheta_{t}}),\nu_{t}\rangle.
        \end{aligned}
    \end{equation}
    Summing the \eqref{eq:single t} from $t=1$ to $T$, we obtain that:
    \begin{equation}\label{eq:sum t}
        \small
        \begin{aligned}
            \sum_{t=1}^{T}\left(\eta_{t}-\frac{L_{f}}{2}\eta_{t}^{2}\right)\left\|\nabla_{\btheta}J_{\bS}(\pi_{\btheta_{t}})\right\|^{2} &\le
            \Delta J+\frac{L_{f}}{2}\sum_{t=1}^{T}\eta_{t}^{2}\|\nu_{t}\|^{2}-\sum_{t=1}^{T}(\eta_{t}-L_{f}\eta_{t}^{2})\langle \nabla_{\btheta}J_{\bS}(\pi_{\btheta_{t}}),\nu_{t}\rangle
            \\ & \le \Delta J+\frac{L_{f}G^{2}}{2}\sum_{t=1}^{T}\frac{1}{4L_{f}^{2}t}-\sum_{t=1}^{T}(\eta_{t}-L_{f}\eta_{t}^{2})\langle \nabla_{\btheta}J_{\bS}(\pi_{\btheta_{t}}),\nu_{t}\rangle
            \\ &\le \Delta J+\frac{G^{2}}{8L_{f}}(1+\log(T))\underbrace{-\sum_{t=1}^{T}(\eta_{t}-L_{f}\eta_{t}^{2})\langle \nabla_{\btheta}J_{\bS}(\pi_{\btheta_{t}}),\nu_{t}\rangle}_{\rm{Term}\ 1}.
            \end{aligned}
    \end{equation}
    where $\Delta J=J_{\bS}(\pi_{\btheta_{T+1}})-J_{\bS}(\pi_{\btheta_{1}})$, the second equality is from Lemma \ref{lem:second-order moment} that for any $t$, we have $\|\nu_{t}\|^{2} \in [0,G^{2}]$. 
    
    We now need to prove the left term in \eqref{eq:sum t} is small with high probability. 

    By using Freedman’s inequality (Theorem 1.6 in\cite{freedman1975tail}), and combining $\mE[\langle \nabla_{\btheta}J_{\bS}(\btheta_{t}),\nu_{t} \rangle] = 0$, for any $\lambda \neq 0$, with probability at least $1-\rho_{1}$, we have 
    \begin{equation}\label{eq:term 1}
        \small
        \begin{aligned}
            -\sum_{t=1}^{T}m_{t}\langle \nabla_{\btheta}J_{\bS}(\btheta_{t}),\nu_{t} \rangle &\le \frac{e^{\lambda}-\lambda-1}{\lambda}\sum_{t=1}^{T}\mathbb{E}\left[\|m_{t}\langle \nabla_{\btheta}J_{\bS}(\btheta_{t}),\nu_{t} \rangle\|^{2}|\mathcal{F}_{t}\right]+\frac{1}{\lambda}\log\left(\frac{1}{\rho}\right)
            \\ & \le \frac{e^{\lambda}-\lambda-1}{\lambda}\sum_{t=1}^{T}\frac{\eta_{t}^{2}}{4}\|\nabla_{\btheta}J_{\bS}(\btheta_{t})\|^{2}\cdot\mathbb{E}\left[\|\nu_{t}\|^{2}|\mathcal{F}_{t-1}\right]+\frac{1}{\lambda}\log\left(\frac{1}{\rho}\right)
            \\ &\le \frac{e^{\lambda}-\lambda-1}{\lambda}\cdot \frac{G^{2}}{4}\cdot \sum_{t=1}^{T}\eta_{t}^{2}\|\nabla_{\btheta}J_{\bS}(\btheta_{t})\|^{2}+\frac{1}{\lambda}\log\left(\frac{1}{\rho}\right),
        \end{aligned}
    \end{equation}
    where $m_{t}=\eta_{t}-L_{f}\eta_{t}^{2}\le \frac{\eta_{t}}{2}$ for any $t \ge 1$.

    Combining \eqref{eq:term 1} with \eqref{eq:sum t}, it holds that
    \begin{equation}\label{eq:sum norm with lr}
        \small
        \begin{aligned}
            \sum_{t=1}^{T}\left(\eta_{t}-\frac{L_{f}}{2}\eta_{t}^{2}-\frac{e^{\lambda}-\lambda-1}{\lambda}\cdot \frac{G^{2}}{4}\eta_{t}^{2}\right)\left\|\nabla_{\btheta}J_{\bS}(\pi_{\btheta_{t}})\right\|^{2} \le \Delta J+\frac{G^{2}}{8L_{f}}(1+\log(T))+ \frac{1}{\lambda}\log\left(\frac{1}{\rho}\right).
        \end{aligned}
    \end{equation}
    Based on the property of the function $\frac{e^{\lambda}-\lambda-1}{\lambda}$, we can choose $\lambda$ such that $\frac{e^{\lambda}-\lambda-1}{\lambda}\le\frac{2L_{f}}{G^{2}}$. Therefore, for any $t\ge 1$, we have
    \begin{equation}\label{eq:simplify lr}
        \small
        \begin{aligned}
            \eta_{t}-\frac{L_{f}}{2}\eta_{t}^{2}-\frac{e^{\lambda}-\lambda-1}{\lambda}\cdot \frac{G^{2}}{4}\eta_{t}^{2} \ge \eta_{t}\left(1-\frac{L_{f}}{2}\frac{1}{2L_{f}\sqrt{t}}-\frac{2L_{f}}{G^{2}}\cdot \frac{G^{2}}{4}\frac{1}{2L_{f}\sqrt{t}}\right) \ge \frac{\eta_{t}}{2}.
        \end{aligned}
    \end{equation}
    Take \eqref{eq:simplify lr} into \eqref{eq:sum norm with lr}, then with probability at least $1-\rho_{1}$, it holds that
    \begin{equation}\label{eq:expected gradient norm with t}
        \small
        \begin{aligned}
            \frac{1}{\sum_{t=1}^{T}\frac{\eta_{t}}{2}}\sum_{t=1}^{T}\frac{\eta_{t}}{2}\left\|\nabla_{\btheta}J_{\bS}(\pi_{\btheta_{t}})\right\|^{2} 
            & \le \frac{1}{\sum_{t=1}^{T}\frac{\eta_{t}}{2}}\cdot \sum_{t=1}^{T}(\eta_{t}-\frac{L_{f}}{2}\eta_{t}^{2}-\frac{e^{\lambda}-\lambda-1}{\lambda}\cdot \frac{G^{2}}{4}\eta_{t}^{2})\left\|\nabla_{\btheta}J_{\bS}(\pi_{\btheta_{t}})\right\|^{2} 
            \\ & \le \frac{2L_{f}}{\sqrt{T+1}-1}\left(\Delta J+\frac{G^{2}}{8L_{f}}(1+\log(T))+ \frac{1}{\lambda}\log\left(\frac{1}{\rho}\right)\right)
            \\ & \le \frac{6L_{f}}{\sqrt{T}}\left(\left(\Delta J+\frac{G^{2}}{8L_{f}}\right)+\frac{G^{2}}{8L_{f}}\log(T)+ \frac{1}{\lambda}\log\left(\frac{1}{\rho}\right)\right).
        \end{aligned}
    \end{equation}
    According to the definition of $\btheta_{\bS,T}^{\rm{SGA}}$ in \ref{def:SGA output}, it holds that
    \begin{equation}\label{eq:SGA convergence rate}
        \small 
        \left\|\nabla_{\btheta}J_{\bS}(\pi_{\btheta_{\bS,T}^{\rm{SGA}}})\right\|^{2}\le\mathbb{E}_{\tau}\left[\left\|\nabla_{\btheta}J_{\bS}(\pi_{\btheta_{\tau}})\right\|^{2}\right] \le \frac{6L_{f}}{\sqrt{T}}\left(\left(\Delta J+\frac{G^{2}}{8L_{f}}\right)+\frac{G^{2}}{8L_{f}}\log(T)+ \frac{1}{\lambda}\log\left(\frac{1}{\rho}\right)\right).
    \end{equation}
    According to Lemmas \ref{lem:bounded loss} and \ref{lem:second-order moment}, we have $\small \Delta J \le 6RC$ and $\small G^{2}=16R^{2}C^{4}$. By combining with \eqref{eq:SGA convergence rate}, we prove our conclusion.
\end{proof}

\section{Proofs for Section \ref{sec:rl with sufficient prompt}}\label{app:sec_sufficient prompts}

\generalizationwithsufficientprompts*
\begin{proof}
	As clarified before, to make $\btheta_{\bS}^{*} = \btheta^{*}$, it is sufficient to show $\bV_{\bS}(\btheta^{*}_{\bS})$ is positive-definite. This can be done by combining the concentration of $\bV_{\bS}(\btheta^{*})$ and the continuity obtained from Lemma \ref{lem:Holder continuous}. Firstly, let us show the concentration result. By defining $\bV(\btheta) = \mE_{x}[\bV_{x}(\btheta)]$, and i.i.d. Rademacher random variable $\{\epsilon_{i}\}_{i=1}^{n}$ independent with $\bS$, for any $\gamma > 0$, we have 
	\begin{equation}
		\small
		\begin{aligned}
			\mE\left[\exp\left(\gamma\|\bV_{\bS}(\btheta) - \bV(\btheta)\|\right)\right] = & \mE\left[\exp\left(\gamma\sup_{\bu:\|\bu\| = 1}\bu^{\top}\left(\bV_{\bS}(\btheta) - \bV(\btheta)\right)\bu\right)\right]	\\
			\leq & \mE\left[\exp\left(2\gamma\sup_{\bu:\|\bu\|=1}\bu^{\top}\left(\frac{1}{n}\sum_{i=1}^{n}\epsilon_{i}\bV_{x_{i}}(\btheta)\right)\bu\right)\right] \\
			= & \mE\left[\mE\left[\exp\left(2\gamma\sup_{\bu:\|\bu\|=1}\frac{1}{n}\sum_{i=1}^{n}\epsilon_{i}\bu^{\top}\bV_{x_{i}}(\btheta)\bu\right)\mid \bS\right]\right] \\
			\leq & \exp\left(\frac{8C^{4}\gamma^{2}}{n}\right),
		\end{aligned}
	\end{equation}
	where the first inequality is from Proposition 4.11 (b) in \cite{wainwright2019high}, the last inequality is obtained from $\|\bV_{x}(\btheta)\|\leq C^{2}$, and the sub-Gaussian property of Rademacher random variable. With the above inequality, by applying Markov's inequality, we have 
	\begin{equation}\label{eq:single concentration}
		\small
		\begin{aligned}
			\bbP\left(\left\|\bV_{\bS}(\btheta) - \bV(\btheta)\right\| \geq \delta\right) \leq \inf_{\gamma}\exp\left(-\gamma\delta + \frac{8C^{4}\gamma^{2}}{n}\right) 
            = \exp\left(-\frac{n\delta^{2}}{32C^{4}}\right),
		\end{aligned}
	\end{equation} 
	by taking $\gamma = \frac{n\delta}{16C^{4}}$. Since $\rm{diam}(\Theta_{R}) \leq R$, from \citep{zhang2017empirical} we know that $\Theta_{R}$ has a $r$-covering $\cC$ such that for any $\btheta\in\Theta_{R}$ there exists a $\btheta^{\prime}\in\cC$ satisfies $\|\btheta - \btheta^{\prime}\| \leq r$ and $|\cC|\leq (3R/r)^{d}$. Then by applying Bonferroni inequality and \eqref{eq:single concentration}, we have 
	\begin{equation}\label{eq:uniform concentration}
		\small
        \begin{aligned}
            \bbP\left(\sup_{\btheta\in\cC}\left\|\bV_{\bS}(\btheta) - \bV(\btheta)\right\| \geq \delta\right) 
            \leq |\cC|\exp\left(-\frac{n\delta^{2}}{32C^{4}}\right) 
            \leq \exp\left(d\log{\left(\frac{3R}{r}\right)} - \frac{n\delta^{2}}{32C^{4}}\right). 
        \end{aligned}
	\end{equation}
	By applying Lemma \ref{lem:Holder continuous}, we know that when $\{\sup_{\btheta\in\cC}\left\|\bV_{\bS}(\btheta) - \bV(\btheta)\right\| \leq \delta\}$ happens, for any $\btheta\in\Theta_{R}$ and its corresponded $\btheta^{\prime}\in\cC$, 
	\begin{equation}
		\small
		\begin{aligned}
			\|\bV_{\bS}(\btheta) - \bV(\btheta)\| 
            \leq \|\bV_{\bS}(\btheta) - \bV_{\bS}(\btheta^{\prime})\| + \|\bV(\btheta) - \bV(\btheta^{\prime})\|	+ \|\bV_{\bS}(\btheta^{\prime}) - \bV(\btheta^{\prime})\| 
			 \leq 6C^{\frac{5}{2}}r^{\frac{1}{2}} + \delta.
		\end{aligned}
	\end{equation}
	Then, by Weyl’s theorem, we have 
	\begin{equation}
		\small
		  \begin{aligned}
		      |\sigma_{\min}(\bV_{\bS}(\btheta)) - \sigma_{\min}(\bV(\btheta))| 
                \leq \|\bV_{\bS}(\btheta) - \bV(\btheta)\| 
              \leq 6C^{\frac{5}{2}}r^{\frac{1}{2}} + \delta,
		  \end{aligned} 
	\end{equation}
	which implies $\sigma_{\min}(\bV_{\bS}(\btheta)) > 0$ by taking $r = (1 / 6C^{\frac{5}{2}})^{2}\delta^{2}$ and $\delta < \sigma / 2$. By plugging these into \eqref{eq:uniform concentration}, we obtain our conclusion. 
\end{proof}

\section{Proofs for Section \ref{sec:rl without sufficient prompts}}\label{app:sec_insufficient prompts}

\paragraph{Theorem 2.}\emph{Under Assumptions \ref{ass:bounded}, \ref{ass:ground truth r} and \ref{ass:span space}, we have 
\begin{equation}
    \small
    \cE_{\rm{stab}}(\btheta_{\bS}^{*}) \leq \left(\frac{\Gamma_{1}}{n}+2RC\sqrt{\frac{\Gamma_{1}}{n}}\right),
\end{equation}
where \[\small\Gamma_{1}=\sup_{\bS,\bS^{\prime}}\left\{4nR\epsilon_{n}+\frac{8RC^{3}\cdot \sigma_{\max}(\bV_{\bS}(\btheta_{\bS}^{*}))}{\sigma_{\min}^{+}(\bV_{\bS}(\btheta_{\bS}^{*}))\sigma_{\min}^{+}(\bV_{\bS}(\btheta_{\bS^{\prime}}^{*}))}\right\}.\]
Then,
for any $\rho \in (0,1)$, the following bound holds with probability at least $1-\rho$ over the randomness of $\bS$,  
    \begin{equation}
        \small
        \begin{aligned}
            \left|J_{\bS}(\pi_{\btheta_{\bS}^{*}})-J(\pi_{\btheta_{\bS}^{*}})\right|
             \le \cE_{\rm{stab}}(\btheta_{\bS}^{*})\log (n)\log\left(\frac{1}{\rho}\right)+6RC\sqrt{\frac{\log(1/\rho)}{n}} 
             = \tilde{\cO}\left(\frac{1}{\sqrt{n}}\right).
        \end{aligned}
    \end{equation}}

\begin{proof}
    As clarified in Section \ref{sec:suboptimality}, the generalization error is implied by invoking uniform algorithmic stability. To obtain this, for any $\btheta_{\bS}^{*}$ and $\btheta_{\bS}^{*}$ satisfies $\small \|\nabla_{\btheta}J_{\bS}(\btheta_{\bS}^{*})\|=\|\nabla_{\btheta}J_{\bS^{\prime}}(\btheta_{\bS^{\prime}}^{*})\|=0$, we explore the algorithmic stability. We first evaluate the gradient norm of $\btheta_{\bS^{\prime}}^{*}$ with respect to the objective $J_{\bS}$. According to Lemma \ref{lem:gradient norm between two different sets}, it holds that
    \begin{equation*}
        \small
        \begin{aligned}
            \left\|\nabla_{\btheta}J_{\bS}(\btheta_{\bS^{\prime}}^{*})\right\|=\|\nabla_{\btheta}J_{\bS}(\btheta_{\bS^{\prime}}^{*})-\nabla_{\btheta}J_{\bS^{\prime}}(\btheta_{\bS^{\prime}}^{*})+\nabla_{\btheta}J_{\bS^{\prime}}(\btheta_{\bS^{\prime}}^{*})\|\le \|\nabla_{\btheta}J_{\bS}(\btheta_{\bS^{\prime}}^{*})-\nabla_{\btheta}J_{\bS^{\prime}}(\btheta_{\bS^{\prime}}^{*})\|+\|\nabla_{\btheta}J_{\bS^{\prime}}(\btheta_{\bS^{\prime}}^{*})\| \le\frac{4RC^{2}}{n}.
        \end{aligned}
    \end{equation*}
    By using Lemma \ref{lem:key lemma}, it holds that
    \begin{equation}
        \small
        \begin{aligned}
           \cE_{\rm{stab}}(\btheta_{\bS}^{*})= & \sup_{x\in \mathcal{X},\bS,\bS^{\prime}}|f_{\btheta_{\bS}^{*}}(x)-f_{\btheta_{\bS^{\prime}}^{*}}(x)| 
            \\ \le & \sup_{\bS,\bS^{\prime}}\left\{4R\epsilon_{n}+\frac{8RC^{3}\cdot \sigma_{\max}(\bV_{\bS}(\btheta_{\bS}^{*}))}{n\sigma_{\min}^{+}(\bV_{\bS}(\btheta_{\bS}^{*}))\sigma_{\min}^{+}(\bV_{\bS}(\btheta_{\bS^{\prime}}^{*}))}\right\}
             +2CR\sup_{\bS,\bS^{\prime}}\left\{\sqrt{4R\epsilon_{n}+\frac{8RC^{3}\cdot \sigma_{\max}(\bV_{\bS}(\btheta_{\bS}^{*}))}{n\sigma_{\min}^{+}(\bV_{\bS}(\btheta_{\bS}^{*}))\sigma_{\min}^{+}(\bV_{\bS}(\btheta_{\bS^{\prime}}^{*}))}}\right\}
            \\ = & \frac{\Gamma_{1}}{n}+2CR\sqrt{\frac{\Gamma_{1}}{n}},
        \end{aligned}
    \end{equation}
    which proves our conclusion by invoking Proposition \ref{pro:generalization bound}.
\end{proof}

\paragraph{Theorem 3.}\emph{Under Assumption \ref{ass:bounded}, \ref{ass:ground truth r}, \ref{ass:span space}, for any $\rho \in (0,1)$, the following bound holds with probability at least $1-\rho$ over the randomness of $\bS$, 
    \begin{equation}
        \small 
        \begin{aligned}
            |J(\pi_{\btheta_{\bS}^{*}})-J(\pi_{\btheta^{*}})| 
            \le \left(\frac{\Gamma_{1}}{n}+2RC\sqrt{\frac{\Gamma_{1}}{n}}\right)\log (n)\log\left(\frac{2}{\rho}\right)
            +6RC\sqrt{\frac{\log(2/\rho)}{n}}
            +3RC\sqrt{\frac{2\log(4/\rho)}{n}}
            = \tilde{\cO}\left(n^{-\frac{1}{2}}\right).
        \end{aligned}
    \end{equation}}
\begin{proof}
    Firstly, we decompose the suboptimality gap of $\pi_{\btheta_{\bS}^{*}}$ as follows:
    \begin{equation}\label{eq:exact optimization decomposition}
        \small
        \begin{aligned}
            \left|J(\pi_{\btheta_{\bS}^{*}})-J(\pi_{\btheta^{*}})\right|
            \le & \left|J(\pi_{\btheta_{\bS}^{*}})-J_{\bS}(\pi_{\btheta_{\bS}^{*}})\right|+\left|J_{\bS}(\pi_{\btheta_{\bS}^{*}})-J(\pi_{\btheta^{*}})\right|.
        \end{aligned}
    \end{equation}
    Then, according to Theorem \ref{thm:generalization of exact optimization}, for $\rho_{1} \in (0,1)$, with probability at least $1-\rho_{1}$, it holds that
    \begin{equation}\label{eq:exact optimization decomposition term1}
        \small
        \begin{aligned}
            \left|J(\pi_{\btheta_{\bS}^{*}})-J_{\bS}(\pi_{\btheta_{\bS}^{*}})\right| \le \left(\frac{\Gamma_{1}}{n}+2RC\sqrt{\frac{\Gamma_{1}}{n}}\right)\log(n)\log\left(\frac{1}{\rho_{1}}\right)+6RC\sqrt{\frac{\log(1/\rho_{1})}{n}}.
        \end{aligned}
    \end{equation}
    Now we turn to the second term in \eqref{eq:exact optimization decomposition}. By using Lemma \ref{lem:concentration of empirical optimal}, for $\rho_{2} \in (0,1)$, with probability at least $1-\rho_{2}$, it holds that
    \begin{equation}\label{eq:exact optimization decomposition term2}
        \small
        \begin{aligned}
            \left|J_{\bS}(\pi_{\btheta_{\bS}^{*}})-J(\pi_{\btheta^{*}})\right| \le 3RC\sqrt{\frac{2\log(2/\rho_{2})}{n}}.
        \end{aligned}
    \end{equation}
    Let $\rho_{1}=\rho_{2}=\frac{\rho}{2}$, by combining \eqref{eq:exact optimization decomposition term1} and \eqref{eq:exact optimization decomposition term2}, we complete the proof.
\end{proof}



\section{Proofs for Section \ref{sec:suboptimality of GA}}\label{app:sec_GA}

\optimizationofGA*
\begin{proof}
    According to Lemma \ref{lem:convergence of GA}, we have 
    \begin{equation*}
        \small
        \left\|\nabla_{\btheta}J_{\bS}(\pi_{\btheta_{\bS,T}^{\rm{GA}}})\right\| \le 2\sqrt{\frac{3L_{f}RC}{T}} \quad \text{and} \quad \left\|\nabla_{\btheta}J_{\bS}(\pi_{\btheta_{\bS}^{*}})\right\|=0.
    \end{equation*}
    By using Lemma \ref{lem:key lemma}, it holds that
    \begin{equation*}
        \small
        \begin{aligned}
            & \left|J_{\bS}(\pi_{\btheta_{\bS,T}^{\rm{GA}}})-J_{\bS}(\pi_{\btheta_{\bS}^{*}})\right| \le \frac{1}{n}\sum_{i=1}^{n}\left|f_{\btheta_{\bS,T}^{\rm{GA}}}(x_{i})-f_{\btheta_{\bS}^{*}}(x_{i})\right|\le \sup_{x\in \mathcal{X}}\left|f_{\btheta_{\bS,T}^{\rm{GA}}}(x)-f_{\btheta_{\bS}^{*}}(x)\right| 
            \\ \le & 4R\epsilon_{n}+4C\sqrt{\frac{3L_{f}RC}{T}}\frac{ \sigma_{\max}(\bV_{\bS}(\btheta_{\bS}^{*}))}{\sigma_{\min}^{+}(\bV_{\bS}(\btheta_{\bS}^{*}))\sigma_{\min}^{+}(\bV_{\bS}(\btheta_{\bS,T}^{\rm{GA}}))}+2RC\sqrt{4R\epsilon_{n}+4C\sqrt{\frac{3L_{f}RC}{T}}\frac{ \sigma_{\max}(\bV_{\bS}(\btheta_{\bS}^{*}))}{\sigma_{\min}^{+}(\bV_{\bS}(\btheta_{\bS}^{*}))\sigma_{\min}^{+}(\bV_{\bS}(\btheta_{\bS,T}^{\rm{GA}}))}}.
        \end{aligned}
    \end{equation*}
    We highlight that the bound above is of the form $X + \alpha\sqrt{X}$. In the regime where $n$ and $T$ are large, $X$ is small enough such that the linear term $X$ becomes negligible compared to the square root term $\sqrt{X}$ (i.e., $X = o(\sqrt{X})$). Thus, the bound is dominated by the square root component. Applying the inequality $\sqrt{a+b} \le \sqrt{a} + \sqrt{b}$, we obtain that
    \begin{equation*}
        \small
        \left|J_{\bS}(\pi_{\btheta_{\bS,T}^{\rm{GA}}})-J_{\bS}(\pi_{\btheta_{\bS}^{*}})\right| \lesssim \sqrt{\epsilon_{n}}+\sqrt{\Gamma^{\rm{GA}}_{\bS,\bS}}T^{-\frac{1}{4}} =\cO(T^{-\frac{1}{4}}+n^{-\frac{1}{2}}).
    \end{equation*}
    where $\Gamma^{\rm{GA}}_{\bS,\bS}=\sup_{\bS}\frac{ \sigma_{\max}(\bV_{\bS}(\btheta_{\bS}^{*}))}{\sigma_{\min}^{+}(\bV_{\bS}(\btheta_{\bS}^{*}))\sigma_{\min}^{+}(\bV_{\bS}(\btheta_{\bS,T}^{\rm{GA}}))}$.
    
    We complete the proof of this theorem.
\end{proof}

\paragraph{Theorem 4.}\emph{Under Assumptions \ref{ass:bounded}, \ref{ass:ground truth r} and \ref{ass:span space}, for the GA output $\btheta_{\bS,T}^{\rm{GA}}$ defined in Definition \ref{def:GA output}, we have 
        \[ \small \cE_{\rm{stab}}(\btheta_{\bS,T}^{GA}) \lesssim \sqrt{\epsilon_n} + \left(\sqrt{\Gamma^{\rm{GA}}_{\bS,\bS}}+\sqrt{\Gamma^{\rm{GA}}_{\bS,\bS^{\prime}}}\right) \left( T^{-\frac{1}{4}} + n^{-\frac{1}{2}} \right), \]
        where \[\small \Gamma^{\rm{GA}}_{\bS,\bS^{\prime}}=\sup_{\bS,\bS^{\prime}}\frac{\sigma_{\max}(\bV_{\bS}(\btheta_{\bS}^{*}))}{\sigma_{\min}^{+}(\bV_{\bS}(\btheta_{\bS}^{*}))\sigma_{\min}^{+}(\bV_{\bS}(\btheta_{\bS^{\prime},T}^{\rm{GA}}))}
        \quad \rm{and} \quad
        \small \Gamma^{\rm{GA}}_{\bS,\bS}=\sup_{\bS}\frac{\sigma_{\max}(\bV_{\bS}(\btheta_{\bS}^{*}))}{\sigma_{\min}^{+}(\bV_{\bS}(\btheta_{\bS}^{*}))\sigma_{\min}^{+}(\bV_{\bS}(\btheta_{\bS,T}^{\rm{GA}}))}.\] 
       Then, for any $\rho \in (0,1)$, the following bound holds with probability at least $1-\rho$ over the randomness of $\bS$:
		\begin{equation}
			\small
			\begin{aligned}
				\left|J(\pi_{\btheta_{\bS,T}^{\rm{GA}}})-J_{\bS}(\pi_{\btheta_{\bS,T}^{\rm{GA}}})\right| 
				\le  \cE_{\rm{stab}}(\btheta_{\bS,T}^{GA})\log (n)\log\left(\frac{1}{\rho}\right)+6RC\sqrt{\frac{\log(1/\rho)}{n}} 
				= \tilde{\cO}\left(T^{-\frac{1}{4}}+n^{-\frac{1}{2}}\right).
			\end{aligned}
		\end{equation}}
\begin{proof}
    Similarly to Theorem \ref{thm:generalization of exact optimization}, we prove the algorithmic stability. Given $\bS,\bS^{\prime}$, we can obtain $\btheta_{\bS,T}^{*}$ and $\btheta_{\bS^{\prime},T}^{*}$ defined in Definition \ref{def:GA output}, respectively. According to Lemma \ref{lem:convergence of GA}, we have
    \begin{equation*}
        \small
        \begin{aligned}
            \left\|\nabla_{\btheta}J_{\bS}(\pi_{\btheta_{\bS,T}^{\rm{GA}}})\right\| \leq 2\sqrt{\frac{3L_{f}RC}{T}} \quad \text{and} \quad \left\|\nabla_{\btheta}J_{\bS^{\prime}}(\pi_{\btheta_{\bS^{\prime},T}^{\rm{GA}}})\right\| \leq 2\sqrt{\frac{3L_{f}RC}{T}}.
        \end{aligned}
    \end{equation*}
    According to Lemma \ref{lem:gradient norm between two different sets}, it holds that
    \begin{equation*}
        \small
        \begin{aligned}
            \left\|\nabla_{\btheta}J_{\bS}(\pi_{\btheta_{\bS^{\prime},T}^{\rm{GA}}})\right\| \le \left\|\nabla_{\btheta}J_{\bS}(\pi_{\btheta_{\bS^{\prime},T}^{\rm{GA}}})-\nabla_{\btheta}J_{\bS^{\prime}}(\pi_{\btheta_{\bS^{\prime},T}^{\rm{GA}}})\right\|+\left\|\nabla_{\btheta}J_{\bS^{\prime}}(\pi_{\btheta_{\bS^{\prime},T}^{\rm{GA}}})\right\| \le 2\sqrt{\frac{3L_{f}RC}{T}}+\frac{4RC^{2}}{n}.
        \end{aligned}
    \end{equation*}
    By using Lemma \ref{lem:key lemma}, it holds that
    \begin{equation}
        \small
        \begin{aligned}
            &\cE_{\text{stab}}(\btheta_{\bS,T}^{\rm{GA}}) = \sup_{x\in \mathcal{X},\bS,\bS^{\prime}}|f_{\btheta_{\bS,T}^{\rm{GA}}}(x)-f_{\btheta_{\bS^{\prime},T}^{\rm{GA}}}(x)| 
            \\ \le & \sup_{\bS,\bS^{\prime}}\left\{4R\epsilon_{n}+4C\sqrt{\frac{3L_{f}RC}{T}} \Gamma^{\rm{GA}}_{\bS,\bS}+\left(4C\sqrt{\frac{3L_{f}RC}{T}}+\frac{8RC^{3}}{n}\right)\Gamma^{\rm{GA}}_{\bS,\bS^{\prime}}\right\}
            \\ &+2RC\sup_{\bS,\bS^{\prime}}\left\{\sqrt{4R\epsilon_{n}+4C\sqrt{\frac{3L_{f}RC}{T}} \Gamma^{\rm{GA}}_{\bS,\bS}+\left(4C\sqrt{\frac{3L_{f}RC}{T}}+\frac{8RC^{3}}{n}\right)\Gamma^{\rm{GA}}_{\bS,\bS^{\prime}}}\right\}
            \\ \le & \sup_{\bS,\bS^{\prime}}\left\{4R\epsilon_{n}+4C\sqrt{\frac{3L_{f}RC}{T}} \Gamma^{\rm{GA}}_{\bS,\bS}+\left(4C\sqrt{\frac{3L_{f}RC}{T}}+\frac{8RC^{3}}{n}\right)\Gamma^{\rm{GA}}_{\bS,\bS^{\prime}}\right\}
            \\ &+2RC\sup_{\bS,\bS^{\prime}}\left\{\sqrt{4R\epsilon_{n}}+\sqrt{4C\sqrt{\frac{3L_{f}RC}{T}} \left(\Gamma^{\rm{GA}}_{\bS,\bS}+\Gamma^{\rm{GA}}_{\bS,\bS^{\prime}}\right)}+\sqrt{\frac{8RC^{3}}{n}\Gamma^{\rm{GA}}_{\bS,\bS^{\prime}}}\right\}
            \\ \lesssim &  \sqrt{\epsilon_n} + \left(\sqrt{\Gamma^{\rm{GA}}_{\bS,\bS}}+\sqrt{\Gamma^{\rm{GA}}_{\bS,\bS^{\prime}}}\right) \left( T^{-\frac{1}{4}} + n^{-\frac{1}{2}} \right).
        \end{aligned}
    \end{equation}
    where $\Gamma^{\rm{GA}}_{\bS,\bS}=\sup_{\bS}\frac{\sigma_{\max}(\bV_{\bS}(\btheta_{\bS}^{*}))}{\sigma_{\min}^{+}(\bV_{\bS}(\btheta_{\bS}^{*}))\sigma_{\min}^{+}(\bV_{\bS}(\btheta_{\bS,T}^{\rm{GA}}))}$ and $\Gamma^{\rm{GA}}_{\bS,\bS^{\prime}}=\sup_{\bS,\bS^{\prime}}\frac{\sigma_{\max}(\bV_{\bS}(\btheta_{\bS}^{*}))}{\sigma_{\min}^{+}(\bV_{\bS}(\btheta_{\bS}^{*}))\sigma_{\min}^{+}(\bV_{\bS}(\btheta_{\bS^{\prime},T}^{\rm{GA}}))}$. 

    We can now complete the proof of the theorem by applying Proposition \ref{pro:generalization bound}.
\end{proof}

\errorboundofGA*
\begin{proof}
    Recall the suboptimality gap decomposition in Definition \ref{def:decomposition}, we have
    \begin{equation*}
        \small
        \begin{aligned}
            \left|J(\pi_{\btheta_{\bS,T}^{\rm{GA}}})-J(\pi_{\btheta^{*}})\right|\le\left|J(\pi_{\btheta_{\bS,T}^{\rm{GA}}})-J_{\bS}(\pi_{\btheta_{\bS,T}^{\rm{GA}}})\right|+\left|J_{\bS}(\pi_{\btheta_{\bS,T}^{\rm{GA}}})-J_{\bS}(\pi_{\btheta_{\bS}^{*}})\right|+\left|J_{\bS}(\pi_{\btheta_{\bS}^{*}})-J(\pi_{\btheta^{*}})\right|.
        \end{aligned}
    \end{equation*}
    By combining the results of Theorem \ref{thm:generalization of GA}, Lemma \ref{lem:optimization of GA} and Lemma \ref{lem:concentration of empirical optimal}, we complete the proof of this theorem. 
\end{proof}

\section{Proofs for Section \ref{sec:suboptimality of SGA}}\label{app:sec_SGA}

When we consider a randomized algorithm $\mathcal{H}:\mathcal{X}^{n}\times\Omega \rightarrow \mathbb{R}^{d}$, the standard notion of uniform stability defined in Definition \ref{def:uniform stability} is insufficient to provide a high-probability generalization guarantee. Instead, we adopt the notion of uniform stability with high probability \citep{yuan2023l_2}.
\begin{definition}\label{def:uniform stability for randomized algorithms}
    A randomized algorithm $\mathcal{H}$ has $\cE_{\rm{stab}}-$uniform stability with probability at least $1-\rho^{\prime}$ for some $\rho^{\prime}$ over the random draw of $\zeta \in \Omega$, if for any two datasets $\bS, \bS^{\prime} \in \mathcal{X}^{n}$ that differ by at most one example, it holds that:
    \begin{equation}
        \small
        \begin{aligned}
            \mathbb{P}\left\{\sup_{x\in \mathcal{X},\bS,\bS^{\prime}}|f_{\mathcal{H}(\bS,\zeta)}(x)-f_{\mathcal{H}(\bS^{\prime},\zeta)}(x)|\le \cE_{\rm{stab}}\right\} \ge 1-\rho^{\prime}.
        \end{aligned}
    \end{equation}
\end{definition}

\begin{corollary}\label{cor:generalization via stability for rondomized algorithm}
    Suppose that the randomness of $\mathcal{H}$ is independent of the training set $\bS$. Then the bound in Proposition\ref{pro:generalization bound} naturally implies that with probability at least $1-\rho-\rho^{\prime}$ over $\bS$ and $\zeta$,
    \begin{equation}
        \small
        \begin{aligned}
             \left|J(\pi_{\mathcal{H}(\bS,\zeta)}) - J_{\bS}(\pi_{\mathcal{H}(\bS,\zeta)})\right| 
             \lesssim  \cE_{\rm{stab}}\log(n)\log\left(\frac{1}{\rho}\right)+6 RC\sqrt{\frac{\log\left(1/\rho\right)}{n}} + \cO(\rho^{\prime}).
        \end{aligned}
    \end{equation}
\end{corollary}

\paragraph{Theorem 6.}\emph{Under Assumptions \ref{ass:bounded}, \ref{ass:ground truth r} and \ref{ass:span space}, for any $\rho^{\prime} \in (0,1)$ and parameter $\btheta_{\bS,T}^{\rm{SGA}}$ , we have
 \begin{equation*}
     \small
     \cE_{\rm{stab}}(\btheta_{\bS,T}^{\rm{SGA}}) \lesssim \sqrt{\epsilon_n} + \left(\sqrt{\Gamma^{\rm{SGA}}_{\bS,\bS}}+\sqrt{\Gamma^{\rm{SGA}}_{\bS,\bS^{\prime}}}\right) \left( T^{-\frac{1}{8}} + n^{-\frac{1}{2}} \right),
 \end{equation*}
 where \[\small\Gamma^{\rm{SGA}}_{\bS,\bS}=\sup_{\bS}\frac{\sigma_{\max}(\bV_{\bS}(\btheta_{\bS}^{*}))}{\sigma_{\min}^{+}(\bV_{\bS}(\btheta_{\bS}^{*}))\sigma_{\min}^{+}(\bV_{\bS}(\btheta_{\bS,T}^{\rm{SGA}}))}
 \quad \rm{and}\quad
 \small\Gamma^{\rm{SGA}}_{\bS,\bS^{\prime}}=\sup_{\bS,\bS^{\prime}}\frac{\sigma_{\max}(\bV_{\bS}(\btheta_{\bS}^{*}))}{\sigma_{\min}^{+}(\bV_{\bS}(\btheta_{\bS}^{*}))\sigma_{\min}^{+}(\bV_{\bS}(\btheta_{\bS^{\prime},T}^{\rm{SGA}}))}.\]
 Then, for any $\rho \in (0,1)$, the following bound holds with probability at least $1-\rho-\rho^{\prime}$ over the randomness of $\bS$ and $i_{t}$ in SGA:
    \begin{equation}
        \small
        \begin{aligned}
            \left|J(\pi_{\btheta_{\bS,T}^{\rm{SGA}}})-J_{\bS}(\pi_{\btheta_{\bS,T}^{\rm{SGA}}})\right| 
            \le \cE_{\rm{stab}}(\btheta_{\bS,T}^{\rm{SGA}})\log (n)\log\left(\frac{1}{\rho}\right)+6RC\sqrt{\frac{\log(1/\rho)}{n}} 
            = \tilde{\cO}\left(T^{-\frac{1}{8}}+n^{-\frac{1}{2}}\right).
        \end{aligned}
    \end{equation}}
\begin{proof}
    We first analyze the uniform stability of the algorithm. For the given dataset $\bS,\bS^{\prime}$, we can obtain two parameters $\btheta_{\bS,T}^{\rm{SGA}}$ and $\btheta_{\bS^{\prime},T}^{\rm{SGA}}$ defined in Definition \ref{def:SGA output}, respectively. According to Lemma \ref{lem:gradient norm of SGA}, with probability at least $1-\frac{\rho^{\prime}}{2}$, it holds that:
    \begin{equation*}
        \small
        \begin{aligned}
            &\left\|\nabla_{\btheta}J_{\bS}(\pi_{\btheta_{\bS,T}^{\rm{SGA}}})\right\| \leq \sqrt{\frac{6L_{f}}{\sqrt{T}}\left(\left(6RC+\frac{2R^{2}C^{4}}{L_{f}}\right)+\frac{2R^{2}C^{4}}{L_{f}}\log(T)+ \frac{1}{\lambda}\log\left(\frac{2}{\rho^{\prime}}\right)\right)} \triangleq \Delta(T,\rho^{\prime})
            \\ \text{and} \quad &\left\|\nabla_{\btheta}J_{\bS^{\prime}}(\pi_{\btheta_{\bS^{\prime},T}^{\rm{SGA}}})\right\| \leq \sqrt{\frac{6L_{f}}{\sqrt{T}}\left(\left(6RC+\frac{2R^{2}C^{4}}{L_{f}}\right)+\frac{2R^{2}C^{4}}{L_{f}}\log(T)+ \frac{1}{\lambda}\log\left(\frac{2}{\rho^{\prime}}\right)\right)}=\Delta(T,\rho^{\prime}).
        \end{aligned}
    \end{equation*}    
    Furthermore, according to Lemma \ref{lem:gradient norm between two different sets}, it holds that
    \begin{equation*}
        \small
        \begin{aligned}
            \left\|\nabla_{\btheta}J_{\bS}(\pi_{\btheta_{\bS^{\prime},T}^{\rm{SGA}}})\right\| & \le \left\|\nabla_{\btheta}J_{\bS}(\pi_{\btheta_{\bS^{\prime},T}^{\rm{SGA}}})-\nabla_{\btheta}J_{\bS^{\prime}}(\pi_{\btheta_{\bS^{\prime},T}^{\rm{SGA}}})\right\|+\left\|\nabla_{\btheta}J_{\bS^{\prime}}(\pi_{\btheta_{\bS^{\prime},T}^{\rm{SGA}}})\right\| 
            \\ & \le \sqrt{\frac{6L_{f}}{\sqrt{T}}\left(\left(6RC+\frac{2R^{2}C^{4}}{L_{f}}\right)+\frac{2R^{2}C^{4}}{L_{f}}\log(T)+ \frac{1}{\lambda}\log\left(\frac{2}{\rho^{\prime}}\right)\right)}+\frac{4RC^{2}}{n} = \Delta(T,\rho^{\prime})+\frac{4RC^{2}}{n}.
        \end{aligned}
    \end{equation*}
    According to Definition \ref{def:uniform stability for randomized algorithms}, by using Lemma \ref{lem:key lemma}, with probability at least $1-\rho^{\prime}$, it holds that:
    \begin{equation}
        \small
        \begin{aligned}
            \cE_{\rm{stab}}(\btheta_{\bS,T}^{\rm{SGA}}) \le & \sup_{\bS,\bS^{\prime}}\left\{4R\epsilon_{n}+2C\Delta(T,\rho^{\prime}) \Gamma^{\rm{SGA}}_{\bS,\bS}+2C\left(\Delta(T,\rho^{\prime})+\frac{4RC^{2}}{n}\right)\Gamma^{\rm{SGA}}_{\bS,\bS^{\prime}}\right\}
            \\ &+2RC\sup_{\bS,\bS^{\prime}}\left\{\sqrt{4R\epsilon_{n}+2C\Delta(T,\rho^{\prime}) \Gamma^{\rm{SGA}}_{\bS,\bS}+2C\left(\Delta(T,\rho^{\prime})+\frac{4RC^{2}}{n}\right)\Gamma^{\rm{SGA}}_{\bS,\bS^{\prime}}}\right\}
            \\ \le & \sup_{\bS,\bS^{\prime}}\left\{4R\epsilon_{n}+2C\Delta(T,\rho^{\prime}) \Gamma^{\rm{SGA}}_{\bS,\bS}+2C\left(\Delta(T,\rho^{\prime})+\frac{4RC^{2}}{n}\right)\Gamma^{\rm{SGA}}_{\bS,\bS^{\prime}}\right\}
            \\ &+2RC\sup_{\bS,\bS^{\prime}}\left\{\sqrt{4R\epsilon_{n}}+\sqrt{2C\Delta(T,\rho^{\prime}) \left(\Gamma^{\rm{SGA}}_{\bS,\bS}+\Gamma^{\rm{SGA}}_{\bS,\bS^{\prime}}\right)}+\sqrt{\frac{8RC^{3}}{n}\Gamma^{\rm{SGA}}_{\bS,\bS^{\prime}}}\right\}
            \\ \lesssim &  \sqrt{\epsilon_n} + \left(\sqrt{\Gamma^{\rm{SGA}}_{\bS,\bS}}+\sqrt{\Gamma^{\rm{SGA}}_{\bS,\bS^{\prime}}}\right) \left( T^{-\frac{1}{8}} + n^{-\frac{1}{2}} \right),
        \end{aligned}
    \end{equation}
    where $\Gamma^{\rm{SGA}}_{\bS,\bS}=\sup_{\bS}\frac{\sigma_{\max}(\bV_{\bS}(\btheta_{\bS}^{*}))}{\sigma_{\min}^{+}(\bV_{\bS}(\btheta_{\bS}^{*}))\sigma_{\min}^{+}(\bV_{\bS}(\btheta_{\bS,T}^{\rm{SGA}}))}$ and $\Gamma^{\rm{SGA}}_{\bS,\bS^{\prime}}=\sup_{\bS,\bS^{\prime}}\frac{\sigma_{\max}(\bV_{\bS}(\btheta_{\bS}^{*}))}{\sigma_{\min}^{+}(\bV_{\bS}(\btheta_{\bS}^{*}))\sigma_{\min}^{+}(\bV_{\bS}(\btheta_{\bS^{\prime},T}^{\rm{SGA}}))}$.

    Then we can complete the proof using Corollary \ref{cor:generalization via stability for rondomized algorithm}.

\end{proof}

\optimizationofSGA*
\begin{proof}
    According to Lemma \ref{lem:gradient norm of SGA}, with probability at least $1-\frac{\rho^{\prime}}{2}$, it holds that: 
    \begin{equation*}
        \small
        \left\|\nabla_{\btheta}J_{\bS}(\pi_{\btheta_{\bS,T}^{\rm{SGA}}})\right\| \leq \sqrt{\frac{6L_{f}}{\sqrt{T}}\left(\left(6RC+\frac{2R^{2}C^{4}}{L_{f}}\right)+\frac{2R^{2}C^{4}}{L_{f}}\log(T)+ \frac{1}{\lambda}\log\left(\frac{2}{\rho^{\prime}}\right)\right)} \triangleq \Delta(T,\rho^{\prime}) \quad \text{and} \quad \left\|\nabla_{\btheta}J_{\bS}(\pi_{\btheta_{\bS}^{*}})\right\|=0.
    \end{equation*}
    By using Lemma \ref{lem:key lemma}, with probability at least $1-\rho^{\prime}$, it holds that
    \begin{equation*}
        \small
        \begin{aligned}
            &\left|J_{\bS}(\pi_{\btheta_{\bS,T}^{\rm{SGA}}})-J_{\bS}(\pi_{\btheta_{\bS}^{*}})\right| \le \frac{1}{n}\sum_{i=1}^{n}\left|f_{\btheta_{\bS,T}^{\rm{SGA}}}(x_{i})-f_{\btheta_{\bS}^{*}}(x_{i})\right|\le \sup_{x\in \mathcal{X}}\left|f_{\btheta_{\bS,T}^{\rm{SGA}}}(x)-f_{\btheta_{\bS}^{*}}(x)\right| 
            \\ \le & 4R\epsilon_{n}+2C\Delta(T,\rho^{\prime})\Gamma^{\rm{SGA}}_{\bS,\bS}+2RC\sqrt{4R\epsilon_{n}+2C\Delta(T,\rho^{\prime})\Gamma^{\rm{SGA}}_{\bS,\bS}}
            \\ \le &4R\epsilon_{n}+2C\Delta(T,\rho^{\prime})\Gamma^{\rm{SGA}}_{\bS,\bS}+2RC\sqrt{4R\epsilon_{n}}+\sqrt{2C\Delta(T,\rho^{\prime})\Gamma^{\rm{SGA}}_{\bS,\bS}}
            \\ = & \cO\left(T^{-\frac{1}{8}}\right).
        \end{aligned}
    \end{equation*}
    where $\Gamma^{\rm{SGA}}_{\bS,\bS}=\sup_{\bS}\frac{\sigma_{\max}(\bV_{\bS}(\btheta_{\bS}^{*}))}{\sigma_{\min}^{+}(\bV_{\bS}(\btheta_{\bS}^{*}))\sigma_{\min}^{+}(\bV_{\bS}(\btheta_{\bS,T}^{\rm{SGA}}))}$.
    
    We complete the proof of this theorem.
\end{proof}

\errorboundofSGA*
\begin{proof}
     Recall the suboptimality gap decomposition in Definition \ref{def:decomposition}, we have
    \begin{equation*}
        \small
        \begin{aligned}
            \left|J(\pi_{\btheta_{\bS,T}^{\rm{SGA}}})-J(\pi_{\btheta^{*}})\right|\le\left|J(\pi_{\btheta_{\bS,T}^{\rm{SGA}}})-J_{\bS}(\pi_{\btheta_{\bS,T}^{\rm{SGA}}})\right|+\left|J_{\bS}(\pi_{\btheta_{\bS,T}^{\rm{SGA}}})-J_{\bS}(\pi_{\btheta_{\bS}^{*}})\right|+\left|J_{\bS}(\pi_{\btheta_{\bS}^{*}})-J(\pi_{\btheta^{*}})\right|.
        \end{aligned}
    \end{equation*}
    By combining the results of Theorem \ref{thm:generalization of SGA}, Lemma \ref{lem:concentration of empirical optimal} and Lemma \ref{lem:optimization of SGA}, we complete the proof of this theorem.
\end{proof}
\begin{remark}
    Crucially, while both the optimization and generalization bounds are probabilistic, they are not independent events. The optimization bound depends on the randomness of the single parameter $\btheta_{\bS,T}^{\rm{SGA}}$, while the generalization bound  depends on the randomness of both $\btheta_{\bS,T}^{\rm{SGA}}$ and the perturbed parameter $\btheta_{\bS^{\prime},T}^{\rm{SGA}}$. Consequently, the event where the generalization bound holds strictly implies the satisfaction of the optimization bound. 
\end{remark}

\section{Detailed Experiments Setup}
In this appendix, we provide the detailed configurations for the experiments presented in Section \ref{sec:experiments}.
\subsection{The Results under Positive Definite Covariance Matrix}\label{app:positive definite covariance}
This setup corresponds to the results in Section \ref{subsec:parameter identifiable}.
\paragraph{Feature Construction.}
We construct the feature space using a standard orthonormal basis $\mathcal{X}=\{x_{1}, \cdots, x_{d}\}$ in $\mathbb{R}^d$. $x$ is directly sampled from the basis set $\mathcal{S}$ according to a non-uniform distribution: the specific direction $x_{d}$ is selected with a probability $p$ (where $p$ is very small), while the remaining $d-1$ directions are selected uniformly.
Consequently, by defining the action space to encompass opposite directions, the set of feature vectors corresponding to all responses for a given $x$ is simply $\{x, -x\}$, i.e., the feature vectors for a given $x$ are defined as $\bphi(x,a_{1})=x, \bphi(x,a_{2})=-x$.

\paragraph{Hyper-parameters.}
We set the ambient dimension to $d=20$. We vary the sample size $n$ to observe the convergence behavior. Furthermore, we assume the ground truth parameter $\btheta^{*}$ and all feasible parameters are bounded by a unit ball. 

\paragraph{Verification of Assumption \ref{ass:positive definite}}
According to Lemma \ref{lem:the same span space}, for any $x \in \mathcal{S}$, it holds that
\begin{equation*}
    \small
    \begin{aligned}
        \mathbb{E}_{x}[\bV_{x}(\btheta)]&=\sum_{i=1}^{d-1}\frac{1-p}{d-1}\bV_{x_{i}}(\btheta) +p\bV_{x_{d}}(\btheta)
        \\ & =\sum_{i=1}^{d-1}\frac{1-p}{d-1}\frac{4}{(\exp({\langle \btheta, x_{i} \rangle})+\exp({-\langle \btheta, x_{i} \rangle}))^{2}} x_{i} x_{i}^{\top}+\frac{4p}{(\exp({\langle \btheta, x_{d} \rangle})+\exp({-\langle \btheta, x_{d} \rangle}))^{2}} x_{d} x_{d}^{\top}.
    \end{aligned}
\end{equation*}
From the definition of $\{x_{1}, \cdots, x_{d}\}$, it is straightforward to see that each $x_{i}$ is an eigenvector of matrix $\mathbb{E}_{x}[\bV_{x}(\btheta)]$, with a corresponding eigenvalue satisfies 
\begin{align*}
        \small
        \sigma_{i}(\mathbb{E}_{x}[\bV_{x}(\btheta)])=
        \begin{cases}
        \frac{1-p}{d-1}\frac{4}{(\exp({\langle \btheta, x_{i} \rangle})+\exp({-\langle \btheta, x_{i} \rangle}))^{2}}\quad, i \in \{1,2,\cdots,d-1\};\\
        p\cdot\frac{4}{(\exp({\langle \btheta, x_{i} \rangle})+\exp({-\langle \btheta, x_{i} \rangle}))^{2}}\ \ \quad, i=d.
        \end{cases}
\end{align*}
Therefore, for any $\btheta$ bounded by a unit ball, it holds that
\begin{equation*}
    \small
    \min_{\btheta}\sigma_{\min}(\mathbb{E}_{x}[\bV_{x}(\btheta)])=\frac{4p}{e+e^{-1}}>0.
\end{equation*}

This confirms that Assumption \ref{ass:positive definite} holds theoretically. We vary the sample size $n$ and check whether the empirical covariance matrix is strictly positive definite.

\paragraph{Evaluation.}
We estimate the probability of strict empirical coverage ($1-\rho$) as a function of sample size $n$ using Monte Carlo simulation. In our setting, checking whether the empirical covariance matrix $\bV_{\bS}$ is positive definite simplifies to a coverage problem. Since the feature vectors are drawn from a standard orthonormal basis $\mathcal{X}=\{x_{1}, \cdots,  x_{d}\}$, the matrix $\bV_{\bS}$ is strictly positive definite if and only if the sampled features span the entire $\mathbb{R}^{d}$ space. Mathematically, this requires that every basis vector in $\mathcal{X}$ appears at least once in the dataset.

For each rare-feature probability $p$ and sample size $n$, we conduct $m = 10,000$ independent trials. In each trial, we independently sample $n$ feature vectors from the specified distribution. We record a trial as ``successful'' if the set of unique vectors collected has a cardinality of $d$ (i.e., full coverage). The reported probability $1-\rho$ is the fraction of successful trials out of the $m$ repetitions.

\subsection{Generalization under Insufficient Coverage}\label{app:feature coverage}
This setup corresponds to the results in Section \ref{subsec:insufficient coverage}.

\paragraph{Feature Construction.}

We construct the feature space constrained to a low-dimensional subspace spanned by an orthonormal basis set $\mathcal{A} = \{a_{1}, \cdots, a_{d_{\rm{eff}}}\}$ in $\mathbb{R}^{d}$, where $d_{\rm{eff}} < d$. For a given $x$, there are two responses corresponding to it, the feature vectors are defined as $\bphi(x,a_{i_u})=a_{i_u} a_{i_u}^\top x$ and $\bphi(x,a_{i_{v}})=a_{i_{v}} a_{i_{v}}^\top x$.


In detail, we generate a vector $x$ from a standard Gaussian distribution and normalize it to satisfy the boundedness assumption, ensuring it is non-orthogonal to the effective subspace $\mathcal{A}$.
At the same time, we randomly select two indices $i_{u}, i_{v} \in \{1, \dots, d_{\rm{eff}}\}$.

\paragraph{Hyper-parameters.}
    We set the ambient dimension $d$ and the effective subspace dimension $d_{\rm{eff}} < d$, creating a rank-deficient environment. The ground truth parameter $\btheta^{*}$ is generated strictly within this effective subspace.
    \par
    \textbf{Case 1: Suboptimality gap scaling with sample size $n$.} We vary the sample size $n$ from 7 to 22 to study the generalization error scaling. To verify the conclusion of Theorems \ref{thm:generalization of exact optimization}, \ref{thm:generalization of GA} and \ref{thm:generalization of SGA}, we calculate the empirical stationary point exactly, run Gradient Ascent (GA) and Stochastic Gradient Ascent (SGA) for sufficiently many steps to ensure convergence, respectively. Specifically, based on our convergence analysis, we set the iteration steps $T^{\rm{GA}} = 50 n^{2}$ and $T^{\rm{SGA}} = 10 n^{4}$. We evaluate the generalization performance under two distinct settings: $d=38, d_{\rm{eff}}=32$ and $d=32, d_{\rm{eff}}=24$. 

    \textbf{Case 2: Suboptimality gap decreasing with optimization steps $T$.}
    We set the ambient dimension $d=18$ and the effective subspace dimension $d_{\rm{eff}}=14$, creating a rank-deficient environment. We fix the sample size at two representative levels, $n=7$ (high rank deficiency) and $n=10$ (moderate rank deficiency). We vary the optimization steps $T$ on a scale from $10^1$ to $10^7$ to observe the trajectory of the learned policy. 

\paragraph{Verification of Assumption \ref{ass:span space}.}
    By design, every feature vector $\bphi(x, a)$ lies within the $d_{\rm{eff}}$-dimensional effective subspace. According to Lemma \ref{lem:the same span space}, When $n <d_{\rm{eff}}$, the empirical covariance matrix $\bV_{\bS}$ can only span at most an $n-$ dimensional subspace. Consequently, for any new feature vector $\bphi$, it cannot be fully represented by the column space of $\bV_{\bS}$. It decomposes into a component within $\mathrm{C}(\bV_{\bS})$ and a non-zero residual term $\br_{\bS}$ orthogonal to it. As the sample size $n$ increases, the subspace spanned by $\bV_{\bS}$ expands, causing the magnitude of the residual $\br_{\bS}$ to diminish statistically. This geometric property perfectly simulates the conditions characterized in Assumption \ref{ass:span space}, allowing us to validate the dimension-independent bounds derived in our theorems.

\paragraph{Evaluation.} 
We first generate an independent set $\bS_{\rm{train}}$ of size $n$ for training. To rigorously evaluate the theoretical quantities, we generate a large, independent test set $\bS_{\rm{test}}$ of size 15000 to approximate the expected objective.

 For \textbf{Case 1}, the learned $\btheta$ is the parameter obtained after the full training duration. For \textbf{Case 2}, $\btheta$ corresponds to the parameter at each check-pointed iteration step $t$. For any learned $\btheta$, 
 the Suboptimality Gap is estimated by  $|J_{\bS_{\rm{test}}}(\pi_{\btheta^{*}}) - J_{\bS_{\rm{test}}}(\pi_{\btheta})|$.

\section{Some Lemmas on Algebra.}

\begin{lemma}\label{lem:A orthogonal complement}
    Let $\mathbb{U}$ and $\mathbb{V}$ be any two subspaces of the vector space $\mathbb{R}^{d}$. Then, the orthogonal complement of their intersection, $(\mathbb{U}\cap\mathbb{V})^{\perp}$, is equal to the sum of their individual orthogonal complements, $\mathbb{U}^{\perp}$ and $\mathbb{V}^{\perp}$. That is:
    \begin{equation*}
        \small
        \mathbb{U}^{\perp}\cap\mathbb{V}^{\perp}=(\mathbb{U}+\mathbb{V})^{\perp}.
    \end{equation*}
\end{lemma}
\begin{proof}
    We prove the equality by showing both inclusions. First, we show that $(\mathbb{U}+\mathbb{V})^{\perp} \subseteq \mathbb{U}^{\perp}\cap\mathbb{V}^{\perp}$.
    Clearly we have $\mathbb{U} \subseteq \mathbb{U}+\mathbb{V}$, it is straightforward to show that \[(\mathbb{U}+\mathbb{V})^{\perp} \subseteq \mathbb{U}^{\perp},\] similarly, we also have \[(\mathbb{U}+\mathbb{V})^{\perp} \subseteq \mathbb{V}^{\perp},\] combining these two conditions, we have\[(\mathbb{U}+\mathbb{V})^{\perp} \subseteq \mathbb{U}^{\perp}\cap\mathbb{V}^{\perp}.\]
    On the other hand, we will show that $\mathbb{U}^{\perp}\cap\mathbb{V}^{\perp}\subseteq(\mathbb{U}+\mathbb{V})^{\perp}$.
    Let $x$ be an arbitrary vector in $\mathbb{U}^{\perp}\cap\mathbb{V}^{\perp}$, which means that for any $y_{1} \in \mathbb{U}$ and $y_{2}\in \mathbb{V}$, we have \[\langle x,y_{1} \rangle=\langle x,y_{2}\rangle=\langle x,y_{1}+y_{2} \rangle=0.\] Since $y_{1}$ and $y_{2}$ are chosen arbitrarily, we have $x \in \mathbb{U}+\mathbb{V}$. Therefore, we can obtain that\[\mathbb{U}^{\perp}\cap\mathbb{V}^{\perp}\subseteq(\mathbb{U}+\mathbb{V})^{\perp}.\]
    This completes the proof of the lemma.
    
\end{proof}

\begin{lemma}\label{lem:A two matrix null space}
    Let $A,B \in \mathbb{R}^{n\times n}$ be two positive semi-definite matrices, then it holds that 
    \begin{equation}
    \small
        \begin{aligned}
            \text{N}(A+B)=\text{N}(A)\cap \text{N}(B).
        \end{aligned}
    \end{equation}
\end{lemma}
\begin{proof}
    We prove the equality by showing both inclusions. First, we show that $\text{N}(A+B)\subseteq\text{N}(A)\cap \text{N}(B)$. Let $\bx$ be an arbitrary vector in $\text{N}(A+B)$, which means that $\bx$ satisfies \[(A+B)\bx=0,\] then we have \[\bx^{\top}A\bx+\bx^{\top}B\bx=\bx^{\top}(A+B)\bx=0,\]
    moreover, given that $A$ and $B$ are both positive semi-definite, it follows that\[\bx^{\top}A\bx=\bx^{\top}B\bx=0,\] it is straightforward to know that $\bx$ belonging to both $\text{N}(A)$ and $\text{N}(B)$, i.e., $\bx \in \text{N}(A)\cap \text{N}(B)$. Therefore, we conclude that \[\text{N}(A+B)\subseteq\text{N}(A)\cap \text{N}(B).\] On the other hand, we will show that $\text{N}(A)\cap \text{N}(B) \subseteq \text{N}(A+B)$. Similarly, for an arbitrary vector $\bx \in \text{N}(A)\cap \text{N}(B)$, we have\[\bx^{\top}(A+B)\bx=\bx^{\top}A\bx=\bx^{\top}B\bx=Ax=B\bx=0.\] Therefore, it holds that $(A+B)\bx=0$, i.e. $\bx \in \text{N}(A+B)$. Therefore, we conclude that \[\text{N}(A)\cap \text{N}(B)\subseteq \text{N}(A+B).\] This completes the proof of the lemma. 
\end{proof}

\begin{lemma}\label{lem:A two matrix column space}
    Let $A,B \in \mathbb{R}^{n\times n}$ be two positive semi-definite matrices, it holds that
    \begin{equation}
        \small 
        \begin{aligned}
            \mathrm{C}(A+B)=\mathrm{C}(A)+\mathrm{C}(B).
        \end{aligned}
    \end{equation}
\end{lemma}
\begin{proof}
    For the positive semi-definite matrix $A+B$, we have
    \begin{equation*}
        \small
        \mathrm{C}(A+B)=\mathrm{C}((A+B)^{\top})=\text{N}(A+B)^{\perp},
    \end{equation*}
    where the first equality holds because the matrix $A+B$ is symmetric, while the second equality follows from the Fundamental Theorem of Linear Algebra. Combining with lemma\ref{lem:A orthogonal complement} and lemma\ref{lem:A two matrix null space}, it holds that
    \begin{equation*}
        \small
        \mathrm{C}(A+B)=\text{N}(A+B)^{\perp}=(\text{N}(A)\cap \text{N}(B))^{\perp}=\text{N}(A)^{\perp}+\text{N}(B)^{\perp}=\mathrm{C}(A^{\top})+\mathrm{C}(B^{\top})=\mathrm{C}(A)+\mathrm{C}(B),
    \end{equation*}
    which proves our conclusion.
\end{proof}

\begin{lemma}\label{lem:Cp}
    Given a positive vector $\bp=(p_{1},\cdots ,p_{n})^{\top} \in \mathbb{R}^{n}$ that satisfies $\sum_{i=1}^{n}p_{i}=1$, denote $C_{\bp}$ as:
    \[C_{\bp}=\diag\{p_{1},\cdots,p_{n}\}-\bp\bp^{\top}.\] Then we have that the rank of the matrix $C_{\bp}$ is $n-1$, and the column space of $C_{\bp}$ can be demonstrated as \[\mathrm{C}(C_{\bp})=\{\bw|\bw^{\top}\bv=0,\bv=(1,\cdots,1)^{\top}\}.\]
\end{lemma}
\begin{proof}
     Let vector $\bv$ be any vector falls in the null space of matrix $C_{\bp}$, then we have  
    \begin{equation*}
        \small
        0=C_{\bp}\bv=\left(\diag\{p_{1},\cdots,p_{n}\}-\bp\bp^{\top}\right)\bv= \diag\{p_{1},\cdots,p_{n}\}\bv-(\bp^{\top}\bv)\bp,
    \end{equation*}
    then for any $i,j \in \{1,2,\cdots,n\}$, it holds that
    \begin{equation*}
        \small
        0=p_{i}v_{i}-(\bp^{\top}\bv)p_{i}=p_{j}v_{j}-(\bp^{\top}\bv)p_{j},
    \end{equation*}
    therefore, it holds that \[\small v_{i}=(\bp^{\top}\bv)=v_{j},\] for vector $\bp$ is a positive vector. 
    It follows that the null space of matrix $C_{\bp}$ is $\text{N}(C_{\bp})=\{\bv|\bv=k(1,\cdots,1)^{\top},k\in \mathbb{R}\}$. According to the Fundamental Theorem of Linear Algebra, we can further express the row space of $C_{\bp}$ as \[\mathrm{C}(C_{\bp}^{\top})=\{\bw|\bw^{\top}\bv=0,\bv=(1,\cdots,1)^{\top}\}.\]
    Furthermore, since $C_{\bp}$ is a symmetric matrix, its row space and column space are identical, which proves the lemma.
\end{proof}

\begin{lemma}\label{lem:Cp2}
    Let $A \in \mathbb{R}^{m \times n}$ be a matrix with full column rank and $C_{\bp}$ be a matrix satisfying Lemma \ref{lem:Cp}, where $m>n$. Then, the column spaces of matrices $AC_{\bp}A^{\top}$ and $AC_{\bp}$ are identical, and both can be expressed in the form of $\{A\bw|\bw^{\top}\bv=0,\bv=(1,\cdots,1)^{\top}\}$, i.e. \[\mathrm{C}(AC_{\bp}A^{\top})=\mathrm{C}(AC_{\bp})=\{A\bw|\bw^{\top}\bv=0,\bv=(1,\cdots,1)^{\top}\}.\]
\end{lemma}
\begin{proof}
    
    We begin by showing that $\mathrm{C}(AC_{\bp}A^{\top})$ is contained in $\mathrm{C}(AC_{\bp})$. For any $\by \in \mathrm{C}(AC_{\bp}A^{\top})$, there exists a $\bu$ such that \[\by=AC_{\bp}A^{\top}\cdot \bu=AC_{\bp}\cdot(A^{\top} \bu),\]
    thus it is clear that we also have $\by \in \mathrm{C}(AC_{\bp})$.
    
    On the other hand, for any $\by \in \mathrm{C}(AC_{\bp})$, we can find a $\bv$ such that \[\by = AC_{\bp}\cdot\bv.\]
    Since matrix $A$ has full column rank, it follows that there exists a vector $\bu$ such that \[\small A^{\top}\bu=\bv.\] 
    
    We have thus far proven that $\mathrm{C}(AC_{\bp}A^{\top})=\mathrm{C}(AC_{\bp})$. Furthermore, using Lemma \ref{lem:Cp}, it is straightforward to show that \[\mathrm{C}(AC_{\bp})=\{A\bw|\bw^{\top}\bv=0,\bv=(1,\cdots,1)^{\top}\},\]
    which proves our conclusion.
\end{proof}

\end{document}